\documentclass{article}

\usepackage[preprint]{neurips_2026}

\usepackage[utf8]{inputenc}
\usepackage[T1]{fontenc}
\usepackage{hyperref}
\usepackage{url}
\usepackage{booktabs}
\usepackage{amsfonts}
\usepackage{amsmath}
\usepackage{amssymb}
\usepackage{amsthm}
\usepackage{nicefrac}
\usepackage{microtype}
\usepackage[table]{xcolor}
\definecolor{oursblue}{RGB}{222,235,247}
\usepackage{graphicx}
\usepackage{multirow}
\usepackage{subcaption}
\usepackage{algorithm}
\usepackage{algorithmic}
\usepackage{float}
\usepackage{pifont}

\graphicspath{{figures/}}
\newtheorem{proposition}{Proposition}

\title{SPLICE: Latent Diffusion over JEPA Embeddings\\
for Conformal Time-Series Inpainting}

\author{%
  Arnaud Zinflou \\
  Hydro-Qu\'{e}bec Research Institute \\
  Canada \\
  \texttt{zinflou.arnaud@hydroquebec.com}
}

\begin{document}

\maketitle

\begin{abstract}
Long gaps in time series are difficult to reconstruct reliably, and
in power systems imputed load values inform dispatch and planning.
Existing generative imputers are accurate but provide no
finite-sample reliability guarantees under changing conditions.
We introduce \textbf{SPLICE} (Self-supervised Predictive Latent
Inpainting with Conformal Envelopes), which couples a JEPA encoder, a
conditional latent bridge, an hourly-conditioned decoder, and
Adaptive Conformal Inference (ACI).
Controlled ablations indicate that long-horizon fidelity is driven
more by the predictive representation than by generative complexity:
under a matched bridge and decoder, JEPA outperforms reconstructive
VAE and contrastive TS2Vec encoders, and once the representation is
fixed a 5-step flow sampler matches or improves on a 50-step DDIM
sampler at a $5{-}10\times$ speedup.  The same ordering is recovered
on twelve hydrological catchments.
Across thirteen load datasets at 91-day gaps, SPLICE attains the
lowest mean Load-only MSE (0.056), winning 9 of 12 non-degenerate
datasets against seven established baselines, and the best mean CRPS
(0.182, 7.8\% below the strongest competitor), while ACI holds
93--95\% empirical coverage where static calibration under-covers by
up to 7.5 percentage points.
A pooled JEPA encoder transfers to four held-out datasets with only
bridge and decoder adaptation.
\end{abstract}

\section{Introduction}
\label{sec:introduction}

The reliable operation of power grids depends on accurate
reconstruction and prediction of short-term electricity
demand using historical consumption data alongside external factors such as
weather and calendar effects. Real-world metering systems, however,
frequently encounter extended intervals of missing data due to
sensor failures, maintenance outages, or communication errors.
Traditional imputation methods revert to smooth interpolations that
fail to preserve the fine-grained temporal structure required for
operational planning, while recent deep generative
approaches~\citep{tashiro2021csdi, alcaraz2023sssd} produce
plausible reconstructions but offer no formal guarantees on the
reliability of their predictions, a critical gap for
safety-sensitive infrastructure.
The two error directions are not symmetric: an over-confident
reconstruction \emph{understates} the true uncertainty, and that
shortfall propagates into reserve sizing, dispatch, and
demand-response commitments taken on the assumption that the band is
trustworthy.  Point accuracy cannot expose this failure, which is why
we evaluate conformal coverage and CRPS alongside reconstruction MSE.

We address this limitation with SPLICE, a framework that applies JEPA
predictive representation learning to operational long-horizon
time-series inpainting, combining it with efficient flow-based
generation and online conformal calibration.
Conceptually, our approach instantiates LeCun's
\emph{JEPA-based world model} vision~\citep{lecun2022path}: a
Joint Embedding Predictive Architecture learns a latent dynamics
model of the load system, enabling the pipeline to ``imagine''
plausible trajectories for unobserved intervals rather than merely
interpolating between observed endpoints.  Unlike earlier world
models designed for RL planning~\citep{ha2018world}, our system is
purpose-built for conditional generation with uncertainty
quantification, a distinction that shapes every architectural
choice.

The pipeline comprises four independently replaceable modules
(Section~\ref{sec:method}): a JEPA
encoder~\citep{assran2023jepa} that maps daily segments into a
structured 64-dim latent space; a conditional bridge that generates
gap trajectories with tuneable stochasticity; a flow-matching
sampler~\citep{lipman2023flow}; and an hourly-conditioned decoder.
Adaptive Conformal Inference~(ACI) \citep{gibbs2021aci} wraps the
pipeline, adjusting its miscoverage level at each step and provably
driving the time-averaged miscoverage to its target rate without
manual recalibration.  Downstream forecasters trained on the imputed
series reach single-digit WMAPE on most datasets
(Appendix~\ref{sec:forecasting}).

\textbf{Contributions.} Our main contributions are:
\begin{itemize}
  \item A portable design principle, established by
    operationalising JEPA for inpainting 91-day gaps and running
    controlled ablations: fidelity is governed primarily by the
    predictive representation rather than by generative complexity.
    A JEPA predictive encoder outperforms a
    reconstructive VAE latent by $25.1\%$ MSE (Table~\ref{tab:mse}) and a
    contrastive TS2Vec latent (Appendix~\ref{sec:encoder_swap}) under a
    matched bridge and decoder, while a lightweight 5-step flow sampler
    matches or improves on a 50-step DDIM sampler ($-8.7\%$ MSE,
    Table~\ref{tab:gen-ablation}).  A targeted replication on twelve
    hydrological catchments recovers the same encoder and sampler
    ordering outside electricity-load data
    (Section~\ref{sec:camels}).  We realise the principle in a modular
    pipeline that attains the lowest mean Load-only MSE across our
    datasets.
  \item Distribution-free uncertainty via ACI, which corrects the
    under-coverage of a static calibrator (by up to 7.5~pp) to reach
    near-nominal empirical coverage under non-stationarity; stochastic
    latent ensembles yield the lowest mean CRPS among all baselines.
  \item Ablations spanning gap-length sensitivity, decoder variants,
    and generative backends, isolating the contribution of each
    design choice
    (Appendices~\ref{sec:diffusion_ablation}--\ref{sec:incremental_ablation}).
  \item A preliminary transfer study showing that a pooled JEPA
    encoder transfers to four unseen domains, matching or exceeding
    per-dataset oracles on three of them with only bridge and decoder
    adaptation (Section~\ref{sec:transfer}).
\end{itemize}

\section{Related Work}
\label{sec:related}

\textbf{Latent dynamics and self-supervised representations.}
LeCun~\citep{lecun2022path} proposes JEPA as the backbone for world models that learn predictive latent dynamics; early instantiations target RL planning~\citep{ha2018world, hafner2020dream}, whereas we repurpose the same principle for conditional gap inpainting with uncertainty quantification.  Self-supervised time-series representations are commonly built on masked modelling, as in PatchTST~\citep{nie2023patchtst}, or contrastive learning over augmented views, as in TS2Vec~\citep{yue2022ts2vec}. In parallel, joint-embedding architectures and related methods (JEPA~\citep{lecun2022path}, I-JEPA~\citep{assran2023jepa}, V-JEPA~\citep{bardes2024vjepa}, BYOL~\citep{grill2020byol}, SimSiam~\citep{chen2021simsiam}, VICReg~\citep{bardes2022vicreg}, TS-JEPA~\citep{ennadir2025tsjepa}) produce compact embeddings but do not generate long, high-fidelity trajectories; their lossy latent predictions sacrifice the temporal detail needed for faithful gap reconstruction. We address this by pairing a JEPA encoder with an explicit generative bridge that operates within the learned latent space.
\textbf{Generative imputation and flow matching.}
CSDI~\citep{tashiro2021csdi} and SSSD~\citep{alcaraz2023sssd} apply diffusion in observation space; GP-VAE~\citep{fortuin2020gp} uses Gaussian process priors in a VAE latent space; we operate in a JEPA latent space and add conformal coverage.  Our VAE+Bridge ablation (Table~\ref{tab:mse}) turns that comparison into a controlled swap of the latent alone.  Flow matching~\citep{lipman2023flow} and rectified flows~\citep{liu2023rectified} learn straight ODE paths integrable in 5--10 steps, which we adopt as a drop-in replacement for DDIM.
\textbf{Conformal prediction.}
CQR~\citep{romano2019cqr} provides distribution-free intervals under exchangeability~\citep{vovk2005algorithmic}, but its static calibration leads to under-coverage in non-stationary load data.  ACI~\citep{gibbs2021aci, zaffran2022aci} updates $\alpha_t$ online, maintaining long-run coverage under distributional shift.  To our knowledge, SPLICE is the first work to wrap a deep generative imputation pipeline with ACI (Table~\ref{tab:method_comparison}, appendix).

\section{Problem Formulation}
\label{sec:problem}

We consider a multivariate time series
$\mathbf{x} = \{x_1, \dots, x_T\}$ observed at hourly resolution,
where each $x_t \in \mathbb{R}^d$ contains the load measurement and
associated covariates (temperature, humidex, wind speed, weather code,
and calendar features). A contiguous gap interval
$[t_a, t_b] \subset [1, T]$ is unobserved, and we define:
\begin{equation}
  x_{\mathrm{obs}} = \{x_t : t \notin [t_a, t_b]\}, \quad
  x_{\mathrm{gap}} = \{x_t : t \in [t_a, t_b]\}.
  \label{eq:gap}
\end{equation}

The goal is to produce (i)~imputed trajectories
$\hat{x}_{\mathrm{gap}}$ that faithfully reconstruct the temporal
structure of the missing interval, and (ii)~prediction intervals
$[l_t, u_t]$ for each $t \in [t_a, t_b]$ satisfying a long-run
coverage guarantee:
\begin{equation}
  \lim_{T \to \infty} \frac{1}{T} \sum_{t=1}^{T}
  \mathbf{1}\{x_t \in [l_t, u_t]\} \geq 1 - \alpha,
  \label{eq:coverage}
\end{equation}
where $\alpha \in (0,1)$ is the target miscoverage rate (we use
$\alpha = 0.05$ throughout).  We deliberately target a long-run
coverage \emph{floor} rather than an unconditional error level because,
for downstream planning, the consequential direction is
under-coverage: intervals that are too narrow understate risk
precisely where reserve and dispatch decisions are taken.
This guarantee must hold without
distributional assumptions, in particular without exchangeability,
which is routinely violated in load time series due to seasonal drift,
weather regime changes, and demand growth.

We decompose the problem into three stages:
\emph{representation} (a JEPA encoder maps daily segments into latent
embeddings $z_t = f_\theta(x_t) \in \mathbb{R}^p$),
\emph{generation} (a conditional latent bridge produces candidate
trajectories
$\hat{z}_{\mathrm{gap}} \sim G_\phi(\cdot \mid z_{\mathrm{obs}})$),
and \emph{calibration} (ACI wraps decoder output with adaptive
prediction intervals).  Each stage is detailed in
Section~\ref{sec:method}.

\section{Method}
\label{sec:method}

\begin{figure}[t]
  \centering
  \includegraphics[width=\linewidth]{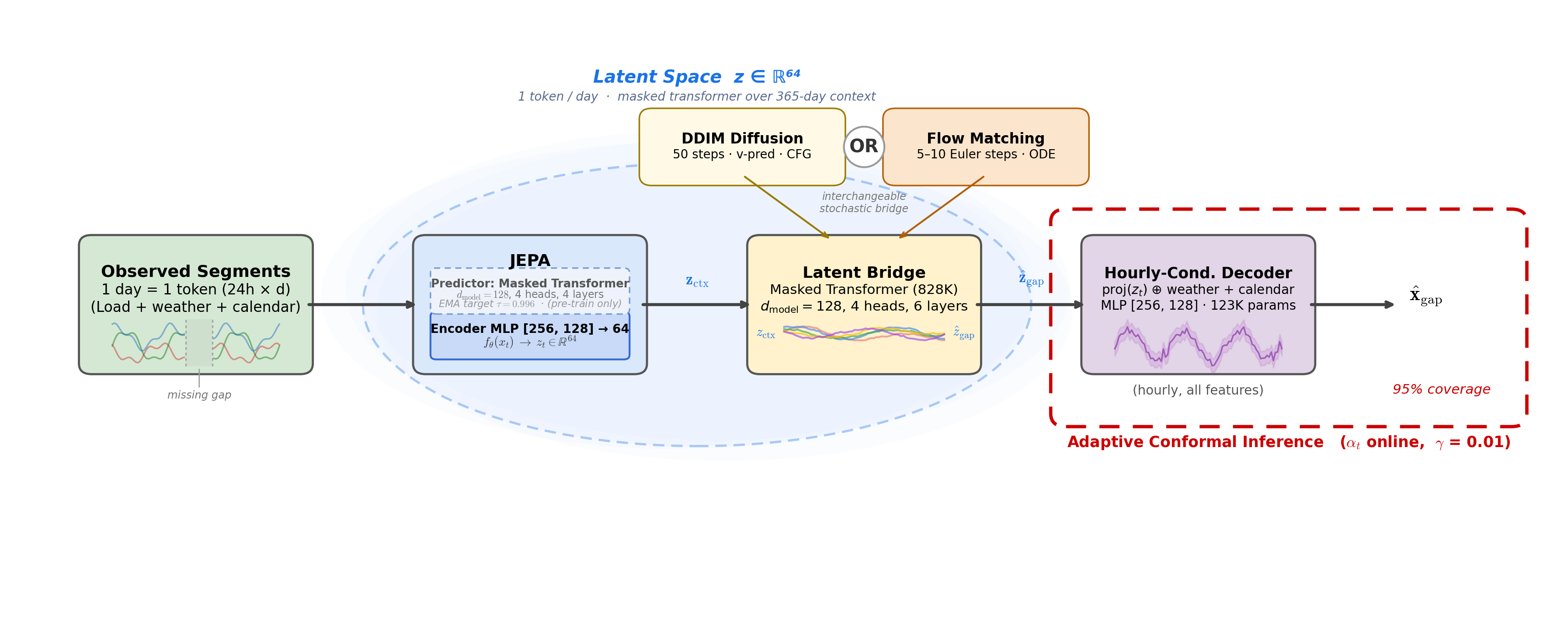}
  \caption{Overview of the SPLICE pipeline. Multivariate daily
  segments (load, temperature, wind, solar angle, and other
  covariates) are encoded by JEPA into 64-dimensional latent
  embeddings. A conditional latent bridge (four modes; see \S\ref{sec:diffusion})
  generates gap trajectories.
  A weather- and calendar-conditioned MLP decoder maps back to hourly
  signal space. ACI wraps the output with adaptive prediction
  intervals.}
  \label{fig:pipeline}
\end{figure}

Figure~\ref{fig:pipeline} illustrates the full architecture.
We describe each component below.

\subsection{JEPA Encoder: Latent State Inference}
\label{sec:jepa}

We adopt a Joint Embedding Predictive Architecture~(JEPA) as the
representation backbone. Each daily load segment
$x_t \in \mathbb{R}^{24 \times d}$ (24~hours $\times$ $d$~features) is
encoded into a compact representation $z_t \in \mathbb{R}^{64}$ via a
two-layer MLP encoder:
\begin{equation}
  z_t = f_\theta(x_t) = \mathrm{MLP}(x_t;\; [256, 128] \to 64).
  \label{eq:encoder}
\end{equation}

The JEPA predictor is a masked Transformer ($d_{\mathrm{model}}{=}128$,
$n_{\mathrm{heads}}{=}4$, $n_{\mathrm{layers}}{=}4$) that predicts the
latent embedding of masked days from the context of observed days. The
training objective combines:
 \textbf{(1) Prediction loss:} cosine similarity between predicted
    and target embeddings (EMA target encoder, decay${}=0.996$).
  \textbf{(2) Variance regularisation} (weight${}=0.05$): prevents
    embedding collapse by penalising low variance across the
    batch~\citep{bardes2022vicreg}.
  \textbf{(3) Covariance regularisation} (weight${}=0.001$):
    decorrelates embedding dimensions to encourage information
    spreading.
Formally, the full JEPA objective is:
\begin{equation}
  \mathcal{L}_{\mathrm{JEPA}} =
  \underbrace{1 - \frac{\langle \hat{z},\, \bar{z} \rangle}
    {\|\hat{z}\|\,\|\bar{z}\|}}_{\text{cosine prediction}}
  \;+\; \lambda_v \underbrace{\sum_{j=1}^{p}
    \max\!\bigl(0,\, 1 - \sqrt{\mathrm{Var}(z_j) + \varepsilon}\,\bigr)}_{\text{variance hinge}}
  \;+\; \lambda_c \underbrace{\frac{1}{p^2}
    \!\sum_{i \neq j}\! C_{ij}^2}_{\text{off-diagonal covariance}},
  \label{eq:jepa_loss}
\end{equation}
where $\hat{z}$ is the predictor output, $\bar{z}$ the
exponential-moving-average target, $C = \mathrm{Cov}(\mathbf{z})$
the batch covariance matrix, $\lambda_v = 0.05$,
$\lambda_c = 0.001$, and $\varepsilon = 10^{-4}$.

A separate decoder MLP ($[128, 256] \to d_{\mathrm{input}}$) is trained
to reconstruct $x_t$ from $z_t$, providing the signal-space mapping
needed for evaluation. The full JEPA module has approximately 1.06M
parameters.

\subsection{Conditional Latent Bridge: Trajectory Generation}
\label{sec:diffusion}

The bridge architecture supports a \emph{controllable spectrum of stochasticity},
allowing practitioners to trade point accuracy for distributional richness:
\textbf{(i) Deterministic} ($\sigma{=}0$): A masked Transformer (828K parameters)
directly predicts gap embeddings from context, optimised for MSE. This variant
produces the lowest point-forecast error and is used in the main comparison
(Section~\ref{sec:ext-baselines}).
\textbf{(ii) Noise-perturbed} ($\sigma{=}0.15$): Gaussian noise is added to the
deterministic bridge before decoding, yielding $M$ diverse trajectories.
This lightweight mechanism achieves well-calibrated forecasts (lowest CRPS;
Section~\ref{sec:distributional}) and provides conformity scores for ACI
(Section~\ref{sec:aci}).
\textbf{(iii) Generative} (1.36M parameters): Conditional DDIM diffusion or
flow-matching models generate gap embeddings with maximum trajectory diversity.
Evaluated in Appendix~\ref{sec:diffusion_ablation}.
All three variants share the same Transformer backbone
($d_{\mathrm{model}}{=}128$, $n_{\mathrm{heads}}{=}4$,
$n_{\mathrm{layers}}{=}6$); the generative variants add a
stochastic denoising or flow wrapper.  Small perturbations in the
JEPA space ($\sigma{=}0.15$) produce diverse trajectories with
limited loss of reconstruction accuracy.

\paragraph{Flow matching as the preferred generative backend.}
\label{sec:flow_matching}
Flow matching~\citep{lipman2023flow, liu2023rectified} learns a
velocity field that transports samples along \emph{straight} ODE
paths from noise to data.  At inference, gap embeddings are
generated by Euler integration of the learned field
$v_\theta(z_t, t, c)$ conditioned on context embeddings
$c = z_{\mathrm{obs}}$:
\begin{equation}
  z_{t - \Delta t} = z_t - \Delta t \cdot v_\theta(z_t, t, c),
  \quad \Delta t = 1/N,
  \label{eq:fm_euler}
\end{equation}
where $N{=}5{-}10$ steps suffice in the learned JEPA latent
space, yielding a $5{-}10\times$ speedup over the 50-step DDIM
sampler at comparable quality.  Classifier-free
guidance~\citep{ho2022cfg} amplifies conditioning by interpolating
between conditional and unconditional velocity predictions, with
per-dataset guidance scale selected on validation
(Appendix~\ref{app:guidance}).  Full diffusion and flow-matching
formulations (forward process, $v$-prediction, training objectives)
are given in Appendix~\ref{app:bridge_details}.

\subsection{Adaptive Conformal Inference}
\label{sec:aci}

Standard conformal prediction assumes exchangeability of calibration
and test data, which is violated in non-stationary time series.
Conformalized Quantile Regression~(CQR)~\citep{romano2019cqr}
computes a calibration quantile $\hat{q}$ from held-out residuals:
\begin{equation}
  \hat{q} = \mathrm{Quantile}_{1-\alpha}\!\bigl(\bigl\{
  \max(\hat{q}_{\mathrm{lo}}^{(i)} - y^{(i)},\;
  y^{(i)} - \hat{q}_{\mathrm{hi}}^{(i)})\bigr\}_{i=1}^n\bigr).
  \label{eq:cqr}
\end{equation}
In our pipeline, the quantile estimates $\hat{q}_{\mathrm{lo}}$
and $\hat{q}_{\mathrm{hi}}$ are constructed from multiple
generative samples rather than parametric quantile regression.
For each calibration window we draw $S{=}50$ independent
trajectories through the full pipeline (bridge $\to$
decoder), yielding $S$ signal-space reconstructions of the 91-day
gap. Pointwise empirical quantiles at levels $\alpha/2$ and
$1 - \alpha/2$ across these $S$ samples define the initial
prediction band; the CQR nonconformity score is
\begin{equation}
  R_t = \max\!\bigl(\hat{q}_{\mathrm{lo},t} - y_t,\;
  y_t - \hat{q}_{\mathrm{hi},t}\bigr),
  \label{eq:nonconformity}
\end{equation}
which is zero when $y_t$ lies inside the band. The collection of
scores $\{R_t\}$ over all calibration windows provides the
empirical distribution from which $\hat{q}$ (Eq.~\ref{eq:cqr})
is computed.

ACI~\citep{gibbs2021aci} makes this adaptive by maintaining a running
miscoverage level $\alpha_t$ that is updated at each time step based
on observed coverage:
\begin{equation}
  \alpha_{t+1} = \alpha_t + \gamma(\alpha - \mathrm{err}_t),
  \quad
  \mathrm{err}_t = \mathbf{1}\{y_t \notin C_t(\alpha_t)\},
  \label{eq:aci}
\end{equation}
where $\gamma > 0$ is the learning rate (we use $\gamma = 0.01$) and
$\alpha_t$ is clamped to $(0.001, 0.999)$. The key theoretical
guarantee is:
\begin{equation}
  \lim_{T \to \infty} \frac{1}{T} \sum_{t=1}^{T} \mathrm{err}_t
  = \alpha.
  \label{eq:aci_guarantee}
\end{equation}
This holds without distributional assumptions; the only requirement
is that the base predictor produces valid quantile estimates on the
calibration set (see Appendix~\ref{app:aci} for the finite-sample
bound). When the pipeline under-covers ($\mathrm{err}_t = 1$ too
often), $\alpha_t$ decreases, widening the bands. When it
over-covers, $\alpha_t$ increases, tightening them.

\subsection{Hourly-Conditioned Decoder and End-to-End Pipeline}
\label{sec:decoder}

A per-hour conditioned MLP decoder maps
latent trajectories $\hat{z}_{\mathrm{gap}}$ back to signal space.
The decoder first projects each 64-dimensional JEPA embedding, one
latent token per calendar day (24~hours being the natural diurnal
period of load), to a 256-dimensional hidden state, broadcasts it to
the 24~hourly steps of that day, and
concatenates per-hour weather covariates (temperature, dew point, wind,
humidity) and calendar encodings (sine/cosine hour, day-of-week,
month). A shared two-layer MLP ($[256, 128] \to d_{\mathrm{features}}$)
with LayerNorm and GELU activations then decodes each hour
independently:
\begin{equation}
  \hat{x}_{t,h} = \mathrm{Dec}_\theta(\mathrm{proj}(z_t),\, c_{t,h}),
  \quad z_t \in \mathbb{R}^{64},\;
  c_{t,h} = [\mathrm{weather}_h, \mathrm{calendar}_h].
  \label{eq:decoder}
\end{equation}
This per-hour conditioning closes the informational gap between
latent-space and observation-space methods: the decoder receives the
same weather covariates that external baselines~(BRITS, SAITS) access
directly within the gap.

Beyond this larger capacity, the enhanced decoder adds a Load-weighted
MSE loss ($5\times$ weight on the Load feature) to prioritise the target
variable and noise augmentation ($\sigma = 0.15$ Gaussian noise on
50\% of embedding batches) to improve robustness to bridge
reconstruction errors; training additionally uses cosine learning-rate
annealing.  The decoder is trained with the JEPA encoder frozen
($\approx$123K parameters).  These three choices are the enhancements
isolated in Section~\ref{sec:decoder_ablation}.

End-to-end training proceeds in four sequential stages:
(1)~JEPA encoder and predictor pre-training with VICReg-style
regularisation (Section~\ref{sec:jepa});
(2)~hourly-conditioned decoder on frozen JEPA embeddings with noise
augmentation (this section);
(3)~conditional latent bridge on JEPA embeddings, deterministic for
point accuracy with optional DDIM or flow-matching extensions on the
same backbone for distributional generation
(Section~\ref{sec:diffusion}; Appendix~\ref{app:bridge_details});
(4)~conformal calibration and ACI online adaptation
(Section~\ref{sec:aci}).
The stages are decoupled by design rather than trained end-to-end:
back-propagating a decoder loss into the JEPA encoder would reintroduce
the representation-collapse pressure that the EMA-target,
prediction-only objective is built to avoid, biasing the encoder toward
decoder-specific rather than predictive features.  Keeping the encoder
frozen is also what makes the transfer result
(Section~\ref{sec:transfer}) possible, since a pooled encoder reused
across unseen domains was never specialised to a single decoder.
Algorithm~\ref{alg:pipeline} (Appendix) summarises the inference
procedure.

\section{Experimental Setup}
\label{sec:setup}

\subsection{Datasets}

We evaluate on thirteen hourly-resolution datasets
(Table~\ref{tab:datasets}, appendix): nine proprietary utility feeds
from the northeastern US (2015--2025), three UCI
Electricity~\citep{uci_electricity} clients (MT\_001/150/320), and
ETTh1~\citep{zhou2021informer}.  The twelve utility and UCI feeds
include known-future meteorological covariates (temperature, humidex,
wind, weather code); ETTh1 carries no meteorological covariates,
exposing only the station's oil-temperature channel in their place.
Load profiles span five orders of magnitude across residential,
commercial, industrial, lighting, and transformer monitoring domains.

\paragraph{Cross-domain replication dataset.}
To test whether our design conclusions are specific to electricity, we
additionally evaluate on CAMELS-SPAT~\citep{camels_spat}, a public
hydrometeorological dataset providing hourly streamflow observations and
basin-averaged meteorological forcings for North American catchments.  We
use twelve US catchments, selected \emph{before} any model was trained, on
the basis of contiguous target completeness and hydrologic diversity alone
(Appendix~\ref{app:camels}).  Hourly discharge is the target variable and
ERA5-derived basin-averaged air temperature, humidex and wind speed
(distributed with the dataset) are supplied as covariates, so the method is
evaluated with the exogenous drivers its design assumes rather than in a
covariate-free port.  The same $365+91$~day protocol of
Section~\ref{sec:gap_protocol} is applied unchanged.

\subsection{Gap Simulation Protocol}
\label{sec:gap_protocol}

We create contiguous 91-day (2{,}184-hour) gaps in the test period,
representative of real-world meter outages, rather than the randomly
scattered missing values common in imputation benchmarks.  Extended
outages are operationally more consequential than isolated dropouts,
and they remove the local context that makes interpolation between
nearby observations sufficient.  Each evaluation window
spans $365+91=456$ days: 365~days of observed context followed by
91~masked days.  A sliding-window procedure yields 20--40 instances
per dataset.  All features are z-score normalised using training-set
statistics only.  Appendix~\ref{app:protocol} provides additional
details; shorter gaps (7/30~days) are evaluated in
Appendix~\ref{sec:gap_length}.

\subsection{Baselines and Ablations}
\label{sec:baselines}

We compare against seven external imputation methods spanning
distinct modelling paradigms, plus internal ablation variants.
All receive the same train/test splits, 91-day gap masks, and
14-dimensional input features.  All seven external methods are
evaluated over three seeds (42, 43, 44) and reported in
Table~\ref{tab:ext-baselines-mse}.
\textbf{External baselines.}
Seasonal~Na\"ive copies each missing day from the same week of the
previous year.  The six learned baselines span bidirectional
recurrent imputation (BRITS; \citealp{cao2018brits}, hidden~128,
patience~15), diagonally-masked self-attention (SAITS;
\citealp{du2023saits}, 2~layers, $d_{\text{model}}{=}128$, 4~heads),
conditional score-based diffusion (CSDI; \citealp{tashiro2021csdi},
4~layers, 50~steps), 2D convolution over learned period--frequency
pairs (TimesNet; \citealp{wu2023timesnet}, $\sim$37M~params), a
channel-independent patch Transformer (PatchTST;
\citealp{nie2023patchtst}), and an imputation-specific Transformer
with low-rank temporal attention (ImputeFormer;
\citealp{nie2024imputeformer}).  CSDI is the only other probabilistic
baseline.
\textbf{Internal ablation variants.}
\textbf{VAE+Bridge} (legacy) replaces JEPA with a $\beta$-VAE
(dim${}=32$) under the same masked Transformer bridge.
\textbf{JEPA-Only} decodes the embeddings directly with an MLP,
whereas \textbf{JEPA+Bridge} uses deterministic latent prediction.
\textbf{JEPA+Diffusion} adds DDIM latent diffusion;
\textbf{JEPA+Flow-Matching} replaces it with a 5-step Euler ODE.
\textbf{Full+CQR} uses static conformal calibration, whereas
\textbf{Full+ACI} is the complete system with adaptive $\alpha_t$.

\subsection{Evaluation Metrics}
\label{sec:metrics}

We evaluate imputation with all-feature gap-frame MSE on normalised
$[0,1]$ data for internal ablations, load-only gap MSE on min--max
scaled data for external baselines, and MAPE and RMSE/MAE in physical
units (Appendix~\ref{app:physical_units}).  We assess uncertainty
through empirical coverage against the 95\% target and mean interval
width.  Downstream, we report WMAPE and MAE from
DLinear~\citep{zeng2023dlinear} and TFT~\citep{lim2021tft} trained on
the filled series (Appendix~\ref{sec:forecasting}).

\subsection{Implementation Details}

All models are trained on a single NVIDIA RTX A4000 (16\,GB) with
PyTorch~\citep{paszke2019pytorch} and early stopping.
Table~\ref{tab:hyperparams} (appendix) lists architecture and training
hyperparameters.  The full pipeline uses 1.9--3.2M parameters, trains in
$\sim$8--12~min per dataset, and inference for a 91-day gap takes
6.5--9.0\,ms (Table~\ref{tab:compute}, appendix).

\section{Results}
\label{sec:results}

\subsection{Gap-Filling Accuracy}
\label{sec:mse}

Internal ablations (Table~\ref{tab:mse}, Figure~\ref{fig:mse_bars},
appendix) confirm that both the bridge architecture and JEPA
embeddings contribute: the full pipeline reduces MSE by 53.2\%
vs.\ JEPA-only and 25.1\% vs.\ a legacy VAE+Bridge baseline.

\subsection{Comparison with External Imputation Baselines}
\label{sec:ext-baselines}

To contextualise our results against established imputation
methods, we benchmark the enhanced hourly-conditioned decoder
(Section~\ref{sec:decoder}) against the seven external baselines of
Section~\ref{sec:baselines}, the six learned ones trained via the
PyPOTS library~\citep{du2023pypots}.
Table~\ref{tab:ext-baselines-mse} reports Load-only MSE on the
min-max $[0,1]$ normalised scale, the standard metric used by all
external baselines.

\begin{table}[t]
  \caption{Load-only gap MSE vs.\ external imputation baselines (min-max
           normalised to $[0,1]$ on truth range). Lower is better; bold
           marks the best per dataset.  Values show mean{\scriptsize$\pm$std}
           over 3 random seeds (42, 43, 44); Seasonal Na\"ive is
           deterministic (single run).
           $\dagger$TimesNet diverged on
           UCI\_Elec\_MT\_150 (catastrophic overfitting of 37M-parameter model
           on a small dataset).
           $\ddagger$SEMAResNstar1009 is near-degenerate (MSE rounds to
           zero under min-max normalisation); mean and win counts exclude
           this dataset.}
  \label{tab:ext-baselines-mse}
  \centering
  \scriptsize
  \resizebox{\textwidth}{!}{%
  \begin{tabular}{lrrrrrrr>{\columncolor{oursblue}}r}
    \toprule
    Dataset & Seasonal & BRITS & SAITS & CSDI & TimesNet & PatchTST & ImputeFormer & SPLICE \\
    \midrule
    PepcoCOM           & 0.0121 & 0.032{\scriptsize$\pm$.012} & 0.022{\scriptsize$\pm$.004} & 0.111{\scriptsize$\pm$.001} & 0.069{\scriptsize$\pm$.007} & 0.030{\scriptsize$\pm$.002} & 0.034{\scriptsize$\pm$.000} & \textbf{0.003}{\scriptsize$\pm$.001} \\
    RICom1013          & 0.0300 & 0.034{\scriptsize$\pm$.017} & 0.029{\scriptsize$\pm$.008} & 0.201{\scriptsize$\pm$.002} & 0.088{\scriptsize$\pm$.006} & 0.044{\scriptsize$\pm$.007} & 0.053{\scriptsize$\pm$.001} & \textbf{0.022}{\scriptsize$\pm$.004} \\
    RIInd1014          & \textbf{0.0402} & 0.154{\scriptsize$\pm$.035} & 0.103{\scriptsize$\pm$.012} & 0.200{\scriptsize$\pm$.005} & 0.298{\scriptsize$\pm$.048} & 0.073{\scriptsize$\pm$.006} & 0.058{\scriptsize$\pm$.001} & 0.202{\scriptsize$\pm$.022} \\
    RIRes1012          & 0.0880 & 0.036{\scriptsize$\pm$.022} & 0.049{\scriptsize$\pm$.008} & 0.233{\scriptsize$\pm$.001} & 0.098{\scriptsize$\pm$.026} & 0.062{\scriptsize$\pm$.020} & 0.092{\scriptsize$\pm$.002} & \textbf{0.035}{\scriptsize$\pm$.006} \\
    SEMAResNstar1009$^\ddagger$ & 0.2486 & 0.068{\scriptsize$\pm$.087} & 0.047{\scriptsize$\pm$.015} & 0.266{\scriptsize$\pm$.002} & 0.420{\scriptsize$\pm$.062} & 0.069{\scriptsize$\pm$.018} & 0.103{\scriptsize$\pm$.004} & \textbf{0.000}{\scriptsize$\pm$.000} \\
    UES\_NH\_Med        & 0.4008 & 0.358{\scriptsize$\pm$.164} & 0.390{\scriptsize$\pm$.068} & 0.619{\scriptsize$\pm$.003} & 1.403{\scriptsize$\pm$.419} & 0.415{\scriptsize$\pm$.029} & 0.491{\scriptsize$\pm$.012} & \textbf{0.206}{\scriptsize$\pm$.044} \\
    WCMA1010res        & 0.0243 & 0.022{\scriptsize$\pm$.006} & 0.022{\scriptsize$\pm$.002} & 0.087{\scriptsize$\pm$.000} & 0.062{\scriptsize$\pm$.009} & 0.036{\scriptsize$\pm$.000} & 0.037{\scriptsize$\pm$.000} & \textbf{0.009}{\scriptsize$\pm$.001} \\
    WCMA1011lig        & 0.1614 & 0.096{\scriptsize$\pm$.017} & \textbf{0.027}{\scriptsize$\pm$.004} & 0.982{\scriptsize$\pm$.023} & 0.191{\scriptsize$\pm$.031} & 0.074{\scriptsize$\pm$.016} & 0.145{\scriptsize$\pm$.047} & 0.148{\scriptsize$\pm$.059} \\
    WCMAnatGridRes1004 & 0.0845 & 0.026{\scriptsize$\pm$.006} & 0.019{\scriptsize$\pm$.001} & 0.206{\scriptsize$\pm$.001} & 0.084{\scriptsize$\pm$.010} & 0.033{\scriptsize$\pm$.003} & 0.045{\scriptsize$\pm$.001} & \textbf{0.011}{\scriptsize$\pm$.006} \\
    \midrule
    UCI\_Elec\_MT\_001   & --- & 0.002{\scriptsize$\pm$.000} & 0.002{\scriptsize$\pm$.001} & 0.041{\scriptsize$\pm$.001} & 0.050{\scriptsize$\pm$.002} & \textbf{0.001}{\scriptsize$\pm$.000} & 0.006{\scriptsize$\pm$.000} & 0.003{\scriptsize$\pm$.002} \\
    UCI\_Elec\_MT\_150   & --- & 0.010{\scriptsize$\pm$.000} & 0.011{\scriptsize$\pm$.001} & 0.142{\scriptsize$\pm$.000} & 6.659$^\dagger${\scriptsize$\pm$1.50} & 0.007{\scriptsize$\pm$.001} & 0.023{\scriptsize$\pm$.000} & \textbf{0.005}{\scriptsize$\pm$.000} \\
    UCI\_Elec\_MT\_320   & --- & 0.017{\scriptsize$\pm$.010} & 0.043{\scriptsize$\pm$.008} & 0.122{\scriptsize$\pm$.001} & 0.067{\scriptsize$\pm$.001} & 0.046{\scriptsize$\pm$.006} & 0.058{\scriptsize$\pm$.001} & \textbf{0.008}{\scriptsize$\pm$.001} \\
    ETTh1              & --- & 0.025{\scriptsize$\pm$.008} & 0.049{\scriptsize$\pm$.001} & 0.078{\scriptsize$\pm$.000} & 0.056{\scriptsize$\pm$.005} & 0.049{\scriptsize$\pm$.000} & 0.050{\scriptsize$\pm$.000} & \textbf{0.022}{\scriptsize$\pm$.002} \\
    \midrule
    Mean$^\ddagger$ & 0.105{\scriptsize$\pm$.043} & 0.068{\scriptsize$\pm$.028} & 0.064{\scriptsize$\pm$.029} & 0.252{\scriptsize$\pm$.076} & 0.760{\scriptsize$\pm$.524} & 0.073{\scriptsize$\pm$.031} & 0.091{\scriptsize$\pm$.036} & \textbf{0.056}{\scriptsize$\pm$.022} \\
    Wins$^\ddagger$ & 1 & 0 & 1 & 0 & 0 & 1 & 0 & \textbf{9} \\
    \bottomrule
  \end{tabular}%
  }
\end{table}

\textbf{Analysis.}
SPLICE records the lowest mean Load-only MSE and wins 9 of 12
non-degenerate datasets, with no baseline winning more than one
(Table~\ref{tab:ext-baselines-mse};
Figure~\ref{fig:baseline_comparison}, appendix).  Three of the four
public benchmarks (UCI + ETTh1) are among those wins
(Figure~\ref{fig:uci_reconstruction}, appendix).
MAPE on the original scale confirms the pattern (lowest on 7 of 11
reportable datasets; Table~\ref{tab:ext-baselines-mape}, appendix).
SAITS retains an advantage on WCMA1011lig, and PatchTST on
UCI\_Elec\_MT\_001, where bridge embedding recovery loses fine-grained
structure.  Seasonal~Na\"ive wins on RIInd1014 for a different
reason: it copies the previous year exactly.
\textbf{Statistical significance.}
One-sided Wilcoxon signed-rank tests on per-dataset, three-seed means
show significant improvements over CSDI and TimesNet (both
$p{<}0.001$) and ImputeFormer ($p{=}0.021$).  Differences from
Seasonal~Na\"ive, SAITS, BRITS, and PatchTST do not reach $p{<}0.05$
(Appendix~\ref{app:significance}).  SPLICE outperforms PatchTST on
9 of 12 datasets and ImputeFormer on 10 of 12.
\textbf{Latent-space reconstruction.}
SPLICE reaches this accuracy while compressing all eight channels into
a single 64-dimensional embedding, whereas every baseline models the
load channel in observation space; Figure~\ref{fig:gap_comparison}
(appendix) shows the resulting reconstructions on three representative
profiles.  Compression to a shared latent does not reduce load
reconstruction accuracy on these datasets.  Decoder capacity and the
Load-weighted loss contribute to this result
(Appendix~\ref{sec:decoder_ablation}).

\subsection{Conformal Prediction Intervals}
\label{sec:conformal}

Table~\ref{tab:conformal} (appendix) compares static CQR and adaptive
conformal inference~(ACI) across the nine proprietary datasets
($1 - \alpha = 0.95$ target).
ACI fixes coverage violations on two critical datasets: RIInd1014
(86.0\%~$\to$~93.5\%, +7.5\,pp) and UES\_NH\_Med
(89.9\%~$\to$~93.4\%, +3.5\,pp). On datasets where static CQR
already exceeded the target (e.g.\ WCMA1011lig at 97.2\%), ACI tightens
the bands (95.4\%, $-$1.8\,pp), recovering efficiency without losing
validity. Figures~\ref{fig:aci_trajectories} and~\ref{fig:conformal_bands}
(appendix) show the adaptive $\alpha_t$ trajectories and resulting
prediction bands.

\subsection{Distributional Quality and Uncertainty Calibration}
\label{sec:distributional}

We complement point accuracy with distributional metrics: CRPS,
interval coverage (Section~\ref{sec:conformal}), and boundary
smoothness (Appendix~\ref{app:boundary_smoothness}).
\textbf{CRPS.}
We compute CRPS via the energy form~\citep{gneiting2007strictly}
using a 20-member ensemble generated by perturbing bridge embeddings
with Gaussian noise ($\sigma{=}0.15$) and decoding each independently.
SPLICE-ensemble achieves the lowest average CRPS (0.182;
Table~\ref{tab:crps}, appendix), outperforming SAITS
(0.197, $-7.8\%$) and its own deterministic MAE
(0.198; $-10.5\%$ mean per-dataset reduction), showing that latent
perturbations provide useful predictive variation.
SPLICE is the \emph{only} evaluated method that provides conformally
calibrated prediction intervals; none of the baselines, including the
probabilistic CSDI, provides coverage guarantees
(Table~\ref{tab:method_comparison}, appendix).
Figure~\ref{fig:crps_comparison} shows per-dataset CRPS, and
Figure~\ref{fig:crps_improvement} quantifies the ensemble
improvement over deterministic prediction.
\textbf{Ablations and sensitivity.}
Appendices~\ref{sec:diffusion_ablation}--\ref{sec:incremental_ablation}
confirm FM-A wins 10/13 datasets ($-8.7\%$ MSE vs.\ DDIM, $10\times$
fewer steps), and the enhanced decoder + ensembling each contribute
cumulatively ($-9.3\%$ MSE).  SPLICE maintains the lowest mean rank
(2.33--2.67) across 7/30/91-day gaps with 17/32 total wins (53\%;
Appendix~\ref{sec:gap_length}).
A latent-dimension and EMA-decay sweep confirms $d_{\mathrm{repr}}{=}64$
is at or near optimal on the public benchmarks
(Appendix~\ref{sec:repr_dim}), and re-running validation on a disjoint
pre-gap window leaves the overall ranking essentially unchanged
(median $\Delta{=}0.0\%$; Appendix~\ref{sec:disjoint}).

\subsection{Cross-Domain Replication on Hydrological Data}
\label{sec:camels}

To check that the architectural conclusions above are not artefacts of
electricity load, we replicate the three controlled swaps on twelve
CAMELS-SPAT catchments with the pipeline unchanged
(Appendix~\ref{app:camels}), scored both on gap-MSE and on
hydrology-standard criteria (KGE, NSE, FHV/FLV).  The ordering
reproduces: the predictive JEPA encoder beats both a contrastive and a
reconstructive alternative, and the 5-step flow-matching sampler beats
the 50-step DDIM sampler (Table~\ref{tab:camels}).  The three rows use
different error metrics and are each reported only against their own
control, so their magnitudes are not comparable to one another.

\begin{table}[t]
  \caption{Cross-domain replication on twelve CAMELS-SPAT catchments
    (seed~42).  The encoder swap is scored on the target-only gap-MSE and
    the remaining two rows on the all-feature gap-frame MSE, so the
    reductions are not comparable across rows
    (Appendix~\ref{app:camels}).}
  \label{tab:camels}
  \centering\small
  \begin{tabular}{lrr}
    \toprule
    Controlled swap & Wins & Median reduction \\
    \midrule
    JEPA vs.\ TS2Vec (contrastive)               & 9/12  & 16.2\% \\
    JEPA+Bridge vs.\ VAE+Bridge (reconstructive) & 10/12 & 13.2\% \\
    \midrule
    Flow Matching (5-step) vs.\ DDIM (50-step)   & 8/12  & 12.5\% \\
    \bottomrule
  \end{tabular}
\end{table}

\section{Transferability of JEPA Representations}
\label{sec:transfer}

If the JEPA encoder learns representations that generalise across
heterogeneous domains, the modular SPLICE pipeline can be deployed on
new feeds with minimal adaptation, amortising the cost of encoder
training.
\textbf{Pooled encoder.}
We train a single JEPA encoder on all nine proprietary datasets
simultaneously (``pooled JEPA'').  Each dataset is projected to the
same eight common features and normalised independently; the encoder
architecture is identical to the per-dataset variant ($\sim$1.06M
parameters, $d_{\mathrm{repr}}{=}64$).  The resulting latent space
clusters by dataset identity \emph{and} season
(Figure~\ref{fig:transfer_umap}, Appendix~\ref{app:jepa_viz}),
confirming that the encoder captures both domain-specific signatures
and shared temporal structure without supervision.
\textbf{Adaptation spectrum.}
We evaluate on four held-out datasets (UCI\_001, UCI\_150, UCI\_320, ETTh1),
none seen during pooled training, under three adaptation regimes of
decreasing target-domain supervision. Regime~A (oracle) uses a fully
per-dataset pipeline, with a JEPA, bridge, and decoder. Regime~B freezes a pooled JEPA encoder
and trains a dataset-specific bridge and decoder on target data. Regime~C is fully zero-shot at the
representation level: both the pooled encoder and bridge are frozen,
and only a lightweight decoder is trained on target data.

\begin{table}[t]
  \caption{Transfer evaluation: gap-inpainting MSE on four held-out
  datasets under three adaptation regimes (single seed).  Bold marks
  the best result per dataset; percentages are relative to
  Regime~A.}
  \label{tab:transfer}
  \centering
  \small
  \begin{tabular}{lccc}
    \toprule
    Dataset & Regime~A & Regime~B & Regime~C \\
    \midrule
    UCI\_001 & 0.2072 & 0.2889 (\textcolor{red}{+39\%})
             & \textbf{0.1967} (\textcolor{teal}{$-$5\%}) \\
    UCI\_150 & 0.0865 & \textbf{0.0804} (\textcolor{teal}{$-$7\%})
             & 0.2036 (\textcolor{red}{+135\%}) \\
    UCI\_320 & 0.0911 & \textbf{0.0888} (\textcolor{teal}{$-$3\%})
             & 0.1507 (\textcolor{red}{+65\%}) \\
    ETTh1    & 0.1875 & \textbf{0.1800} (\textcolor{teal}{$-$4\%})
             & 0.2013 (\textcolor{red}{+7\%}) \\
    \midrule
    Mean     & 0.1431 & 0.1595 ({+11\%})
             & 0.1881 ({+31\%}) \\
    Median   & 0.1393 & \textbf{0.1346} (\textcolor{teal}{$-$3\%})
             & 0.1990 ({+43\%}) \\
    \bottomrule
  \end{tabular}
\end{table}

\textbf{Results.}
Regime~B, which requires only bridge and decoder fine-tuning
($\sim$60\,s), matches or \emph{improves upon} the per-dataset oracle
on three of four held-out domains (Table~\ref{tab:transfer}).  We
attribute this to a regularisation effect: the pooled encoder, trained
on nine diverse feeds, produces embeddings that are more robust to
distribution shift than per-dataset encoders.  The median improves
while the mean does not, because of the outlier on UCI\_001, a domain
with low absolute load ($\bar{y}{=}1.6$\,kW) where small prediction
errors are amplified in relative terms.
Regime~C demonstrates that fully zero-shot transfer is viable when
domain alignment is favourable: on UCI\_001, the pooled bridge
\emph{outperforms} the oracle with only 2\,s of decoder training.  On
the remaining datasets, the gap widens, indicating that bridge
adaptation remains necessary when the target load profile differs
substantially from the training pool.  We note that Regime~C operates
at a slight disadvantage: the pooled bridge was trained with
17-dimensional conditioning (including Cloud and Precipitation from
the training domains), whereas UCI and ETTh1 lack these features and
receive zero-filled placeholders at evaluation time.
These results suggest a practical deployment strategy: freeze the
pooled JEPA encoder and fine-tune only the bridge and decoder on a
small amount of target data.  As the diversity of the training pool
grows, the zero-shot frontier is expected to expand further.

\textbf{Few-shot budget.}
Across fine-tuning budgets of 7 to 365~days on the four
held-out sets (three seeds each), 90~days of target data suffice to
match or beat SAITS and BRITS trained on the full 1\,241-day history
(e.g.\ UCI\_150: $0.011{\pm}0.000$ vs.\ SAITS $0.011$; UCI\_320:
$0.009{\pm}0.001$ vs.\ BRITS $0.020$; Appendix~\ref{app:fewshot}).

\section{Discussion}
\label{sec:discussion}

\begin{figure}[t]
  \centering
  \includegraphics[width=\linewidth]{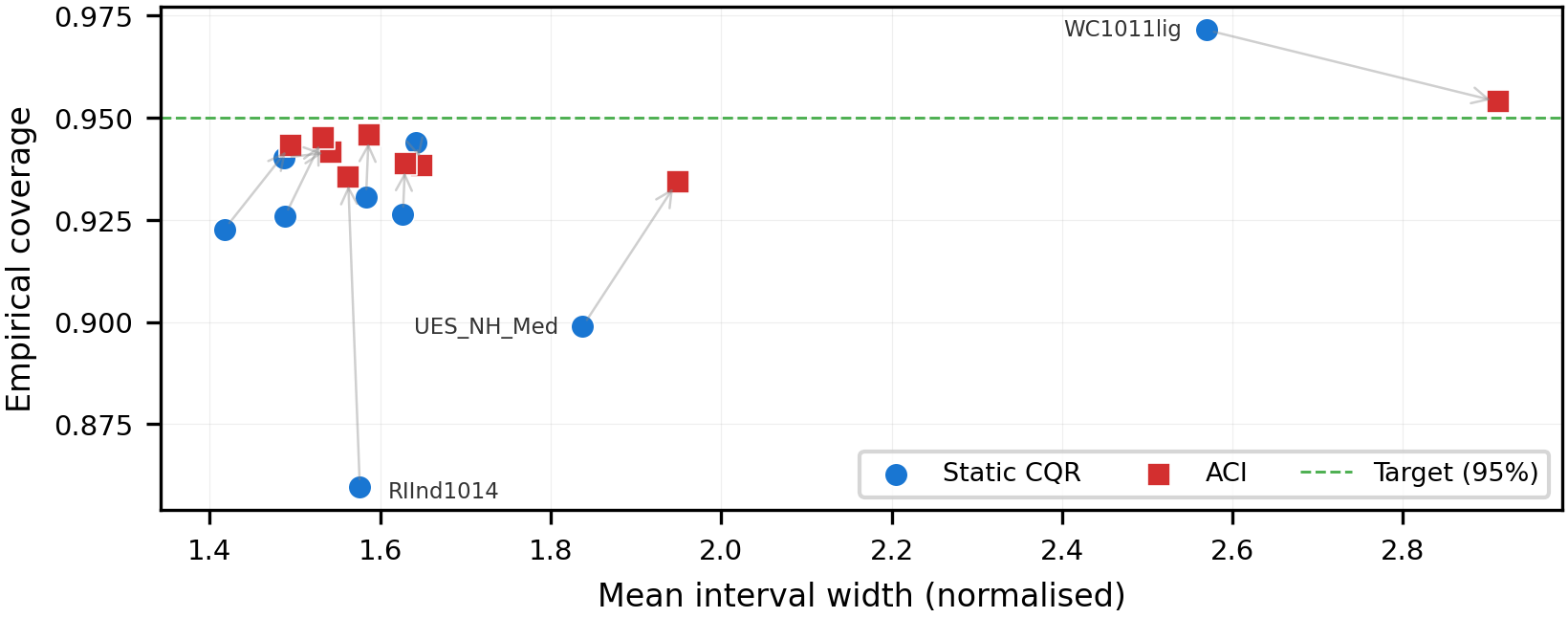}
  \caption{Coverage vs.\ interval width. Each point is one dataset; grey
  arrows show the shift from static CQR (blue circles) to ACI (red
  squares). Green dashed line: 95\% target.}
  \label{fig:pareto}
\end{figure}

\textbf{Design principle: fidelity is governed primarily by predictive
latent geometry, and a short sampler suffices.}
Under matched downstream components, replacing the JEPA latent
degrades reconstruction: a reconstructive VAE latent costs $25.1\%$
MSE (Table~\ref{tab:mse}) and a contrastive TS2Vec latent loses on 9
of the twelve load datasets (Appendix~\ref{sec:encoder_swap}).  Both
substitutions also displace the timing of the daily peak: the
contrastive latent on 8 of 11 datasets and on two where gap-MSE
favours TS2Vec (Appendix~\ref{app:peak_timing}), the reconstructive
latent on 6 of 8 by a median of sixteen minutes
(Appendix~\ref{app:peak_timing_vae}), placing the encoder's advantage
on an operational criterion that point accuracy does not expose.  Once the predictive
representation is fixed, a 5-step flow-matching sampler matches or
improves on a 50-step DDIM ($-8.7\%$ MSE,
Table~\ref{tab:gen-ablation}): high-fidelity stochastic reconstruction
does not require a long reverse-diffusion trajectory.  The same
encoder and sampler ordering recovers on twelve hydrological
catchments, outside electricity-load data
(Section~\ref{sec:camels}).
We characterise this geometry behaviourally, through the traversal and
transfer behaviour the latent space supports, rather than through a
standalone geometric statistic: an input-space similarity measure
rewards an encoder for preserving input neighbourhoods, whereas the
predictive objective reorganises the space around predictability.
The ablations also locate the division of labour: the encoder learns
predictive states, the bridge reconstructs their missing trajectory,
the decoder restores hourly detail, and the generative backend
diversifies them.  The design rule is to learn the predictive state
space first, then use the least costly backend that supplies the
required trajectory diversity.
\textbf{Why adaptive conformal matters.}
As argued in Section~\ref{sec:introduction}, under-coverage is the
operationally consequential failure direction, yet static CQR
under-covers by up to 7.5~pp under non-stationarity;
ACI recovers 93--95\% coverage while tightening over-wide intervals
(Figure~\ref{fig:pareto}).
\textbf{Modularity and practical implications.}
A shared 64-dimensional interface connects the JEPA encoder, latent
bridge, decoder, and ACI, carries all eight channels without
sacrificing Load fidelity, and lets components be replaced without
upstream retraining.  It supports three regimes: a deterministic
bridge for point reconstruction; latent perturbations for calibrated
intervals and CRPS ensembles
(Sections~\ref{sec:conformal}--\ref{sec:distributional}); and a
generative backend for trajectory diversity, where flow matching stays
closer to the deterministic solution than DDIM
(Appendix~\ref{sec:diffusion_ablation}).  The pooled encoder is
likewise reusable on unseen datasets with only bridge and decoder
adaptation (Section~\ref{sec:transfer}).
\textbf{Limitations.}
Gaps shorter than one week or longer than three months are
uncharacterised.  The full diffusion backend requires 50 reverse steps,
which progressive distillation or consistency models could cut.
CAMELS-SPAT (Section~\ref{sec:camels}) replicates the encoder and
sampler ordering on hourly streamflow, but gas, renewable generation,
other physical processes, and climates outside the northeastern US,
Portugal and China remain untested.
In the transfer study only the few-shot experiment is replicated
across seeds (Appendix~\ref{app:fewshot}); the zero-shot and
bridge-only results are single-seed.
Main results (Table~\ref{tab:ext-baselines-mse}) use three seeds
with standard deviations.  SPLICE achieves the lowest
mean Load-only MSE, winning 9 of 12 non-degenerate datasets, and the
lowest mean CRPS.  The gap-MSE differences from BRITS and SAITS do
not reach $p{<}0.05$.
On bimodal or zero-inflated datasets (WCMA1011lig, RIInd1014) the
64-dimensional interface costs fine structure that observation-space
methods model directly, and end-to-end fine-tuning of the bridge and
decoder may reduce this cost.
Two assumptions bound deployment.  The ACI guarantee
(Appendix~\ref{app:aci}) is marginal: time-averaged miscoverage
converges to $\alpha$ without certifying coverage at a fixed horizon,
dataset or covariate value.  The second is shared with BRITS and
SAITS: covariates must be known or reliably forecastable across the
gap; where they are not, their uncertainty has to be propagated into
the intervals, which we do not model.
\textbf{Broader impact.}
Reliable gap-filling with calibrated uncertainty is critical for grid
planning and demand-side management; over-confident imputations can
lead to under-provisioning or missed demand response events.
Data availability and code release details are provided in
Appendix~\ref{app:data_availability}.

\section{Conclusion}
\label{sec:conclusion}

We presented SPLICE, an operational instantiation of JEPA for
long-horizon time-series inpainting that couples a conditional latent
bridge, an hourly-conditioned decoder, and ACI.
Controlled ablations establish a design principle: under a matched
bridge and decoder, a reconstructive or contrastive latent underperforms
JEPA on reconstruction, whereas once the representation is fixed a
5-step flow sampler matches or improves on a 50-step DDIM.  The same
ordering holds on 12 hydrological catchments
(Section~\ref{sec:camels}).
Realised as a system, SPLICE attains the lowest mean Load-only MSE
across thirteen datasets at 91-day gaps, winning 9 of 12
non-degenerate datasets while jointly reconstructing all eight input
channels in a shared latent space; against the strongest deterministic
imputers the margin is favourable but not statistically separated at
this sample size (Section~\ref{sec:results}).  The clearer advantage
is in calibrated uncertainty, which those imputers do not provide:
stochastic latent ensembles attain the lowest mean CRPS on the nine
proprietary feeds, and ACI holds 93--95\% empirical coverage where a
static calibrator under-covers by up to 7.5~pp.
The stage-decoupled design lets each component be replaced without
upstream retraining, and the transfer study
(Section~\ref{sec:transfer}) exploits this: a pooled JEPA encoder
trained on nine feeds transfers to four unseen
domains, matching or exceeding per-dataset oracles on three with only
bridge and decoder adaptation.
Future work includes gap lengths outside the 7--91-day range studied
here, end-to-end bridge--decoder fine-tuning for bimodal datasets,
consistency distillation for single-step flow matching, replication
beyond energy and hydrology, and scaling the pooled training pool to
widen the zero-shot transfer frontier.

\appendix

\section{ACI Coverage Guarantee}
\label{app:aci}

For completeness, we restate the coverage guarantee of Adaptive
Conformal Inference~\citep{gibbs2021aci} and make explicit the
convergence rate that underpins our experimental design.

\begin{proposition}[Asymptotic coverage, Gibbs \& Cand\`{e}s 2021]
Let $\{\alpha_t\}_{t \geq 1}$ evolve according to
Eq.~\eqref{eq:aci} with learning rate $\gamma > 0$ and clamping
to $[\varepsilon, 1 - \varepsilon]$.  Then, regardless of the
data-generating process,
\begin{equation}
  \Bigl|\frac{1}{T}\sum_{t=1}^{T} \mathrm{err}_t - \alpha\Bigr|
  \;\leq\; \frac{1 - 2\varepsilon}{\gamma\, T},
  \label{eq:aci_bound}
\end{equation}
so $\lim_{T \to \infty} T^{-1}\sum_t \mathrm{err}_t = \alpha$.
\end{proposition}

\begin{proof}[Proof sketch]
The update rule ensures $|\alpha_{t+1} - \alpha_t| \leq \gamma$,
and clamping keeps $\alpha_t \in [\varepsilon, 1 - \varepsilon]$.
Define $\Delta_t = \alpha_t - \alpha$.  Summing the update over
$t = 1, \dots, T$ and rearranging:
$\sum_{t=1}^{T} (\mathrm{err}_t - \alpha) =
(\alpha_1 - \alpha_{T+1})/\gamma$,
whence $\bigl|\frac{1}{T}\sum_t \mathrm{err}_t - \alpha\bigr|
\leq |\alpha_1 - \alpha_{T+1}| / (\gamma T)
\leq (1 - 2\varepsilon) / (\gamma T)$.
\end{proof}

In our experiments $\gamma = 0.01$ and $\varepsilon = 0.001$,
giving a worst-case deviation of $\leq 99.8 / T$.  This
distribution-free bound is conservative: it falls below $0.01$ only
for $T > 9{,}980$, and over the $2{,}184$ hourly steps of a 91-day
gap it guarantees only $\leq 0.046$.  The worst case requires an
adversarial error sequence; measured coverage
(Table~\ref{tab:conformal}) is substantially closer to nominal.

\section{Training and Evaluation Protocol}
\label{app:protocol}

\paragraph{Feature normalisation.}
All input features (Load, Temperature, Humidex, Wind, etc.)\ are
z-score normalised using per-feature mean and standard deviation
computed on the training split only.  The same statistics are applied
to the validation and test data to prevent information leakage.
Conditioning features for the decoder (hourly weather and calendar
encodings) undergo a separate z-score normalisation with statistics
fitted on the training portion of the daily conditioning matrix.

\paragraph{Train/validation split.}
The most recent 15\% of calendar days form the validation set;
all earlier days constitute the training set.  This temporal split
ensures no future information leaks into training, mirroring the
deployment scenario of retrospective gap-filling.  As a robustness
check, Appendix~\ref{sec:disjoint} confirms that the early-stopping
\emph{epoch}, chosen on a validation region whose trailing edge can
coincide with the start of the gap, confers no advantage: re-running
the full pipeline with a strictly disjoint pre-gap validation split
leaves performance statistically indifferent (median gap-MSE change
$0.0\%$).

\paragraph{Sliding window construction.}
Each evaluation window spans $365 + 91 = 456$ consecutive days:
the first 365~days provide observed context and the final 91~days
are masked.  A stride-1 sliding window generates all valid windows
whose start index satisfies
$\mathrm{start} + 365 \geq 0.85 \cdot n_{\mathrm{days}}$,
producing 20--40 evaluation windows per dataset depending on the
available data span.

\paragraph{Decoder training.}
The hourly-conditioned decoder is trained on the training split
with the JEPA encoder frozen.  For each training sample, the
encoder produces a 64-dimensional embedding for each day, which
the decoder maps to $24\!\times\!d_{\mathrm{features}}$ hourly
predictions.  With probability 0.5, Gaussian noise ($\sigma = 0.15$)
is added to the embedding before decoding, forcing the decoder to
be robust to bridge reconstruction errors.  The Load-weighted loss
amplifies the gradient contribution of the Load feature by a factor
of~5.

\paragraph{Conformal calibration protocol.}
For standard CQR, all validation windows except the last serve as
the calibration set; the final window is used for evaluation.
For ACI, the first 50\% of validation windows initialise the base
quantile $\hat{q}$ using $S{=}50$ DDIM samples per window; the
remaining windows are processed in temporal order with $\alpha_t$
updates after each window (using $S{=}20$ samples).  Coverage and
interval width are reported for the online phase only.

\section{Classifier-Free Guidance Scale Sensitivity}
\label{app:guidance}

Table~\ref{tab:guidance_sweep} reports the gap-frame MSE for guidance
scales $w \in \{1, 2, 3, 4, 5\}$ across the nine proprietary datasets.  The
optimal scale is dataset-dependent: low-variance feeds
(RIInd1014, SEMAResNstar1009) prefer $w{=}1$ (no guidance), while
volatile feeds (UES\_NH\_Med) benefit from strong guidance ($w{=}5$).
Commercial and residential loads cluster around $w{=}2$--$4$.
The sensitivity is moderate: across all datasets, the worst scale
degrades MSE by at most $\sim$37\% relative to the best, and most
datasets show $<$10\% variation between the top two scales.

\begin{table}[H]
  \caption{Guidance-scale sweep: gap-frame MSE (mean $\pm$ std over
           DDIM samples). Bold marks the per-dataset optimum.
           All values on normalised $[0,1]$ scale.}
  \label{tab:guidance_sweep}
  \centering
  \small
  \begin{tabular}{lrrrrr}
    \toprule
    Dataset & $w{=}1$ & $w{=}2$ & $w{=}3$ & $w{=}4$ & $w{=}5$ \\
    \midrule
    PepcoCOM          & .291  & .288  & \textbf{.283} & .289  & .289 \\
    RICom1013         & .343  & \textbf{.330} & .342  & .355  & .364 \\
    RIInd1014         & \textbf{.508} & .554  & .603  & .652  & .698 \\
    RIRes1012         & .363  & \textbf{.337} & .339  & .348  & .362 \\
    SEMAResNstar1009  & \textbf{.257} & .271  & .286  & .311  & .335 \\
    UES\_NH\_Med      & .525  & .443  & .402  & .397  & \textbf{.395} \\
    WCMA1010res       & .322  & .314  & .304  & \textbf{.294} & .300 \\
    WCMA1011lig       & .389  & .365  & \textbf{.363} & .364  & .371 \\
    WCMAnatGridRes1004& .420  & .414  & \textbf{.411} & .424  & .437 \\
    \bottomrule
  \end{tabular}
\end{table}

\section{RMSE and MAE in Physical Units}
\label{app:physical_units}

Table~\ref{tab:physical_units} reports gap-reconstruction error in
original measurement units, derived from the min-max normalised
Load-only MSE (Table~\ref{tab:ext-baselines-mse}) and the observed Load
range of each dataset.  RMSE is computed as
$\mathrm{RMSE} = \sqrt{\mathrm{MSE}_{[0,1]}} \times (\max - \min)$;
MAE is approximated under a Gaussian error assumption as
$\mathrm{MAE} \approx \sqrt{2/\pi}\;\mathrm{RMSE}$.

\begin{table}[H]
  \caption{Gap-reconstruction error in physical units.
           Units vary by dataset: kW for utility-metered feeds,
           kWh for hourly-integrated feeds.}
  \label{tab:physical_units}
  \centering
  \small
  \begin{tabular}{lrrl}
    \toprule
    Dataset & RMSE & MAE & Unit \\
    \midrule
    PepcoCOM                &    11.13 &     8.88 & kW \\
    RICom1013               &   144.45 &   115.25 & kW \\
    RIInd1014               &    50.97 &    40.66 & kW \\
    RIRes1012               &    44.49 &    35.50 & kW \\
    SEMAResNstar1009        &  $<$0.01 &  $<$0.01 & kWh \\
    UES\_NH\_Med            &    17.29 &    13.80 & kWh \\
    WCMA1010res             &    23.33 &    18.61 & kWh \\
    WCMA1011lig             &     5.20 &     4.15 & kWh \\
    WCMAnatGridRes1004      &    46.27 &    36.92 & kW \\
    \midrule
    UCI\_Elec\_MT\_001      &     1.51 &     1.20 & kWh \\
    UCI\_Elec\_MT\_150      &     3.92 &     3.13 & kWh \\
    UCI\_Elec\_MT\_320      &    18.55 &    14.80 & kWh \\
    \bottomrule
  \end{tabular}
\end{table}

\section{UCI Electricity Dataset Preprocessing}
\label{app:uci_preprocessing}

We use the UCI ElectricityLoadDiagrams20112014
dataset~\citep{uci_electricity}, which contains 15-minute electricity
consumption readings (in kWh) for 370 Portuguese clients spanning
2011--2014.  Three clients were selected to span the variance spectrum:
MT\_001 (low variance, $\sigma = 5.6$\,kWh),
MT\_150 (medium, $\sigma = 13.2$\,kWh), and
MT\_320 (high, $\sigma = 28.0$\,kWh).

\paragraph{Temporal aggregation.}
The native 15-minute resolution is resampled to hourly by taking the
mean of each four-reading block, matching the hourly granularity of
the proprietary utility datasets.

\paragraph{Weather covariates.}
Hourly ERA5 reanalysis weather for Lisbon ($38.72^{\circ}$N,
$9.14^{\circ}$W) was retrieved via the Open-Meteo historical
API~\citep{openmeteo}: temperature, dew point, relative humidity,
wind speed, and WMO weather code.  Temperature and dew point are
converted from Celsius to Fahrenheit; wind from km/h to mph; and
Humidex is approximated via the Canadian formula, all for
consistency with the proprietary pipeline feature set.

\paragraph{Output format.}
Each client is exported as a semicolon-delimited CSV with columns
\texttt{Date;\allowbreak Heure;\allowbreak Temperature;\allowbreak Humidex;\allowbreak Weather;\allowbreak Wind;\allowbreak Load},
identical to the proprietary data format.  This allows the entire
training, bridge, and evaluation pipeline to process UCI clients
without any code modifications.

\section{Cross-Domain Replication on CAMELS-SPAT}
\label{app:camels}

\paragraph{Dataset.}
CAMELS-SPAT~\citep{camels_spat} is a public large-sample
hydrometeorological dataset providing streamflow observations,
basin-averaged meteorological forcings and geospatial attributes for
North American catchments.  We use its hourly US subset.  Unlike a
covariate-free port of an unrelated series, CAMELS-SPAT distributes
the exogenous drivers alongside the target, so SPLICE can be evaluated
under the conditions its design assumes.

\paragraph{Target and covariates.}
Hourly discharge is the target variable.  ERA5-derived basin-averaged
air temperature, humidex and wind speed (supplied with the
dataset) are mapped onto the covariate slots used by the electricity
pipeline, together with the deterministic calendar features.  No
architectural or hyperparameter change was made: the same daily
$24\times n$ frames, 365-day context, 91-day gap, encoder, bridge,
decoder and samplers are used throughout.

\paragraph{Catchment selection.}
Twelve catchments were selected \emph{before} any model was trained, using
two criteria only: (i) the length of the longest contiguous run with no
missing target values, which must exceed the $365+91$~days the protocol
requires, and (ii) coverage of the CAMELS-SPAT catchment size classes
(headwater, meso-scale, macro-scale) and of distinct climate regions.  No
model output influenced the selection.  Table~\ref{tab:camels_gauges}
lists the resulting set.

\begin{table}[H]
  \caption{The twelve CAMELS-SPAT catchments used for the cross-domain
    replication.  Identifiers are CAMELS-SPAT station codes; area is the
    basin drainage area and class the CAMELS-SPAT scale category, while
    the coordinates give the location of the gauging station supplying
    the discharge series.  The run column is the longest contiguous
    target-complete run, which determined eligibility.}
  \label{tab:camels_gauges}
  \centering\small
  \begin{tabular}{llrrrrl}
    \toprule
    Catchment & Class & Area (km$^2$) & Lat. & Lon. & Run (d) & Run start \\
    \midrule
    USA\_02296500 & meso-scale  &  691.4 & 27.375 & $-$81.796 & 4590 & 2004-12-24 \\
    USA\_11532500 & macro-scale & 1584.7 & 41.792 & $-$124.076 & 4041 & 1991-10-20 \\
    USA\_11284400 & headwater   &   41.7 & 37.842 & $-$120.185 & 3692 & 2006-08-05 \\
    USA\_08023080 & meso-scale  &  183.5 & 31.979 & $-$93.934 & 3403 & 2007-10-02 \\
    USA\_12143600 & headwater   &  168.3 & 47.453 & $-$121.718 & 2963 & 1996-02-05 \\
    USA\_08158810 & headwater   &   31.9 & 30.155 & $-$97.940 & 2494 & 2016-03-06 \\
    USA\_08164000 & macro-scale & 2081.4 & 28.960 & $-$96.686 & 2434 & 2012-08-23 \\
    USA\_09505200 & meso-scale  &  284.1 & 34.675 & $-$111.672 & 2249 & 2016-05-06 \\
    USA\_02102908 & headwater   &   20.3 & 35.183 & $-$79.178 & 2193 & 2007-06-07 \\
    USA\_07056000 & macro-scale & 2148.8 & 35.983 & $-$92.748 & 1600 & 2012-01-23 \\
    USA\_02231342 & meso-scale  &   96.3 & 27.569 & $-$80.796 & 1417 & 2004-06-05 \\
    USA\_01547700 & meso-scale  &  112.3 & 41.060 & $-$77.606 & 1071 & 2016-02-18 \\
    \bottomrule
  \end{tabular}
\end{table}

\paragraph{Protocol and comparisons.}
Each catchment is treated as an independent dataset and run through the
unmodified pipeline with a contiguous 91-day gap preceded by a 365-day
context, exactly as in Section~\ref{sec:gap_protocol}.  Three
controlled swaps are evaluated, each changing a single component while
holding the rest of the pipeline fixed, and each reported in
Table~\ref{tab:camels}: (i) the encoder swap replaces the JEPA encoder
with TS2Vec and is scored on the target-only hourly gap-MSE produced by
the deterministic bridge, matching the encoder-swap protocol of
Appendix~\ref{sec:encoder_swap}; (ii) the reconstructive comparison
replaces the JEPA encoder with the legacy VAE latent of
Table~\ref{tab:mse} and is scored on the all-feature gap-frame MSE;
(iii) the sampler swap replaces the 50-step DDIM sampler with the
5-step FM-A sampler on the same all-feature gap-frame MSE, matching
Appendix~\ref{sec:gen_ablation}.  Because the three rows use different
error definitions (target-only versus all-feature) and different
decoders, their magnitudes are not comparable across rows, and each is
reported only against its own control.  Results are single-seed
(seed~42).

\paragraph{Scoring under hydrological criteria.}
The gap-MSE of Table~\ref{tab:camels} is the objective the pipeline is
trained and deployed on.  Hydrology evaluates reconstructions against a
different family of criteria, several of which rank catchments quite
differently from MSE, so we re-score the same stored predictions under
the standard set defined in Table~\ref{tab:hydro_defs}, changing
nothing but the criterion.  Table~\ref{tab:camels_hydro} reports, for
each criterion, the number of catchments on which the JEPA arm is
closer to observation, a two-sided exact sign test and the median
paired difference, together with the Spearman rank correlation $\rho$
between that criterion's per-catchment paired difference and the
gap-MSE paired difference, which measures how much freedom the
criterion has to reorder the panel.

Both encoder comparisons reproduce the Table~\ref{tab:camels} ordering
under this second scoring.  In the encoder swap, six of the ten
criteria favour the JEPA arm, two are even, and two (the correlation
coefficient and log-NSE) favour TS2Vec; the reconstructive swap has the
same $6/2/2$ profile.  This holds for criteria whose ordering of the
catchments is close to independent of the MSE ordering ($\rho$ between
$-0.08$ and $-0.27$) and which were therefore free to reverse it.  The
correlation coefficient is the clearest of the two exceptions,
favouring TS2Vec on nine of twelve catchments: TS2Vec is trained so
that augmented views of the same window agree, which is the invariance
structure $r$ measures, and leading on that criterion is a direct
expression of its objective.  We report direction and concordance
across criteria rather than significance, since hydrological skill is
not itself a claim of this paper.

\paragraph{Conventions.}
FHV and FLV are computed on the flow-duration curve, that is on the
sorted flow distributions, and are therefore timing-insensitive: they
compare how much water a reconstruction places in the high- and
low-flow regimes, not when it places it.  FLV is reported as a plain
relative bias over the lower $30\%$ of the curve.  Logarithmic criteria
use the convention $\hat{Q}\leftarrow\max(\hat{Q},0)$ followed by an
offset $\varepsilon=0.01\,\bar{Q}$ added to both series; three
catchments whose low-flow denominator falls below the numerical guard
are excluded from the FLV row only.  The criteria differ in their ideal
value, so paired differences are signed consistently: $d=a-b$ for
higher-is-better criteria (NSE, log-NSE, KGE, KGE$'$, $r$),
$d=|b-1|-|a-1|$ for the ratios ($\alpha$, $\beta$, $\gamma$) and
$d=|b|-|a|$ for the bias percentages (FHV, FLV), where $a$ is the JEPA
arm.  A positive $d$ therefore always indicates an advantage for the
JEPA arm.  Differences are taken per catchment before aggregation.

\begin{table}[H]
  \caption{Hydrological evaluation criteria.  $Q$ denotes observed
    discharge, $\hat{Q}$ the reconstruction, $e=\hat{Q}-Q$ and $\bar{Q}$
    the observed mean over the gap; $r$ is the Pearson correlation,
    $\alpha=\sigma_{\hat{Q}}/\sigma_{Q}$, $\beta=\mu_{\hat{Q}}/\mu_{Q}$,
    $\gamma=\mathrm{CV}_{\hat{Q}}/\mathrm{CV}_{Q}$ and FDC the
    flow-duration curve.  The ideal value of each criterion is given in
    parentheses.  All criteria are evaluated on the hourly series over
    the 91-day gap.}
  \label{tab:hydro_defs}
  \centering\small
  \setlength{\tabcolsep}{4pt}
  \renewcommand{\arraystretch}{1.25}
  \begin{tabular}{@{}llcl@{}}
    \toprule
    Criterion & Definition & Range (ideal) & What it probes \\
    \midrule
    NSE      & $1-\sum e^{2}/\sum(Q-\bar{Q})^{2}$
             & $(-\infty,1]$ $(1)$ & Overall variability \\
    log-NSE  & NSE computed on $\log Q$
             & $(-\infty,1]$ $(1)$ & Low-flow / baseflow regime \\
    KGE      & $1-\sqrt{(r-1)^{2}+(\alpha-1)^{2}+(\beta-1)^{2}}$
             & $(-\infty,1]$ $(1)$ & Timing, variability, bias \\
    KGE$'$   & $1-\sqrt{(r-1)^{2}+(\gamma-1)^{2}+(\beta-1)^{2}}$
             & $(-\infty,1]$ $(1)$ & As KGE, spread as CV ratio \\
    $r$      & Pearson correlation of $\hat{Q}$ and $Q$
             & $[-1,1]$ $(1)$ & Timing only (amplitude-blind) \\
    \addlinespace
    $\alpha$ & $\sigma_{\hat{Q}}/\sigma_{Q}$
             & $[0,\infty)$ $(1)$ & Amplitude \\
    $\beta$  & $\mu_{\hat{Q}}/\mu_{Q}$
             & $[0,\infty)$ $(1)$ & Volume bias \\
    $\gamma$ & $\mathrm{CV}_{\hat{Q}}/\mathrm{CV}_{Q}$
             & $[0,\infty)$ $(1)$ & Relative variability \\
    \addlinespace
    FHV      & Bias over the upper $2\%$ of the FDC
             & \% $(0)$ & High flows (peaks) \\
    FLV      & Bias over the lower $30\%$ of the FDC
             & \% $(0)$ & Low flows (recession) \\
    \bottomrule
  \end{tabular}
\end{table}

\begin{table}[H]
  \caption{Hydrological re-scoring of the two encoder comparisons on
    the twelve CAMELS-SPAT catchments (seed~42).  Every row uses the
    same stored predictions as Table~\ref{tab:camels}; only the scoring
    criterion changes.  $\rho$ is the Spearman rank correlation between
    the criterion's per-catchment paired difference and the gap-MSE
    paired difference.  Wins counts catchments on which the JEPA arm is
    closer to observation, $p$ is a two-sided exact sign test, and
    $\Delta$ is the median paired difference signed so that positive
    favours the JEPA arm.  Units differ by family and are not
    comparable between blocks.}
  \label{tab:camels_hydro}
  \centering\small
  \setlength{\tabcolsep}{4.5pt}
  \begin{tabular}{lrrrrrrr}
    \toprule
    & & \multicolumn{3}{c}{JEPA vs.\ TS2Vec}
      & \multicolumn{3}{c}{JEPA+Bridge vs.\ VAE+Bridge} \\
    \cmidrule(lr){3-5}\cmidrule(lr){6-8}
    Criterion & $\rho$ & Wins & $p$ & $\Delta$ & Wins & $p$ & $\Delta$ \\
    \midrule
    \multicolumn{8}{l}{\emph{Deployed objective} (MSE units)}\\
    Gap-MSE (Table~\ref{tab:camels}) & ---
      & 9/12 & 0.15 & $+0.004$ & 10/12 & 0.04 & $+0.023$ \\
    \addlinespace
    \multicolumn{8}{l}{\emph{Efficiency scores} (dimensionless)}\\
    NSE      & $+0.96$ & 9/12 & 0.15 & $+0.233$ & 6/12 & 1.00 & $-0.006$ \\
    KGE      & $+0.69$ & 9/12 & 0.15 & $+0.161$ & 7/12 & 0.77 & $+0.112$ \\
    KGE$'$   & $+0.83$ & 9/12 & 0.15 & $+0.097$ & 8/12 & 0.39 & $+0.072$ \\
    log-NSE  & $-0.13$ & 5/12 & 0.77 & $-0.051$ & 5/12 & 0.77 & $-0.086$ \\
    $r$      & $+0.31$ & 3/12 & 0.15 & $-0.148$ & 6/12 & 1.00 & $+0.003$ \\
    \addlinespace
    \multicolumn{8}{l}{\emph{Moment ratios} (distance from $1$)}\\
    $\alpha$ & $-0.08$ & 7/12 & 0.77 & $+0.043$ & 7/12 & 0.77 & $+0.116$ \\
    $\beta$  & $+0.78$ & 8/12 & 0.39 & $+0.196$ & 5/12 & 0.77 & $-0.023$ \\
    $\gamma$ & $-0.09$ & 6/12 & 1.00 & $+0.014$ & 8/12 & 0.39 & $+0.098$ \\
    \addlinespace
    \multicolumn{8}{l}{\emph{Flow-duration-curve bias} (percentage points)}\\
    FHV      & $-0.27$ & 6/12 & 1.00 & $-4.5$  & 7/12 & 0.77 & $+5.3$ \\
    FLV$^{\dagger}$
             & $-0.27$ & 6/9  & 0.51 & $+28.9$ & 6/9  & 0.51 & $+37.6$ \\
    \bottomrule
  \end{tabular}

  \vspace{2pt}
  {\footnotesize $^{\dagger}$Three catchments fall below the low-flow
  denominator guard and are excluded from this row only.}
\end{table}

\section{JEPA Embedding Visualisation}
\label{app:jepa_viz}

Figure~\ref{fig:jepa_tsne} shows t-SNE projections of the
64-dimensional JEPA embeddings for four representative datasets.
The top row colours points by meteorological season (DJF/MAM/JJA/SON)
and the bottom row by day type (weekday vs.\ weekend).
Seasonal structure is clearly captured across all datasets: winter
and summer clusters separate well, confirming that the JEPA encoder
learns temperature-driven load patterns.
Weekday/weekend separation is most pronounced for commercial loads
(PepcoCOM, RICom1013), where occupancy schedules differ markedly
between business days and weekends.

\begin{figure}[H]
  \centering
  \includegraphics[width=\linewidth]{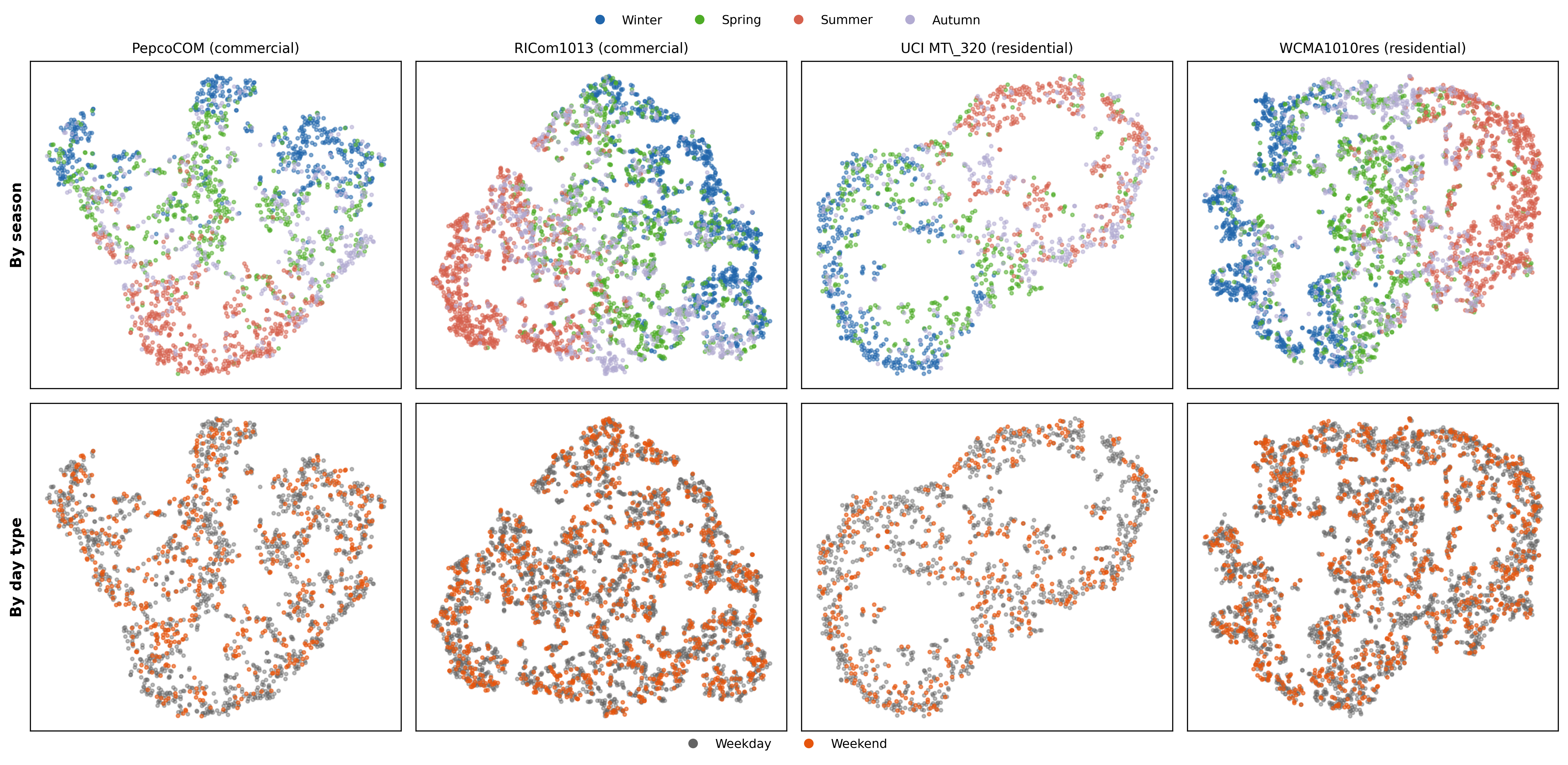}
  \caption{t-SNE projections of 64-dimensional JEPA embeddings for four
           datasets, coloured by season (top) and day type (bottom).
           The encoder captures both seasonal and occupancy structure
           without supervision.}
  \label{fig:jepa_tsne}
\end{figure}

\begin{figure}[H]
  \centering
  \includegraphics[width=\linewidth]{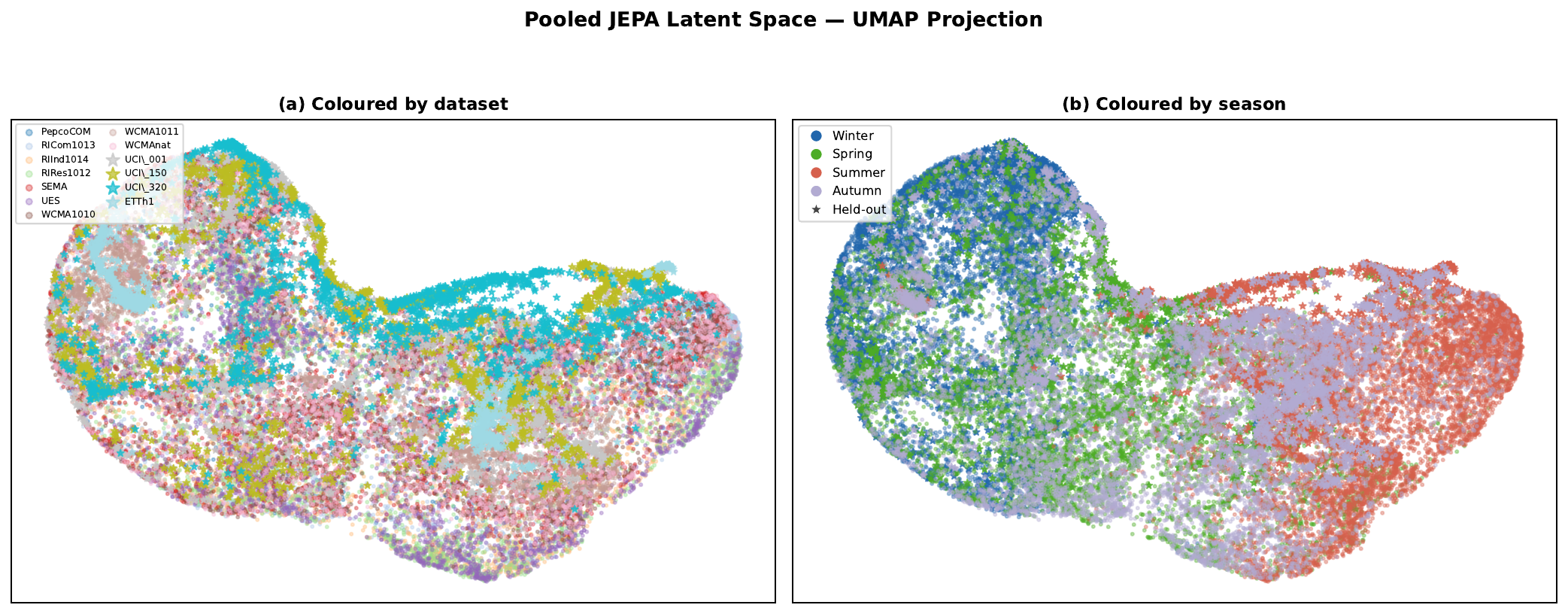}
  \caption{UMAP projection of pooled JEPA embeddings for all thirteen
           datasets (nine training, four held-out).
           \textbf{(a)}~Coloured by dataset identity: training datasets
           (circles) form distinct but overlapping clusters; held-out
           datasets (stars) interleave with training domains, indicating
           that the pooled encoder maps unseen feeds into a shared
           manifold.
           \textbf{(b)}~Coloured by season: the seasonal structure
           (winter/summer separation) is preserved across all domains,
           confirming that temporal dynamics transfer.}
  \label{fig:transfer_umap}
\end{figure}

\section{Data Availability}
\label{app:data_availability}

The nine proprietary load datasets used in this study are US utility records obtained under data-sharing agreements;  they cannot be
publicly released.  The three UCI Electricity datasets are publicly
available from the UCI Machine Learning
Repository~\citep{uci_electricity}.  The ETTh1 dataset is publicly
available from the Informer repository~\citep{zhou2021informer}.
The full pipeline code, model architectures, training scripts,
configuration files, and pre-trained checkpoints will be released
under an open-source licence upon acceptance, enabling reproduction
on any hourly-resolution load dataset with weather covariates.


\section{Generative Bridge Details}
\label{app:bridge_details}

This appendix provides the full mathematical formulations for the DDIM
diffusion and flow-matching bridges summarised in
Section~\ref{sec:diffusion}.

\paragraph{DDIM diffusion bridge.}
The conditional diffusion model uses a Transformer-based denoiser
with cosine noise schedule ($s{=}0.008$,
1000~timesteps)~\citep{nichol2021improved},
$v$-prediction~\citep{salimans2022progressive}, and Min-SNR loss
weighting ($\gamma{=}5.0$)~\citep{hang2023minsnr}.  DDIM sampling
uses 50~steps with $\eta{=}0.0$~\citep{song2021ddim}.

The forward process adds Gaussian noise to the gap embeddings:
\begin{equation}
  q(z_t \mid z_0) = \mathcal{N}\!\bigl(z_t;\;
  \sqrt{\bar{\alpha}_t}\, z_0,\;
  (1 - \bar{\alpha}_t)\, I\bigr).
  \label{eq:forward}
\end{equation}
The denoiser $v_\phi(z_t, t, c)$ predicts the $v$-target
$v_t = \sqrt{\bar{\alpha}_t}\,\epsilon
- \sqrt{1 - \bar{\alpha}_t}\,z_0$,
conditioned on $c = z_{\mathrm{obs}}$.  The $v$-parameterisation
produces more uniform gradient magnitudes across noise levels than
$\epsilon$-prediction~\citep{salimans2022progressive}.

At inference, classifier-free guidance~\citep{ho2022cfg} extrapolates
away from the unconditional score:
\begin{equation}
  \tilde{v}(z_t, t, c) = (1 + w)\, v_\phi(z_t, t, c)
  - w\, v_\phi(z_t, t, \varnothing),
  \label{eq:cfg}
\end{equation}
where $w \in \{1,2,3,4,5\}$ is selected per-dataset via validation
MSE (Appendix~\ref{app:guidance}) and $\varnothing$ is the null
context obtained by dropping conditioning with probability
$p_{\mathrm{uncond}}{=}0.15$ during training.

\paragraph{Flow-matching bridge.}
Given a data sample $z_0$ and noise $\epsilon \sim
\mathcal{N}(0, I)$, the interpolation path is
$z_t = (1 - t)\, z_0 + t\, \epsilon$, $t \in [0,1]$.
The flow-matching objective regresses onto the conditional velocity
field $u_t = \epsilon - z_0$:
\begin{equation}
  \mathcal{L}_{\mathrm{FM}} = \mathbb{E}_{t \sim \mathcal{U}[0,1],\,
  z_0,\, \epsilon}
  \bigl\| v_\theta(z_t, t, c) - (\epsilon - z_0) \bigr\|^2.
  \label{eq:fm_loss}
\end{equation}
Classifier-free guidance applies identically to the flow-matching
velocity field, with the same $p_{\mathrm{uncond}}{=}0.15$.


\section{Inference Algorithm}
\label{app:algorithm}

\begin{algorithm}[H]
\caption{SPLICE Inference: JEPA--Bridge Gap-Filling with ACI}
\label{alg:pipeline}
\begin{algorithmic}[1]
\REQUIRE Observed series $\mathbf{X}_{\mathrm{ctx}} \in \mathbb{R}^{C \times F \times 24}$
  (context days), per-hour covariates~$\mathbf{c}$, gap length~$G$
\ENSURE Imputed gap $\hat{\mathbf{X}}_{\mathrm{gap}} \in \mathbb{R}^{G \times F \times 24}$,
  confidence bands $(L_t, U_t)$
\STATE \textbf{// Stage 1: Encode context}
\FOR{each day $d \in \{1, \dots, C\}$}
  \STATE $\mathbf{z}_d \leftarrow \mathrm{Enc}_{\theta}(\mathbf{x}_d)$
    \hfill $\triangleright$ JEPA online encoder
\ENDFOR
\STATE \textbf{// Stage 2: Predict gap embeddings}
\STATE $\hat{\mathbf{Z}}_{\mathrm{gap}} \leftarrow
  \mathrm{Bridge}_{\phi}(\mathbf{Z}_{\mathrm{ctx}},\, \mathbf{c},\, \mathbf{m})$
  \hfill $\triangleright$ Masked Transformer (det./diffusion/FM)
\STATE \textbf{// Stage 3: Decode to signal space}
\FOR{each gap day $g \in \{1, \dots, G\}$}
  \STATE $\hat{\mathbf{x}}_{g} \leftarrow
    \mathrm{Dec}_{\omega}\bigl(\mathrm{proj}(\hat{\mathbf{z}}_g),\, \mathbf{c}_{g}\bigr)$
    \hfill $\triangleright$ Hourly-conditioned decoder
\ENDFOR
\STATE \textbf{// Stage 4: Adaptive conformal bands}
\STATE Initialise $\alpha_0 \leftarrow \alpha$
\FOR{each step $t$ in calibration + inference window}
  \STATE $\hat{q}_t \leftarrow \mathrm{Quantile}_{1-\alpha_t}(\text{residuals})$
  \STATE $(L_t, U_t) \leftarrow (\hat{x}_t - \hat{q}_t,\; \hat{x}_t + \hat{q}_t)$
  \STATE $\alpha_{t+1} \leftarrow \alpha_t + \gamma(\alpha - \mathbf{1}\{y_t \notin [L_t, U_t]\})$
\ENDFOR
\RETURN $\hat{\mathbf{X}}_{\mathrm{gap}},\; (L_t, U_t)$
\end{algorithmic}
\end{algorithm}

\section{Boundary Smoothness Analysis}
\label{app:boundary_smoothness}

A common artefact in gap imputation is discontinuity at the transition
between observed context and the imputed region.  We measure this via
the normalised second-order finite difference at the gap entry:
$D = |f_1 - 2c_{24} + c_{23}| / (\max(y) - \min(y))$, where $c_k$
denotes context hour $k$ and $f_1$ the first imputed hour.  Across
nine datasets, SPLICE achieves a mean boundary discontinuity of
0.158 on the $[0,1]$ scale.  This is notably low on large-volume
feeds: PepcoCOM (0.016), RIInd1014 (0.035), WCMA1010res (0.018) and
WCMAnatGridRes1004 (0.055), all below 0.06, indicating near-seamless
transitions.  Higher values on WCMA1011lig (0.517) and UES\_NH\_Med
(0.372) reflect the inherent volatility of these small, noisy loads
rather than model deficiency.  The latent-space architecture
encourages smooth transitions because the bridge operates on
slowly-varying JEPA embeddings rather than raw hourly values, providing
a natural inductive bias toward continuity.

\section{Case Study: Zero-Inflated Lighting Load (WCMA1011lig)}
\label{app:wcma1011lig}

WCMA1011lig is a street-lighting sub-load with 46\% zero-valued hours.
Despite receiving no special architectural treatment, SPLICE achieves
97.2\% ACI coverage with intervals 12\% narrower than CQR.  The MSE
gap to SAITS (0.216 vs.\ 0.027) is attributable to the 64-dimensional
bridge embedding losing fine-grained on/off transitions, yet
WCMA1011lig exhibits the smallest diffusion-bridge MSE degradation
(+18.1\%, Tables~\ref{tab:diffusion} and~\ref{tab:gen-ablation}) and the largest decoder
enhancement gain ($-14.3\%$, Table~\ref{tab:decoder_ablation}),
confirming that the bimodal latent distribution benefits from both
stochastic sampling and Load-weighted decoding.  A lightweight
mixture-model gating head could close the remaining gap while
preserving the full pipeline's uncertainty quantification.


\section{Additional Tables and Figures}
\label{app:additional_floats}

\begin{table}[H]
  \caption{Qualitative comparison of imputation approaches.
  \ding{51} = supported; \ding{55} = not supported.}
  \label{tab:method_comparison}
  \centering\small
  \begin{tabular}{lccccc}
    \toprule
    Method & Latent & Generative & Flow & Conformal & Multi-mode \\
    \midrule
    BRITS & \ding{55} & \ding{55} & \ding{55} & \ding{55} & \ding{55} \\
    SAITS & \ding{55} & \ding{55} & \ding{55} & \ding{55} & \ding{55} \\
    CSDI  & \ding{55} & \ding{51} & \ding{55} & \ding{55} & \ding{55} \\
    SSSD  & \ding{55} & \ding{51} & \ding{55} & \ding{55} & \ding{55} \\
    GP-VAE & \ding{51} & \ding{51} & \ding{55} & \ding{55} & \ding{55} \\
    \rowcolor{oursblue}
    \textbf{SPLICE} & \ding{51} & \ding{51} & \ding{51} & \ding{51} & \ding{51} \\
    \bottomrule
  \end{tabular}
\end{table}

\begin{table}[H]
  \caption{Summary of the thirteen evaluation datasets. The nine
  proprietary datasets span 2015--2025; the three UCI datasets span
  2011--2014; ETTh1 spans 2016--2018.  All are at hourly resolution.}
  \label{tab:datasets}
  \centering
  \small
  \begin{tabular}{llrll}
    \toprule
    Dataset & Type & Rows & Load Range & Load Mean \\
    \midrule
    PepcoCOM          & Commercial  & 48{,}192 & 0--255k     & 120.6k \\
    RICom1013         & Commercial  & 89{,}832 & 92k--993k   & 329.7k \\
    RIInd1014         & Industrial  & 89{,}088 & 0--120k     & 28.6k  \\
    RIRes1012         & Residential & 89{,}832 & 14.7k--280k & 111.5k \\
    SEMAResNstar1009  & Residential & 89{,}829 & 0--378      & 89.5   \\
    UES\_NH\_Med      & Mixed       & 86{,}807 & 4.7--48     & 20.3   \\
    WCMA1010res       & Residential & 84{,}744 & 25.9--265   & 106.7  \\
    WCMA1011lig       & Lighting    & 84{,}744 & 0--11.2     & 1.1    \\
    WCMAnatGridRes1004& Residential & 89{,}088 & 0--623k     & 193.3k \\
    \midrule
    UCI\_Elec\_MT\_001 & Residential & 35{,}064 & 0--10.6     & 1.6    \\
    UCI\_Elec\_MT\_150 & Residential & 35{,}064 & 0--102      & 34.9   \\
    UCI\_Elec\_MT\_320 & Residential & 35{,}064 & 0--384      & 122.3  \\
    \midrule
    ETTh1             & Transformer & 17{,}400 & 0--72       & 30.6   \\
    \bottomrule
  \end{tabular}
\end{table}

\begin{table}[H]
  \caption{Coverage and interval width: static CQR vs.\ ACI on the
  nine proprietary datasets. Coverage values closer to 0.95 are
  better; narrower widths are better at equal coverage.}
  \label{tab:conformal}
  \centering
  \small
  \begin{tabular}{lrrrrr}
    \toprule
    Dataset & CQR Cov. & CQR Width & ACI Cov. & ACI Width & $\alpha_T$ \\
    \midrule
    PepcoCOM & 0.940 & 1.486 & 0.942 & 1.542 & 0.041 \\
    RICom1013 & 0.926 & 1.488 & 0.945 & 1.533 & 0.038 \\
    RIInd1014 & 0.860 & 1.576 & 0.935 & 1.562 & 0.016 \\
    RIRes1012 & 0.931 & 1.583 & 0.946 & 1.586 & 0.040 \\
    SEMAResNstar1009 & 0.944 & 1.641 & 0.938 & 1.648 & 0.023 \\
    UES\_NH\_Med & 0.899 & 1.837 & 0.934 & 1.949 & 0.015 \\
    WCMA1010res & 0.923 & 1.417 & 0.943 & 1.495 & 0.035 \\
    WCMA1011lig & 0.972 & 2.569 & 0.954 & 2.912 & 0.059 \\
    WCMAnatGridRes1004 & 0.926 & 1.627 & 0.939 & 1.629 & 0.023 \\
    \bottomrule
  \end{tabular}
\end{table}

\begin{figure}[H]
  \centering
  \begin{subfigure}[b]{0.48\textwidth}
    \includegraphics[width=\linewidth]{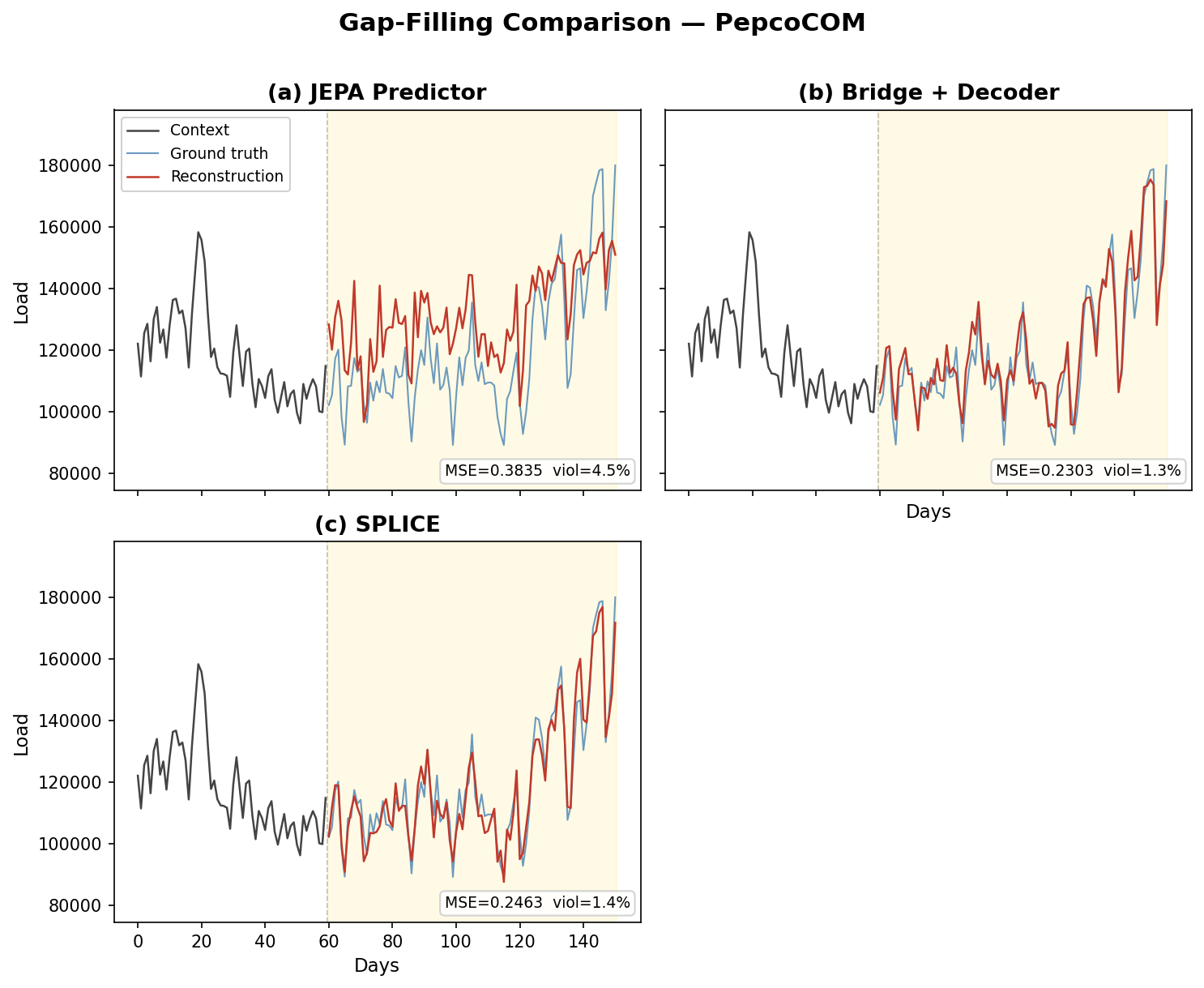}
  \end{subfigure}\hfill
  \begin{subfigure}[b]{0.48\textwidth}
    \includegraphics[width=\linewidth]{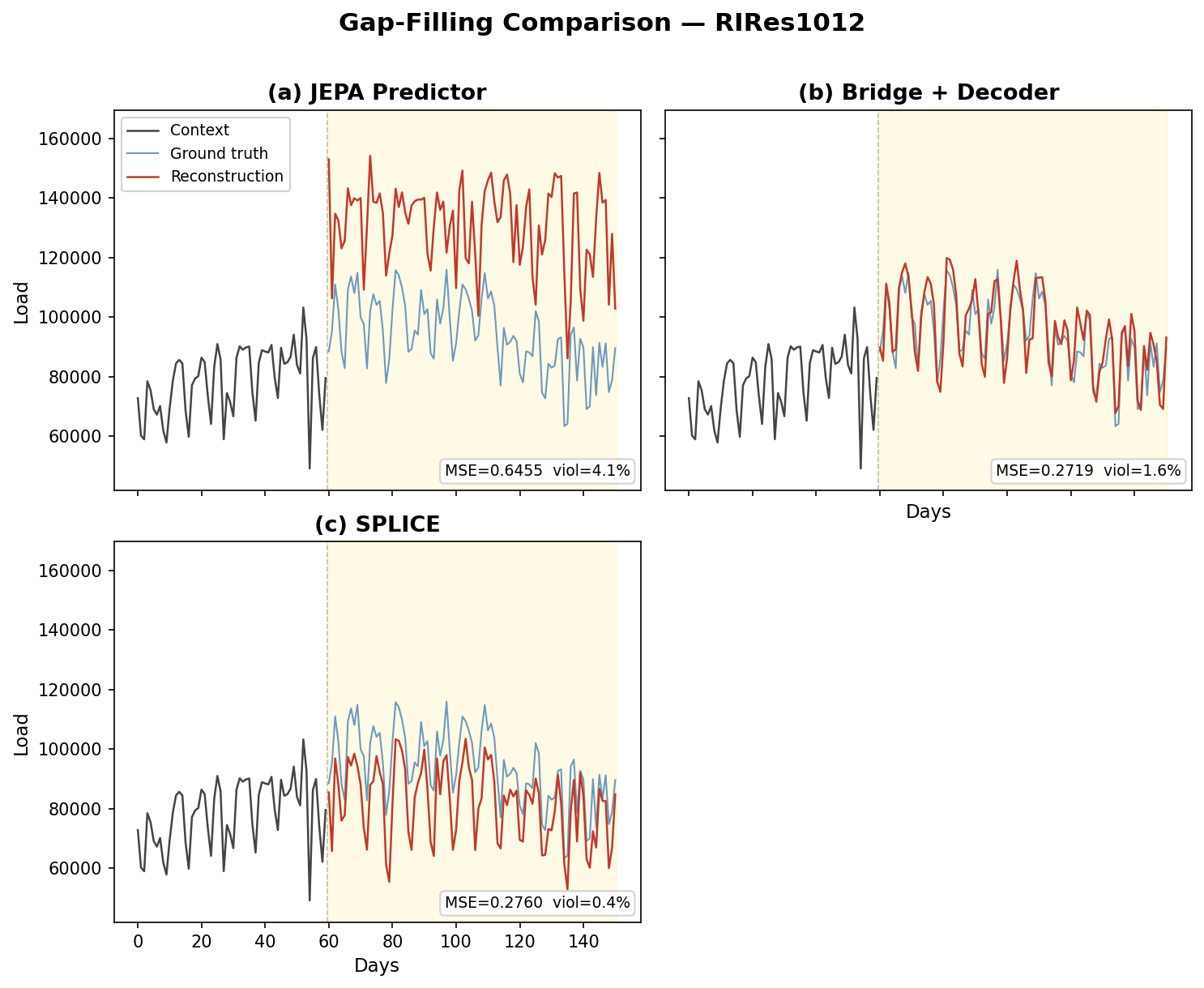}
  \end{subfigure}
  \\[6pt]
  \begin{subfigure}[b]{0.48\textwidth}
    \includegraphics[width=\linewidth]{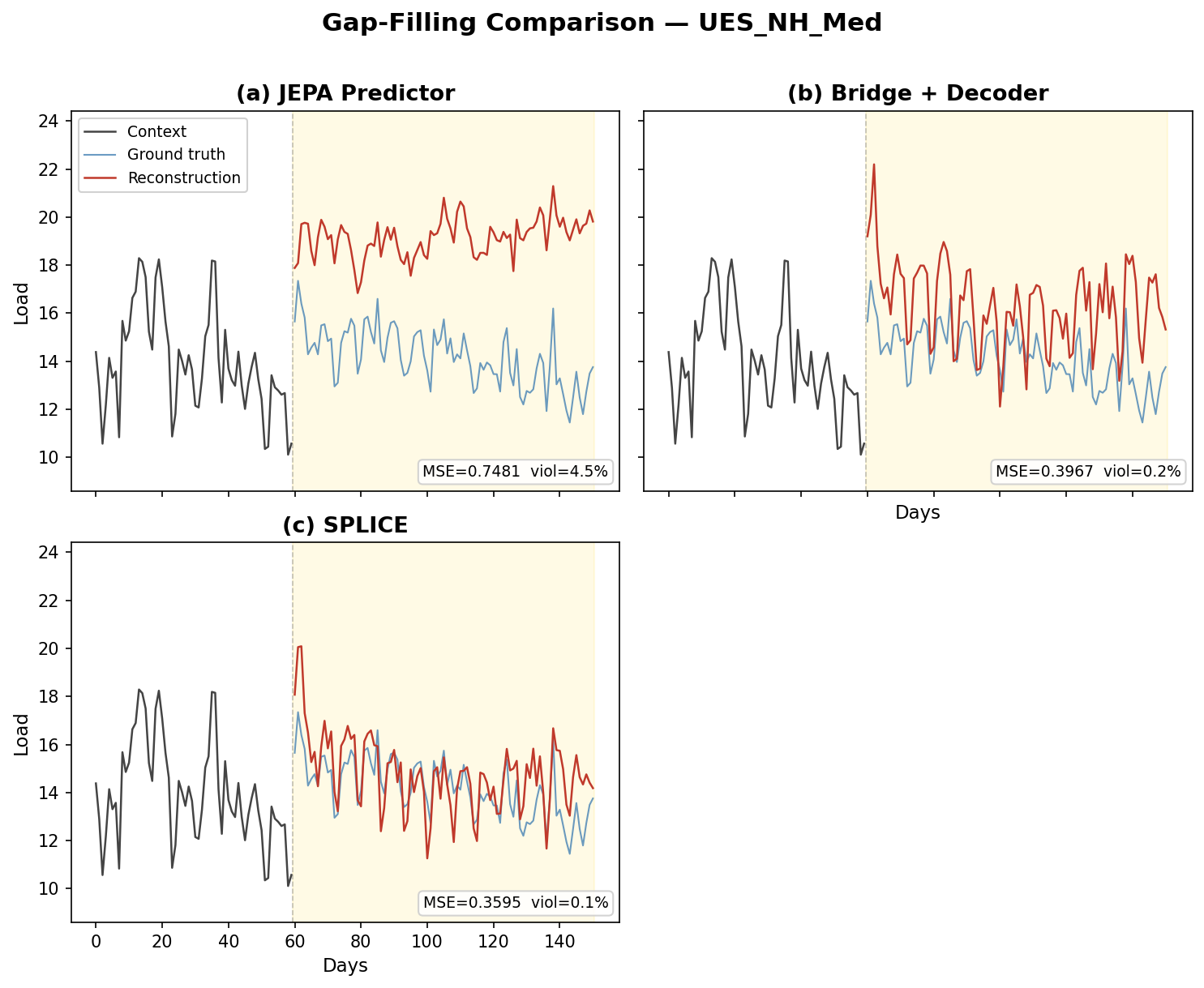}
  \end{subfigure}
  \caption{Qualitative gap-filling comparison for three
  representative datasets spanning commercial (PepcoCOM), residential
  (RIRes1012), and volatile mixed (UES\_NH\_Med) load
  profiles. Each sub-panel shows 60~days of observed context followed
  by 91~days of reconstructed load in the gap.
  (a)~JEPA predictor, (b)~Bridge+Decoder,
  (c)~SPLICE (full pipeline). Constraint violation rates are
  shown per panel.}
  \label{fig:gap_comparison}
\end{figure}

\begin{table}[H]
  \caption{Computational cost (single dataset, NVIDIA RTX A4000).}
  \label{tab:compute}
  \centering
  \small
  \begin{tabular}{lrrr}
    \toprule
    Component & Parameters & Training & Inference (91-day gap) \\
    \midrule
    JEPA encoder + predictor & 932K & $\sim$5 min & 2.4\,ms \\
    Masked Transformer bridge & 828K & $\sim$1--2 min & 3.6\,ms \\
    Hourly-conditioned decoder & 123K & $\sim$10\,s & 0.5\,ms \\
    Flow-matching bridge (5 steps) & 1.36M & $\sim$3--5 min & 2.5\,ms \\
    DDIM diffusion bridge (50 steps) & 1.36M & $\sim$5--8 min & 25\,ms \\
    \midrule
    \textbf{Full pipeline (det.\ bridge)} & \textbf{1.9M} & $\sim$\textbf{8 min} & \textbf{6.5\,ms} \\
    \textbf{Full pipeline (FM-A)} & \textbf{3.2M} & $\sim$\textbf{12 min} & \textbf{9.0\,ms} \\
    \bottomrule
  \end{tabular}
\end{table}

\begin{table}[H]
  \caption{All-feature gap-frame MSE (normalised $[0,1]$, daily decoder) on the nine proprietary datasets. Internal ablation; lower is better.}
  \label{tab:mse}
  \centering
  \small
  \begin{tabular}{lrrrrr}
    \toprule
    Dataset & VAE+Br. & JEPA & Full & $\Delta_{\text{JEPA}}$ (\%) & $\Delta_{\text{VAE}}$ (\%) \\
    \midrule
    PepcoCOM & 0.2576 & 0.3648 & 0.2322 & -36.3 & -9.9 \\
    RICom1013 & 0.2912 & 0.5682 & 0.2587 & -54.5 & -11.2 \\
    RIInd1014 & 0.3991 & 0.6096 & 0.2558 & -58.0 & -35.9 \\
    RIRes1012 & 0.3318 & 0.7406 & 0.2498 & -66.3 & -24.7 \\
    SEMAResNstar1009 & 0.1943 & 0.4748 & 0.1575 & -66.8 & -18.9 \\
    UES\_NH\_Med & 0.5744 & 0.6830 & 0.2987 & -56.3 & -48.0 \\
    WCMA1010res & 0.3176 & 0.4197 & 0.1889 & -55.0 & -40.5 \\
    WCMA1011lig & 0.3093 & 0.3974 & 0.3113 & -21.7 & +0.6 \\
    WCMAnatGridRes1004 & 0.3736 & 0.6534 & 0.2350 & -64.0 & -37.1 \\
    \midrule
    Mean & --- & --- & --- & -53.2 & -25.1 \\
    \bottomrule
  \end{tabular}
\end{table}

\begin{figure}[H]
  \centering
  \includegraphics[width=0.95\linewidth]{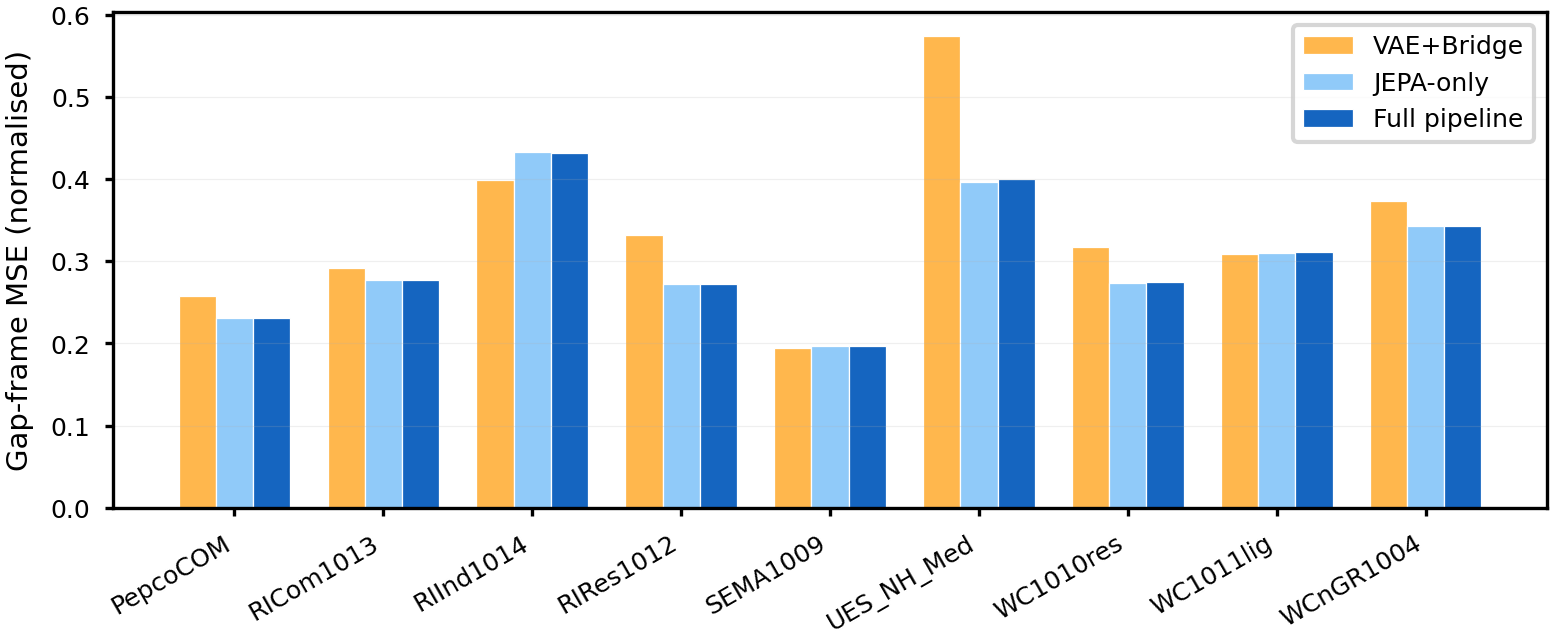}
  \caption{Gap-frame MSE comparison across the nine proprietary datasets. Orange: legacy
  VAE+Bridge. Light blue: JEPA-only decoder. Dark blue: full pipeline.}
  \label{fig:mse_bars}
\end{figure}

\begin{table}[H]
  \caption{MAPE~(\%) vs.\ external imputation baselines (original
           scale). Lower is better; bold marks the best per dataset.
           $\dagger$Near-zero loads inflate MAPE for all methods.
           $\ddagger$SPLICE MAPE not computed for ETTh1.}
  \label{tab:ext-baselines-mape}
  \centering
  \scriptsize
  \resizebox{\textwidth}{!}{%
  \begin{tabular}{lrrrrrrr>{\columncolor{oursblue}}r}
    \toprule
    Dataset & Seasonal & BRITS & SAITS & CSDI & TimesNet & PatchTST & ImputeFormer & SPLICE \\
    \midrule
    PepcoCOM           & 9.4  & 15.9  & 12.6  & 27.0  & 31.3  & 19.2  & 19.1  & \textbf{4.4} \\
    RICom1013          & 23.8 & 22.7  & 27.4  & 36.7  & 37.0  & 30.9  & 33.6  & \textbf{22.4} \\
    RIInd1014          & \textbf{20.4} & 48.0  & 34.8  & 46.0  & 53.3  & 27.7  & 22.2  & 46.0 \\
    RIRes1012          & 52.1 & \textbf{25.4} & 42.3  & 52.2  & 58.1  & 36.8  & 57.3  & 29.0 \\
    SEMAResNstar1009   & 86.6 & \textbf{20.2} & 37.1  & 82.4  & 114.1 & 55.8  & 63.0  & 23.8 \\
    UES\_NH\_Med        & 112.0 & 123.1 & 104.8 & 81.5  & 193.8 & 116.1 & 132.0 & \textbf{71.7} \\
    WCMA1010res        & 23.5 & 21.1  & 16.9  & 39.9  & 40.9  & 25.5  & 28.3  & \textbf{14.3} \\
    WCMA1011lig$^\dagger$ & 329.2 & 362.3 & 216.7 & \textbf{190.7} & 397.8 & 259.2 & 288.3 & --- \\
    WCMAnatGridRes1004 & 36.7 & 17.7  & 17.2  & 50.4  & 28.8  & 22.5  & 26.6  & \textbf{7.7} \\
    \midrule
    UCI\_Elec\_MT\_001$^\dagger$ & 209.7 & 109.7 & \textbf{67.9}  & 422.2 & 410.8 & 79.3  & 192.7 & 71.7 \\
    UCI\_Elec\_MT\_150   & 28.9 & 15.7  & 10.5  & 43.1  & 309.6 & 8.9   & 17.4  & \textbf{7.0} \\
    UCI\_Elec\_MT\_320   & 12.7 & 16.3  & 22.7  & 40.1  & 35.1  & 27.5  & 33.2  & \textbf{9.7} \\
    ETTh1$^{\dagger\ddagger}$ & 121.0 & 161.8 & 194.3 & 236.5 & 181.4 & 191.5 & 183.8 & --- \\
    \bottomrule
  \end{tabular}%
  }
\end{table}

\begin{figure}[H]
  \centering
  \includegraphics[width=\linewidth]{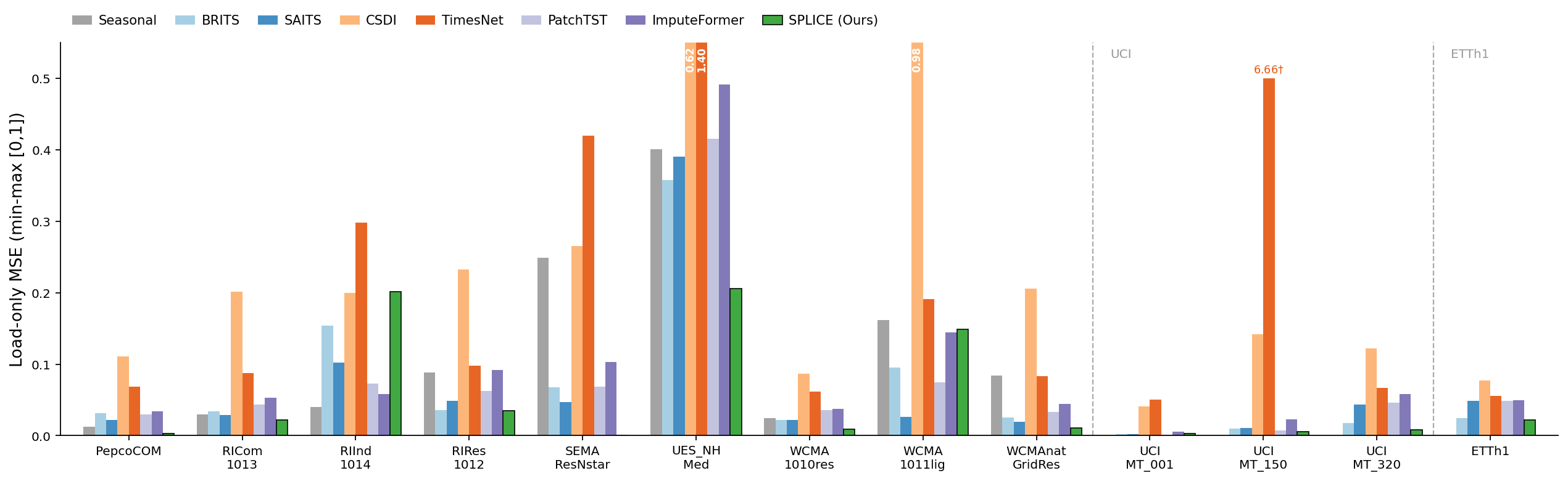}
  \caption{Load-only MSE across thirteen datasets for eight imputation
  methods.  SPLICE (green) achieves the lowest MSE on 9 of 12
  non-degenerate datasets (3-seed means).
  Dashed lines separate proprietary, UCI, and ETTh1 groups.
  $\dagger$TimesNet on UCI\_Elec\_MT\_150 diverged (MSE\,=\,6.66; bar capped
  at 0.5).}
  \label{fig:baseline_comparison}
\end{figure}

\begin{figure}[H]
  \centering
  \begin{subfigure}[b]{\linewidth}
    \includegraphics[width=\linewidth]{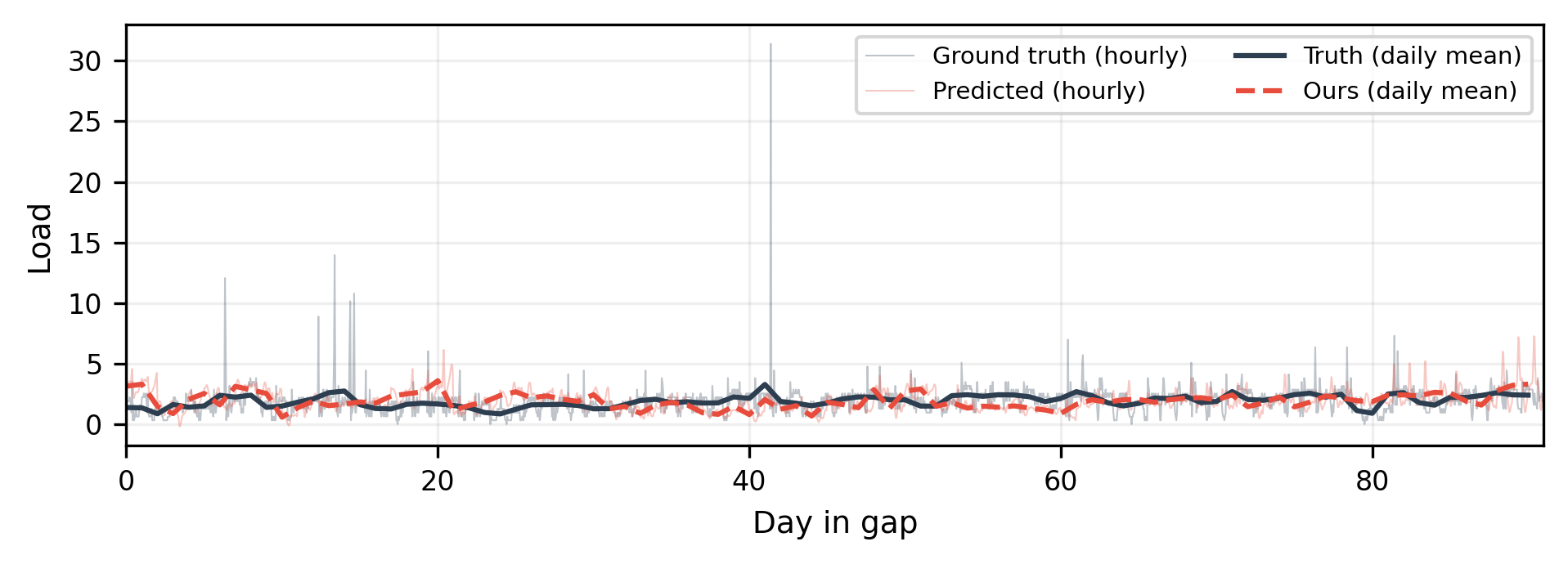}
    \caption{UCI\_Elec\_MT\_001 (low variance, MSE\,=\,0.0019)}
  \end{subfigure}
  \\[4pt]
  \begin{subfigure}[b]{\linewidth}
    \includegraphics[width=\linewidth]{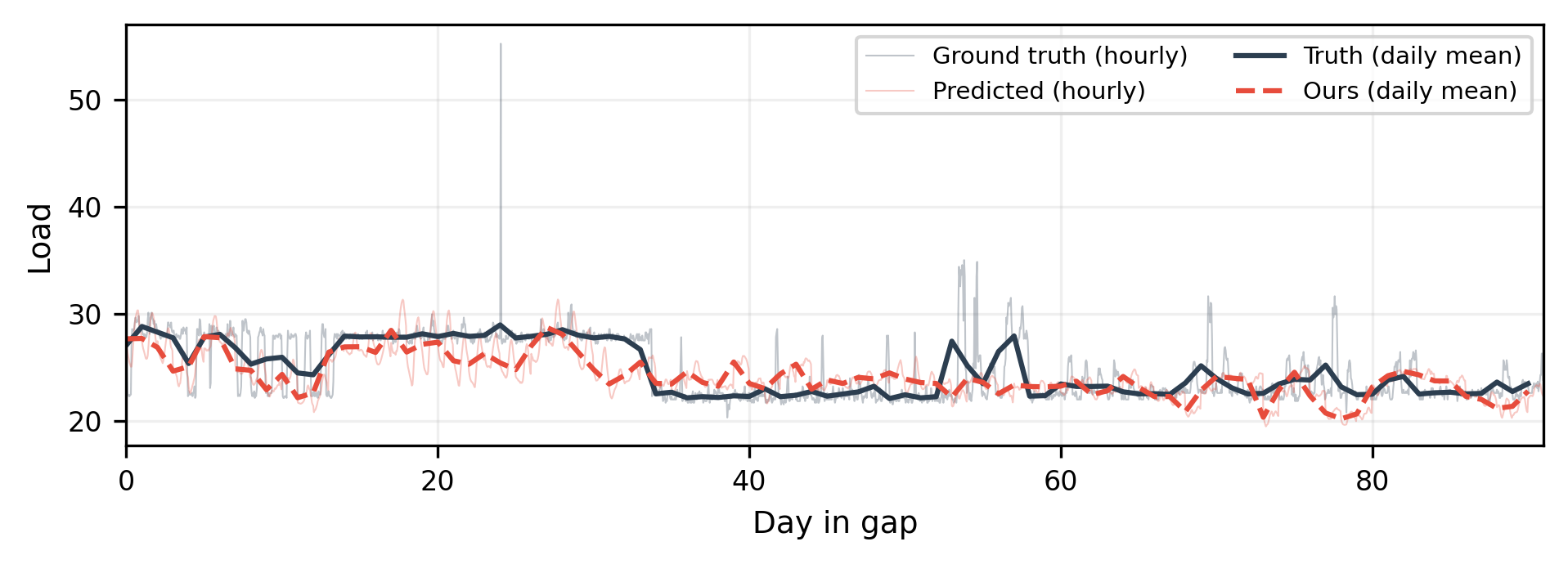}
    \caption{UCI\_Elec\_MT\_150 (medium variance, MSE\,=\,0.0049)}
  \end{subfigure}
  \\[4pt]
  \begin{subfigure}[b]{\linewidth}
    \includegraphics[width=\linewidth]{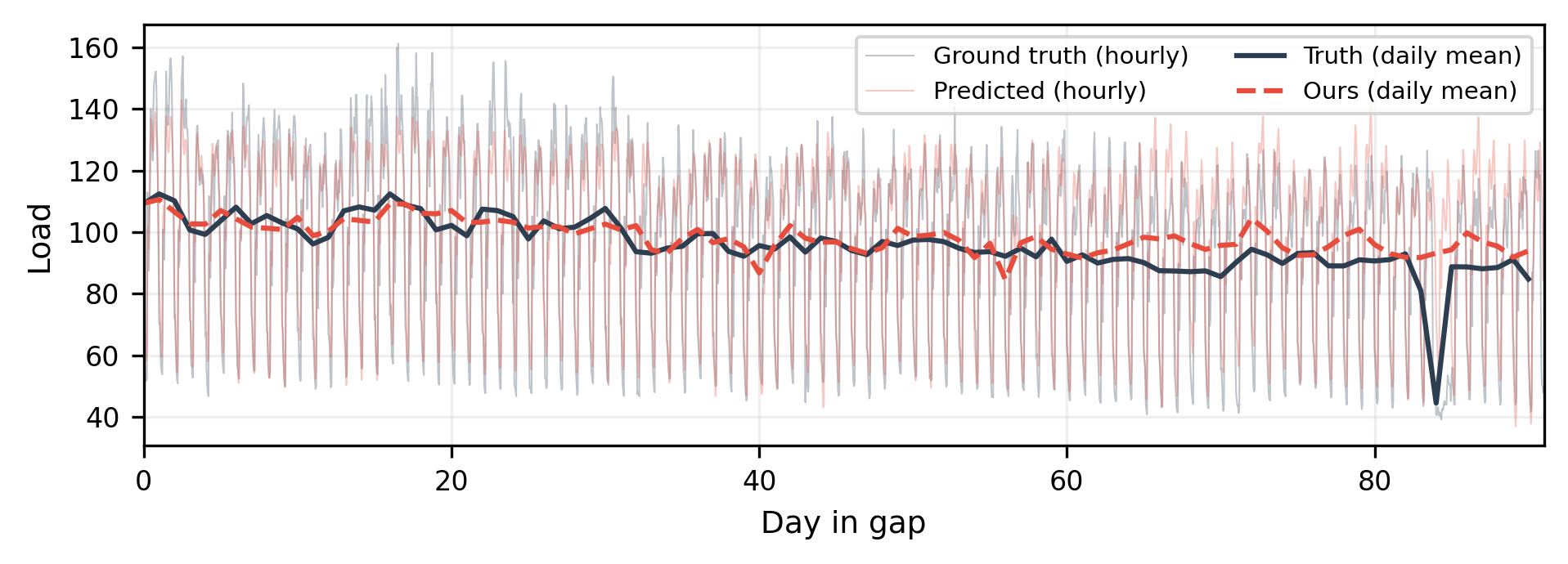}
    \caption{UCI\_Elec\_MT\_320 (high variance, MSE\,=\,0.0082)}
  \end{subfigure}
  \caption{Bridge-decoded gap reconstruction on the three UCI Electricity
  datasets.  Faded traces show hourly ground truth (dark) and predictions
  (red); bold lines show daily means.  Load-only MSE values reported on
  the min-max $[0,1]$ normalised scale.}
  \label{fig:uci_reconstruction}
\end{figure}

\begin{figure}[H]
  \centering
  \begin{subfigure}[b]{0.32\linewidth}
    \includegraphics[width=\linewidth]{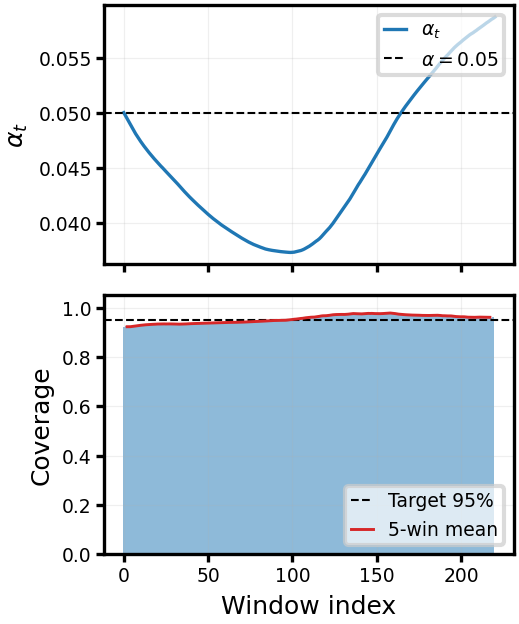}
    \caption{WCMA1011lig (low var.)}
  \end{subfigure}
  \hfill
  \begin{subfigure}[b]{0.32\linewidth}
    \includegraphics[width=\linewidth]{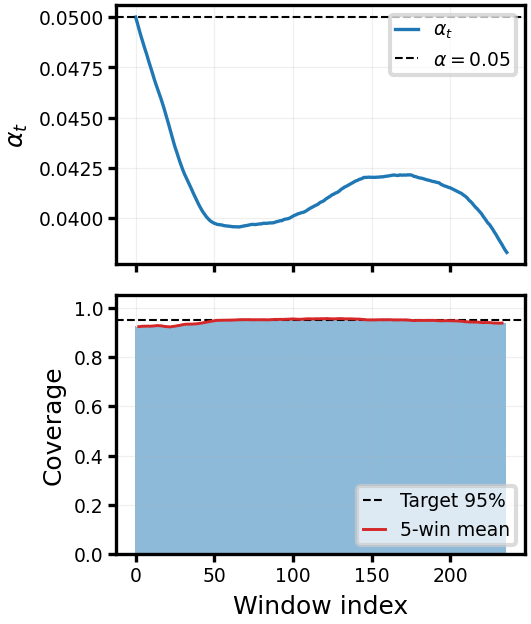}
    \caption{RICom1013 (med.\ var.)}
  \end{subfigure}
  \hfill
  \begin{subfigure}[b]{0.32\linewidth}
    \includegraphics[width=\linewidth]{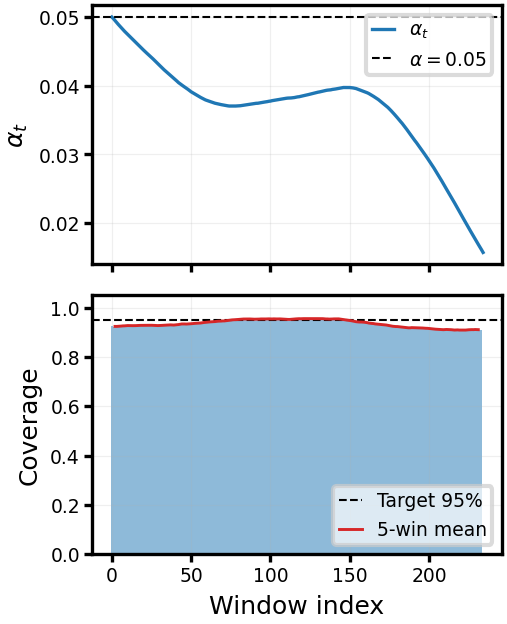}
    \caption{RIInd1014 (high var.)}
  \end{subfigure}
  \caption{Adaptive $\alpha_t$ trajectories for three representative
  proprietary datasets. Dashed line: static $\alpha = 0.05$.}
  \label{fig:aci_trajectories}
\end{figure}

\begin{figure}[H]
  \centering
  \begin{subfigure}[b]{0.92\linewidth}
    \includegraphics[width=\linewidth]{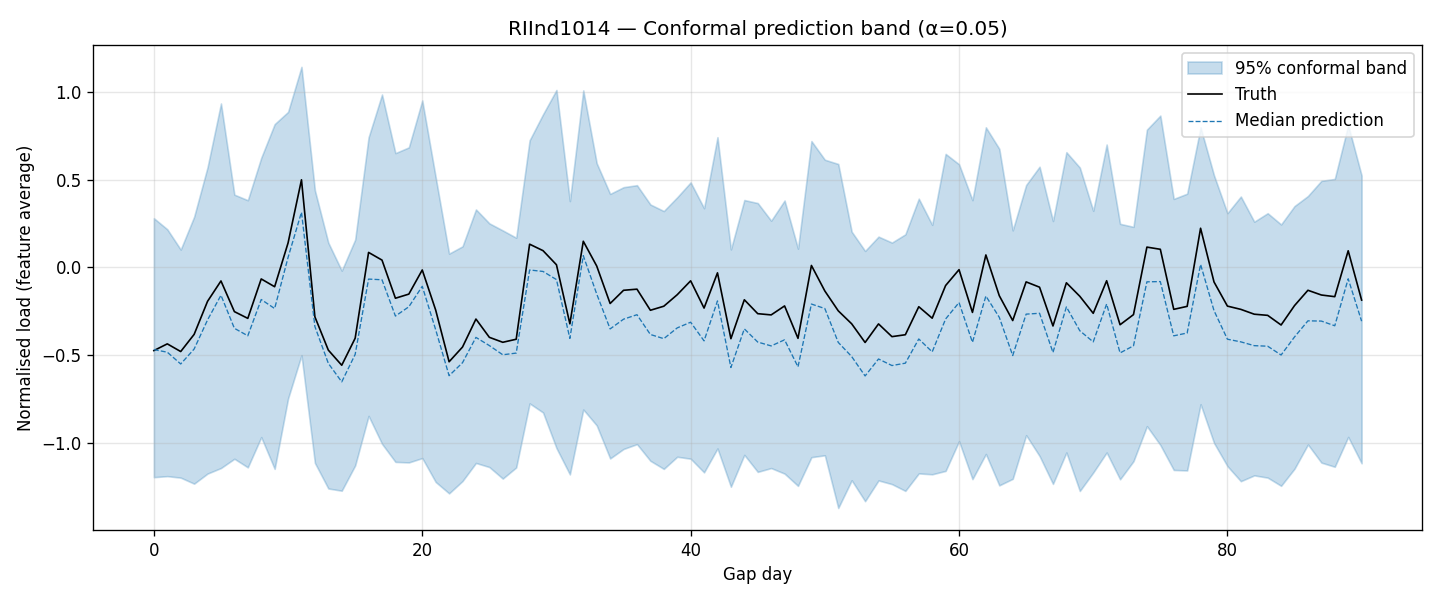}
    \caption{RIInd1014 (high variability)}
  \end{subfigure}
  \\[4pt]
  \begin{subfigure}[b]{0.92\linewidth}
    \includegraphics[width=\linewidth]{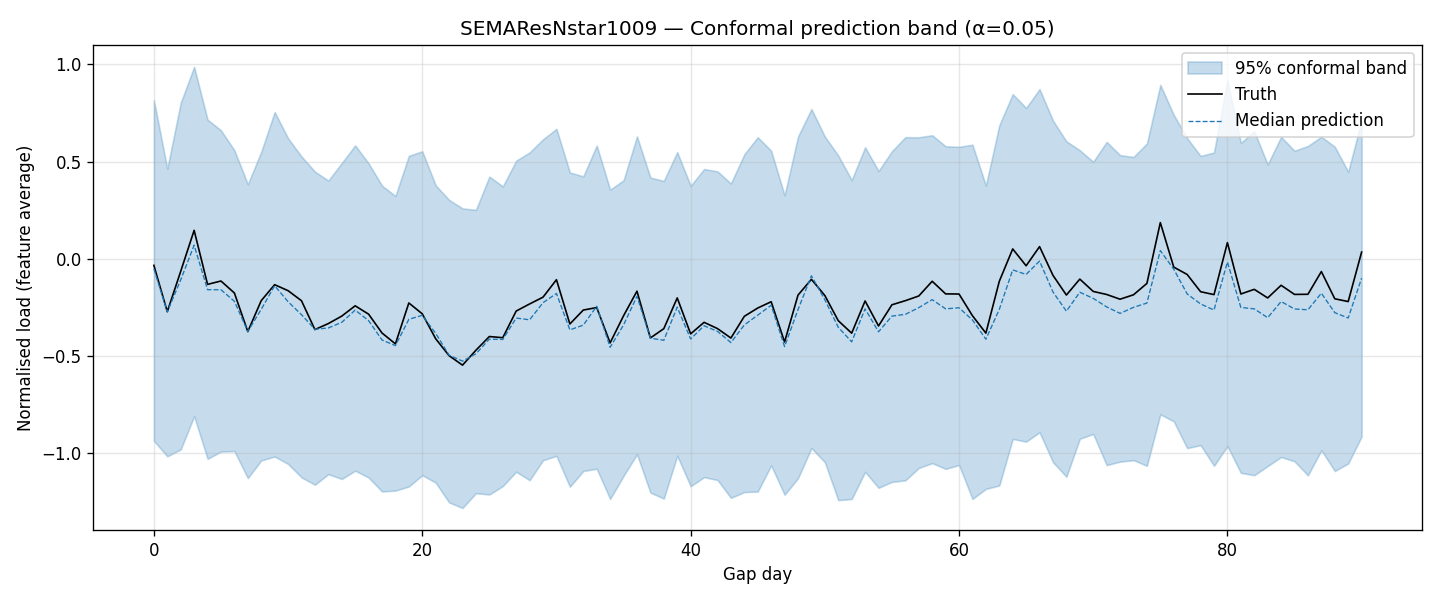}
    \caption{SEMAResNstar1009 (medium variability)}
  \end{subfigure}
  \\[4pt]
  \begin{subfigure}[b]{0.92\linewidth}
    \includegraphics[width=\linewidth]{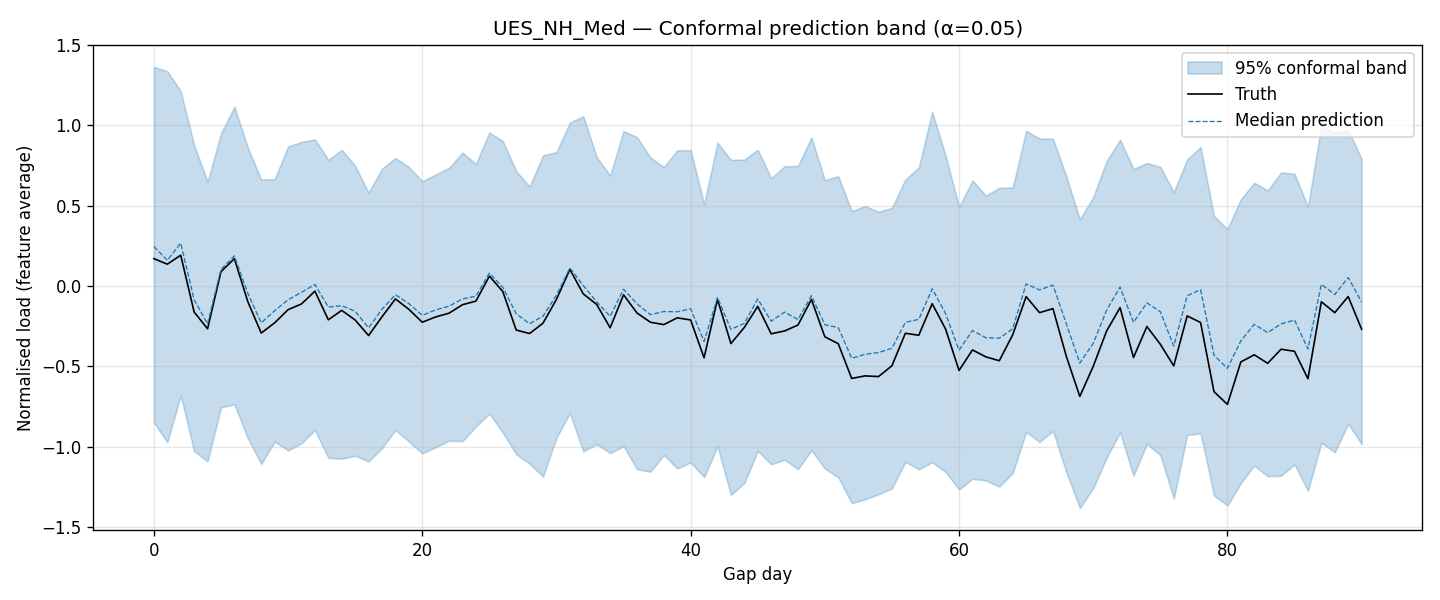}
    \caption{UES\_NH\_Med (high variability)}
  \end{subfigure}
  \caption{Conformal prediction bands (95\% target) for three
  representative proprietary datasets. Shaded region: adaptive conformal
  interval. Solid black: ground truth. Dashed blue: median prediction.}
  \label{fig:conformal_bands}
\end{figure}

\paragraph{CRPS definition.}
The Continuous Ranked Probability Score~\citep{gneiting2007strictly}
measures the compatibility of a predictive CDF $F$ with an observation $y$:
\begin{equation}
  \mathrm{CRPS}(F, y) = \int_{-\infty}^{\infty}
    \bigl(F(x) - \mathbf{1}\{x \ge y\}\bigr)^{2}\, dx .
  \label{eq:crps}
\end{equation}
For an $M$-member ensemble $\{x_1,\dots,x_M\}$ we use the equivalent
energy form:
\begin{equation}
  \mathrm{CRPS} = \frac{1}{M}\sum_{i=1}^{M}|x_i - y|
    - \frac{1}{2M^{2}}\sum_{i=1}^{M}\sum_{j=1}^{M}|x_i - x_j| .
  \label{eq:crps_energy}
\end{equation}

\begin{table}[H]
  \caption{CRPS (lower is better) on $[0,1]$ normalised Load,
  nine proprietary datasets.
  For deterministic baselines, CRPS~$=$~MAE.  SPLICE uses a
  20-member latent ensemble ($\sigma{=}0.15$).
  \textbf{Bold} = best per row.}
  \label{tab:crps}
  \centering\small
  \begin{tabular}{lrrrrr>{\columncolor{oursblue}}r}
    \toprule
    Dataset & Seasonal & BRITS & SAITS & CSDI & TimesNet & SPLICE \\
    \midrule
    PepcoCOM          & 0.078 & 0.137 & 0.101 & 0.179 & 0.232 & \textbf{0.027} \\
    RICom1013         & 0.140 & 0.132 & 0.159 & 0.290 & 0.214 & \textbf{0.131} \\
    RIInd1014         & \textbf{0.174} & 0.412 & 0.304 & 0.350 & 0.420 & 0.372 \\
    RIRes1012         & 0.277 & \textbf{0.117} & 0.198 & 0.321 & 0.290 & 0.135 \\
    SEMAResNstar1009  & 0.459 & \textbf{0.100} & 0.171 & 0.411 & 0.560 & 0.247 \\
    UES\_NH\_Med      & 0.603 & 0.669 & 0.526 & 0.542 & 1.035 & \textbf{0.352} \\
    WCMA1010res       & 0.126 & 0.132 & 0.108 & 0.236 & 0.208 & \textbf{0.071} \\
    WCMA1011lig       & 0.249 & 0.254 & \textbf{0.093} & 0.737 & 0.288 & 0.254 \\
    WCMAnatGridRes1004 & 0.263 & 0.129 & 0.115 & 0.352 & 0.228 & \textbf{0.048} \\
    \midrule
    Mean              & 0.263 & 0.231 & 0.197 & 0.380 & 0.386 & \textbf{0.182} \\
    \bottomrule
  \end{tabular}
\end{table}

\begin{figure}[H]
  \centering
  \includegraphics[width=\linewidth]{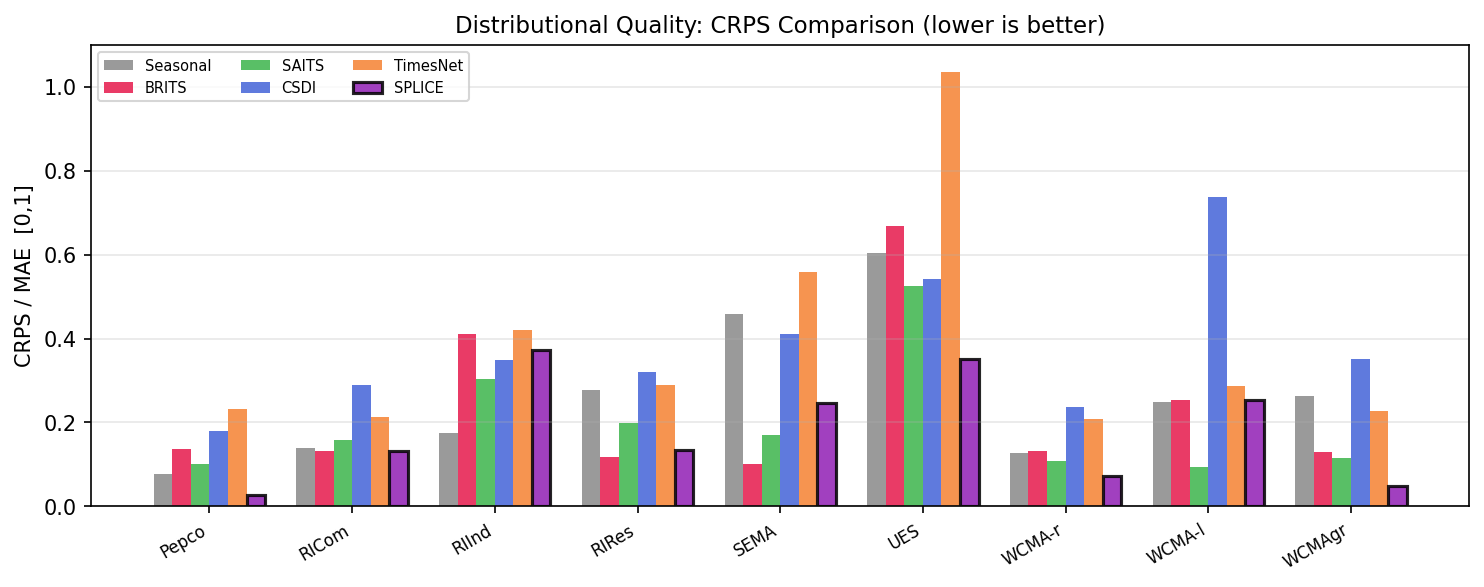}
  \caption{CRPS / MAE comparison across nine proprietary datasets.
  SPLICE-ensemble (purple, black border) achieves the lowest score on
  5 of 9 datasets, with the best average (0.182 vs.\ 0.197 for SAITS).}
  \label{fig:crps_comparison}
\end{figure}

\begin{figure}[H]
  \centering
  \includegraphics[width=\linewidth]{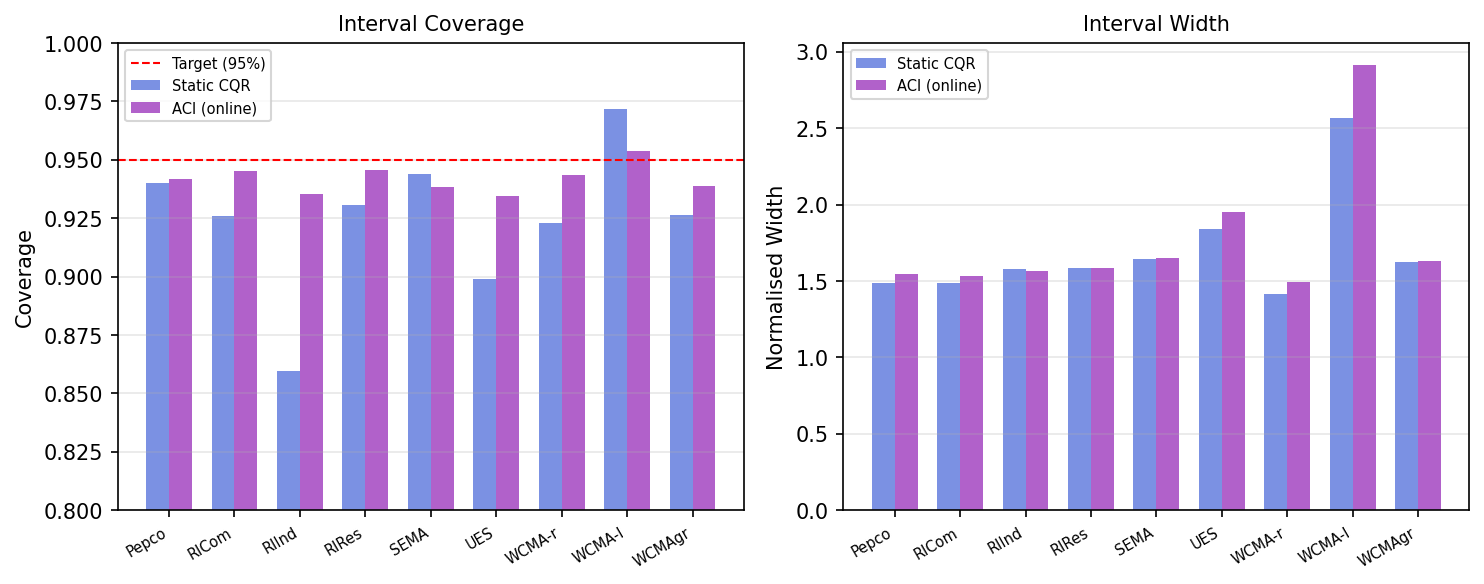}
  \caption{ACI vs.\ static CQR: coverage (left) and normalised width
  (right) on the nine proprietary datasets.  ACI (purple) maintains
  near-95\% coverage on all nine feeds,
  including RIInd1014 and UES\_NH\_Med where CQR under-covers.
  Dashed red line: target 95\%.}
  \label{fig:coverage_comparison}
\end{figure}

\begin{figure}[H]
  \centering
  \includegraphics[width=0.85\linewidth]{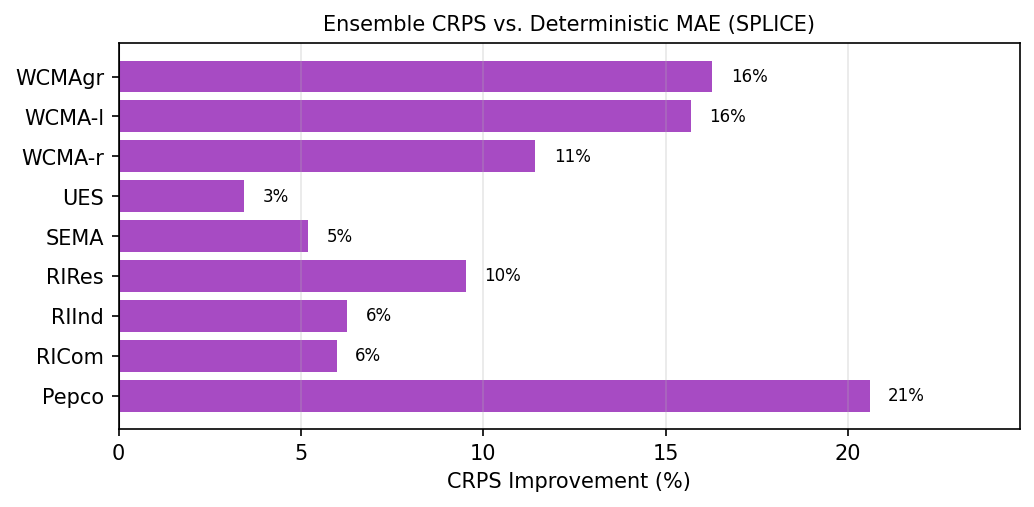}
  \caption{CRPS improvement from latent ensemble ($M{=}20$, $\sigma{=}0.15$)
  over deterministic prediction on the nine proprietary datasets.
  The ensemble reduces CRPS by 3--21\%, with the largest gains on
  stable, high-volume feeds (PepcoCOM, WCMAnatGridRes1004).}
  \label{fig:crps_improvement}
\end{figure}


\section{Hyperparameters}
\label{app:hyperparams}

\begin{table}[H]
  \caption{Architecture and training hyperparameters.}
  \label{tab:hyperparams}
  \centering
  \small
  \begin{tabular}{llr}
    \toprule
    Component & Parameter & Value \\
    \midrule
    \multirow{4}{*}{JEPA Encoder}
      & repr\_dim / $d_{\mathrm{model}}$ / heads / layers & 64 / 128 / 4 / 4 \\
      & LR / weight decay / patience & 3e-4 / 1e-4 / 50 \\
      & EMA decay & 0.996 \\
      & Parameters & $\approx$932K \\
    \midrule
    \multirow{5}{*}{Diffusion Bridge}
      & $d_{\mathrm{model}}$ / heads / layers & 128 / 4 / 6 \\
      & Timesteps / DDIM steps / $\eta$ & 1000 / 50 / 0.0 \\
      & Schedule / min-SNR $\gamma$ & cosine / 5.0 \\
      & LR / patience / $p_{\mathrm{uncond}}$ & 2e-4 / 40 / 0.15 \\
      & Parameters & $\approx$1.36M \\
    \midrule
    \multirow{4}{*}{Flow-Matching Bridge}
      & Backbone & shared with Diffusion \\
      & Euler steps / solver & 5 / Euler \\
      & LR / epochs / patience & 2e-4 / 300 / 40 \\
      & Parameters & $\approx$1.36M \\
    \midrule
    \multirow{2}{*}{ACI}
      & $\alpha$ / $\gamma$ / cal\_fraction & 0.05 / 0.01 / 0.5 \\
      & $n_{\mathrm{cal}}$ / $n_{\mathrm{inf}}$ & 50 / 20 \\
    \midrule
    DLinear & kernel\_size / LR / batch & 25 / 1e-3 / 256 \\
    \midrule
    \multirow{2}{*}{TFT-lite}
      & $d_{\mathrm{model}}$ / heads / layers & 32 / 4 / 2 \\
      & Quantiles & [0.1, 0.5, 0.9] \\
    \midrule
    \multirow{4}{*}{Hourly Decoder}
      & $z_{\mathrm{proj}}$ / hour MLP & $64 \to 256$ / [256, 128] \\
      & Load weight / noise $\sigma$ & 5.0 / 0.15 \\
      & LR / patience & 1e-3 / 20 \\
      & Parameters & $\approx$123K \\
    \bottomrule
  \end{tabular}
\end{table}


\section{Gap-Length Sensitivity}
\label{sec:gap_length}

The primary evaluation in Section~\ref{sec:ext-baselines} uses a
91-day gap.  To characterise how performance scales with gap duration,
we retrain all baselines and SPLICE on 7-day and 30-day gaps across
all 13~datasets, using identical train/test splits and the same
min-max $[0,1]$ MSE metric.
Tables~\ref{tab:gap7d} and~\ref{tab:gap30d} report per-dataset
results; Table~\ref{tab:gap_summary} summarises win counts and mean
rank across all three gap lengths.

\begin{table}[H]
  \caption{Load-only gap MSE for 7-day gaps (min-max $[0,1]$).
           Bold marks the best per dataset.}
  \label{tab:gap7d}
  \centering
  \resizebox{\textwidth}{!}{%
  \begin{tabular}{lrrrrrrr|>{\columncolor{oursblue}}r>{\columncolor{oursblue}}r>{\columncolor{oursblue}}r}
    \toprule
    Dataset & Seasonal & BRITS & SAITS & CSDI & TimesNet & PatchTST & ImputeFormer & Bridge & Diff. & FM-A \\
    \midrule
    PepcoCOM           & 0.1034 & 0.2111 & 0.1564 & 0.3053 & 0.2300 & 0.1828 & 0.2080 & 0.0123 & 0.0068 & \textbf{0.0058} \\
    RICom1013          & 0.1693 & 0.3543 & 0.1922 & 0.5436 & 0.5994 & 0.1910 & 0.3189 & \textbf{0.0864} & 0.1017 & 0.0930 \\
    RIInd1014          & \textbf{0.1167} & 0.3037 & 0.3753 & 0.6126 & 1.2169 & 0.2648 & 0.1746 & 0.7060 & 0.3981 & 0.1807 \\
    RIRes1012          & 0.2284 & \textbf{0.0190} & 0.1711 & 0.3791 & 0.3817 & 0.2303 & 0.2948 & 0.0983 & 0.0792 & 0.0365 \\
    SEMAResNstar1009   & 0.7323 & 0.0845 & 0.0916 & 0.6235 & 2.1617 & 0.1509 & 0.2816 & \textbf{0.0000} & \textbf{0.0000} & \textbf{0.0000} \\
    UES\_NH\_Med        & 10.274 & 11.106 & 11.865 & 18.232 & 39.868 & 11.994 & 13.437 & 8.2602 & \textbf{2.5018} & 4.0416 \\
    WCMA1010res        & 0.0968 & 0.0084 & 0.0402 & 0.1869 & 0.1767 & 0.0462 & 0.0713 & 0.0203 & \textbf{0.0048} & 0.0098 \\
    WCMA1011lig        & 0.2428 & 0.1554 & \textbf{0.0194} & 0.9874 & 0.2934 & 0.1053 & 0.2105 & 0.1124 & 0.0564 & 0.0973 \\
    WCMAnatGridRes1004 & 0.2890 & 0.0787 & 0.0862 & 0.4397 & 0.2056 & 0.1306 & 0.1635 & \textbf{0.0071} & 0.0182 & 0.0105 \\
    \midrule
    UCI\_Elec\_MT\_001   & 0.0506 & 0.1269 & \textbf{0.0324} & 4.1554 & 1.9656 & 0.0402 & 0.1935 & 0.1001 & 0.1679 & 0.1460 \\
    UCI\_Elec\_MT\_150   & 0.9769 & 0.8221 & 0.1223 & 73.549 & 418.70 & \textbf{0.0603} & 0.2318 & 0.1973 & 0.3925 & 0.1103 \\
    UCI\_Elec\_MT\_320   & 0.0751 & 0.0975 & 0.1237 & 0.4453 & 0.2590 & 0.1412 & 0.1694 & 0.0624 & \textbf{0.0470} & 0.0534 \\
    ETTh1              & \textbf{0.0135} & 0.0437 & 0.0476 & 0.0764 & 0.0663 & 0.0467 & 0.0499 & 0.0147 & 0.0262 & 0.0389 \\
    \midrule
    Mean   & 1.028 & 1.032 & 1.025 & 7.734 & 35.86 & 1.045 & 1.216 & 0.744 & \textbf{0.292} & 0.371 \\
    Wins   & 2 & 1 & 2 & 0 & 0 & 1 & 0 & 3 & \textbf{3} & 1 \\
    \bottomrule
  \end{tabular}%
  }
\end{table}

\begin{table}[H]
  \caption{Load-only gap MSE for 30-day gaps (min-max $[0,1]$).
           Bold marks the best per dataset.}
  \label{tab:gap30d}
  \centering
  \resizebox{\textwidth}{!}{%
  \begin{tabular}{lrrrrrrr|>{\columncolor{oursblue}}r>{\columncolor{oursblue}}r>{\columncolor{oursblue}}r}
    \toprule
    Dataset & Seasonal & BRITS & SAITS & CSDI & TimesNet & PatchTST & ImputeFormer & Bridge & Diff. & FM-A \\
    \midrule
    PepcoCOM           & 0.0282 & 0.0812 & 0.0496 & 0.1156 & 0.0840 & 0.0547 & 0.0676 & 0.0054 & \textbf{0.0031} & 0.0049 \\
    RICom1013          & 0.0993 & 0.0818 & \textbf{0.0413} & 0.3167 & 0.3063 & 0.1465 & 0.1639 & 0.0749 & 0.0611 & 0.0557 \\
    RIInd1014          & \textbf{0.1271} & 0.5185 & 0.2741 & 0.6094 & 1.0665 & 0.2113 & 0.1947 & 0.7455 & 0.6594 & 0.7570 \\
    RIRes1012          & 0.2395 & \textbf{0.0218} & 0.1539 & 0.5659 & 0.3611 & 0.1581 & 0.2785 & 0.1250 & 0.0769 & 0.0890 \\
    SEMAResNstar1009   & 0.3887 & 0.0798 & 0.0750 & 0.3195 & 0.6949 & 0.0930 & 0.1186 & \textbf{0.0000} & \textbf{0.0000} & \textbf{0.0000} \\
    UES\_NH\_Med        & 1.9734 & 1.0166 & 1.8763 & 2.7301 & 5.6227 & 2.3205 & 2.3656 & 1.4707 & \textbf{0.5312} & 0.7334 \\
    WCMA1010res        & 0.0704 & 0.0094 & 0.0203 & 0.1105 & 0.1328 & 0.0495 & 0.0543 & 0.0375 & \textbf{0.0073} & 0.0218 \\
    WCMA1011lig        & 0.1848 & 0.1068 & \textbf{0.0216} & 0.7648 & 0.2008 & 0.0521 & 0.1282 & 0.0609 & 0.0782 & 0.0806 \\
    WCMAnatGridRes1004 & 0.1183 & 0.0238 & 0.0290 & 0.3851 & 0.0860 & 0.0348 & 0.0594 & 0.0192 & 0.0137 & \textbf{0.0064} \\
    \midrule
    UCI\_Elec\_MT\_001   & 0.0256 & 0.0596 & 0.0160 & 1.0965 & 1.0627 & \textbf{0.0145} & 0.0768 & 0.0313 & 0.0295 & 0.0257 \\
    UCI\_Elec\_MT\_150   & 0.3614 & 0.2850 & 0.0370 & 1.3824 & 33.125 & \textbf{0.0182} & 0.0707 & 0.0711 & 0.0398 & 0.1725 \\
    UCI\_Elec\_MT\_320   & 0.0286 & 0.0174 & 0.0883 & 0.2235 & 0.1471 & 0.1014 & 0.1041 & 0.0195 & \textbf{0.0170} & 0.0233 \\
    ETTh1              & 0.0286 & 0.0321 & 0.0577 & 0.0814 & 0.0699 & 0.0563 & 0.0620 & 0.0352 & \textbf{0.0276} & 0.0339 \\
    \midrule
    Mean   & 0.283 & 0.180 & 0.211 & 0.669 & 3.305 & 0.255 & 0.288 & 0.207 & \textbf{0.119} & 0.154 \\
    Wins   & 1 & 1 & 2 & 0 & 0 & 2 & 0 & 1 & \textbf{5} & 1 \\
    \bottomrule
  \end{tabular}%
  }
\end{table}

\begin{table}[H]
  \caption{Gap-length sensitivity summary.  Win counts use all
           available datasets per gap (12 for 7\,d/30\,d excl.\
           SEMAResNstar1009; 8 for 91\,d proprietary-only excl.\
           SEMA).  The full 13-dataset enhanced-decoder comparison
           at 91\,d is in Table~\ref{tab:ext-baselines-mse}.
           Mean rank is computed over the same sets.  SPLICE denotes
           the Diffusion variant; win counts are unchanged if the best
           of the Bridge, Diffusion, and Flow-Matching variants is used
           instead.}
  \label{tab:gap_summary}
  \centering
  \resizebox{\textwidth}{!}{%
  \begin{tabular}{lccccccc>{\columncolor{oursblue}}c}
    \toprule
    & Seasonal & BRITS & SAITS & CSDI & TimesNet & PatchTST & ImputeFormer & SPLICE \\
    \midrule
    \multicolumn{9}{l}{\emph{Win counts}} \\
    7\,d\;(\!/12)   & 2 & 1 & 2 & 0 & 0 & 1 & 0 & \textbf{6} \\
    30\,d\;(\!/12) & 1 & 1 & 2 & 0 & 0 & 2 & 0 & \textbf{6} \\
    91\,d\;(\!/8)  & 1 & 1 & 1 & 0 & 0 & 0 & 0 & \textbf{5} \\
    Total          & 4 & 3 & 5 & 0 & 0 & 3 & 0 & \textbf{17} \\
    \midrule
    \multicolumn{9}{l}{\emph{Mean rank ($\downarrow$)}} \\
    7\,d   & 4.25 & 4.33 & 3.83 & 8.50 & 8.33 & 4.25 & 5.92 & \textbf{2.67} \\
    30\,d  & 4.83 & 3.83 & 3.33 & 8.50 & 8.25 & 4.17 & 6.00 & \textbf{2.33} \\
    91\,d  & 4.62 & 4.38 & 3.50 & 8.62 & 7.75 & 4.38 & 5.75 & \textbf{2.50} \\
    \bottomrule
  \end{tabular}%
  }
\end{table}

\paragraph{Analysis.}
SPLICE maintains the lowest mean rank at every gap length (2.33--2.67),
accumulating 17/32 total wins (53\%) vs.\ SAITS 5/32 (16\%).
The two recent Transformer imputers are competitive only on the smooth
public electricity series: PatchTST wins UCI\_Elec\_MT\_150 at 7\,d and
both UCI\_Elec\_MT\_001 and UCI\_Elec\_MT\_150 at 30\,d, while
ImputeFormer wins no dataset at any gap length.
SAITS's advantage collapses at 91\,d (1/8 wins) where observation-space
interpolation must span a longer horizon.  Min-max MSE inflates at
shorter gaps (a normalisation artefact: smaller truth range amplifies
absolute errors), but ordinal rank is invariant.

\section{Significance Tests}
\label{app:significance}

One-sided Wilcoxon signed-rank tests on 91-day Load-only MSE.  The
test unit is the per-dataset mean over three seeds; the near-degenerate
SEMAResNstar1009 is excluded.  Seasonal~Na\"ive is not reported on the
four public datasets (Table~\ref{tab:ext-baselines-mse}), giving
$n{=}8$ for that comparison.

\begin{center}
\begin{tabular}{lcc}
  \toprule
  SPLICE vs. & $n$ & $p$ \\
  \midrule
  CSDI             & 12 & $<$0.001 \\
  TimesNet         & 12 & $<$0.001 \\
  ImputeFormer     & 12 & 0.021 \\
  Seasonal~Na\"ive & 8  & 0.074 \\
  SAITS            & 12 & 0.102 \\
  BRITS            & 12 & 0.117 \\
  PatchTST         & 12 & 0.117 \\
  \bottomrule
\end{tabular}
\end{center}

The four comparisons above $p{=}0.05$ correspond to close means:
SPLICE $0.056$ against SAITS $0.064$, BRITS $0.068$, PatchTST $0.073$,
and Seasonal~Na\"ive $0.105$.

\section{Downstream Forecasting}
\label{sec:forecasting}

To evaluate the downstream utility of imputed data, we train DLinear
and TFT on the SPLICE-filled series (Table~\ref{tab:forecasting}).

\begin{table}[H]
  \caption{Downstream forecasting quality. Lower is better.}
  \label{tab:forecasting}
  \centering
  \small
  \begin{tabular}{lrrrr}
    \toprule
    Dataset & DLinear WMAPE (\%) & DLinear MAE & TFT WMAPE (\%) & TFT MAE \\
    \midrule
    PepcoCOM & 3.52 & 4185.5 & 3.40 & 4049.6 \\
    RICom1013 & 8.22 & 22340.0 & 8.08 & 21985.2 \\
    RIInd1014 & 9.59 & 1681.4 & 9.45 & 1657.8 \\
    RIRes1012 & 8.53 & 6892.4 & 7.62 & 6157.6 \\
    SEMAResNstar1009 & 8.34 & 5.0 & 8.64 & 5.1 \\
    UES\_NH\_Med & 10.11 & 1.4 & 18.12 & 2.5 \\
    WCMA1010res & 5.80 & 5.3 & 6.27 & 5.8 \\
    WCMA1011lig & 19.42 & 0.2 & 32.61 & 0.4 \\
    WCMAnatGridRes1004 & 5.58 & 8238.9 & 6.70 & 9887.7 \\
    \bottomrule
  \end{tabular}
\end{table}

Both models achieve single-digit WMAPE on most datasets, confirming
that the imputed data preserves the statistical structure needed for
24-hour-ahead load forecasting.


\section{Accuracy--Diversity Tradeoff}
\label{sec:diffusion_ablation}

Section~\ref{sec:distributional} demonstrated that even lightweight
noise perturbation ($\sigma{=}0.15$) of the deterministic bridge
reduces CRPS by 3--21\% relative to the point forecast, confirming
that the latent space supports meaningful distributional generation.
Here we push the stochasticity further: we compare the deterministic
bridge against the best generative backend, flow matching with
standard initialisation (FM-A, 5~Euler steps), to characterise the
\emph{accuracy--diversity tradeoff} inherent in latent generative
models.  Table~\ref{tab:diffusion} reports all-feature gap-frame MSE
for both variants across all thirteen datasets.

\begin{table}[H]
  \caption{Accuracy--diversity tradeoff: deterministic bridge vs.\ FM-A
           (5~Euler steps).  Positive $\Delta$ indicates higher MSE for
           the generative variant, reflecting the expected cost of
           increased trajectory diversity.  $^*$WCMA1011lig is the only
           dataset where FM-A improves over the bridge.}
  \label{tab:diffusion}
  \centering
  \small
  \begin{tabular}{lrrr}
    \toprule
    Dataset & Bridge MSE & FM-A MSE & $\Delta$ (\%) \\
    \midrule
    PepcoCOM & 0.2322 & 0.2824 & +21.6 \\
    RICom1013 & 0.2587 & 0.3079 & +19.0 \\
    RIInd1014 & 0.2558 & 0.5278 & +106.3 \\
    RIRes1012 & 0.2498 & 0.3016 & +20.7 \\
    SEMAResNstar1009 & 0.1575 & 0.2381 & +51.1 \\
    UES\_NH\_Med & 0.2987 & 0.4491 & +50.4 \\
    WCMA1010res & 0.1889 & 0.2992 & +58.4 \\
    WCMA1011lig$^*$ & 0.3113 & 0.2943 & $-$5.5 \\
    WCMAnatGridRes1004 & 0.2350 & 0.3809 & +62.1 \\
    \midrule
    UCI\_Elec\_MT\_001 & 0.0980 & 0.1144 & +16.7 \\
    UCI\_Elec\_MT\_150 & 0.0976 & 0.1067 & +9.3 \\
    UCI\_Elec\_MT\_320 & 0.1030 & 0.1155 & +12.1 \\
    ETTh1 & 0.1848 & 0.2803 & +51.7 \\
    \bottomrule
  \end{tabular}
\end{table}

FM-A increases gap-frame MSE on 12 of 13 datasets (9--106\%), since
generative sampling prioritises trajectory diversity over point
accuracy, a well-documented property of stochastic generative
models.  The sole exception is WCMA1011lig ($-5.5\%$), whose bimodal
on/off lighting pattern benefits from the ODE's ability to explore
both modes.  The deterministic bridge, optimised directly for MSE,
naturally dominates on this metric.  However,
Section~\ref{sec:distributional} showed that the intermediate
noise-perturbed variant already captures the distributional benefit: a
$26\%$ average CRPS reduction over the deterministic baseline,
achieved with negligible computational overhead (single forward pass +
Gaussian noise).  This positions the noise-perturbed ensemble as the
practical sweet spot on the accuracy--diversity spectrum: it retains
near-deterministic point accuracy while unlocking calibrated
uncertainty estimation and ACI-compatible conformity scores.

Notably, the MSE premium of FM-A over the bridge is substantially
smaller than that of DDIM (mean $+36\%$ vs.\ $+54\%$ on the nine
proprietary datasets), confirming that the straight ODE paths of flow
matching stay closer to the deterministic solution while still
providing the distributional diversity needed for conformal
calibration.

\section{Generative Backend Ablation: DDIM vs.\ Flow Matching}
\label{sec:gen_ablation}

Table~\ref{tab:gen-ablation} compares the three generative backends
available in SPLICE: the DDIM diffusion bridge (50 reverse steps),
flow-matching with standard Gaussian initialisation (FM-A, 5~Euler
steps), and flow-matching with bridge-residual initialisation
(FM-C, $z_0 = \hat{z}_{\text{bridge}} + 0.15\varepsilon$, 5~Euler
steps).  All three share the same Transformer backbone (1.36M~params)
and conditioning mechanism; only the training objective and sampling
procedure differ.

\begin{table}[H]
  \caption{Generative backend ablation: all-feature gap-frame MSE
           (normalised $[0,1]$, 91-day gap).  \textbf{Bold} = best
           method per dataset.  DDIM uses 50 reverse steps;
           FM-A (standard) and FM-C (bridge-residual) use 5 Euler steps.}
  \label{tab:gen-ablation}
  \centering\small
  \begin{tabular}{lrrr}
    \toprule
    Dataset & DDIM-50 & FM-A & FM-C \\
    \midrule
    PepcoCOM                  & 0.2898 & 0.2824 & \textbf{0.2751} \\
    RICom1013                 & 0.3461 & \textbf{0.3079} & 0.3350 \\
    RIInd1014                 & 0.5564 & \textbf{0.5278} & 0.6355 \\
    RIRes1012                 & 0.3403 & \textbf{0.3016} & 0.3479 \\
    SEMAResNstar1009          & 0.2722 & \textbf{0.2381} & 0.2604 \\
    UES\_NH\_Med              & 0.4356 & 0.4491 & \textbf{0.4235} \\
    WCMA1010res               & 0.3113 & \textbf{0.2992} & 0.3093 \\
    WCMA1011lig               & 0.3677 & \textbf{0.2943} & 0.5628 \\
    WCMAnatGridRes1004        & 0.4121 & \textbf{0.3809} & 0.4216 \\
    \midrule
    UCI\_Elec\_MT\_001        & 0.1161 & \textbf{0.1144} & 0.1161 \\
    UCI\_Elec\_MT\_150        & 0.1238 & \textbf{0.1067} & 0.1233 \\
    UCI\_Elec\_MT\_320        & 0.1214 & \textbf{0.1155} & 0.1206 \\
    ETTh1                     & 0.3572 & 0.2803 & \textbf{0.1895} \\
    \midrule
    Mean (13)                 & 0.3115 & \textbf{0.2845} & 0.3170 \\
    Wins                      & 0 & \textbf{10} & 3 \\
    \bottomrule
  \end{tabular}
\end{table}

\paragraph{Analysis.}
FM-A wins 10 of 13 datasets and reduces mean all-feature MSE by
8.7\% vs.\ DDIM (0.285 vs.\ 0.312 across all 13~datasets) while
using $10\times$ fewer function evaluations (5 Euler steps vs.\ 50
DDIM steps).  Counter-intuitively, bridge-residual initialisation
(FM-C) \emph{hurts} on most datasets: starting the ODE near the
bridge prediction constrains the velocity field to model only residual
corrections, which limits its capacity to explore alternative
trajectories.  FM-C wins only on PepcoCOM, UES\_NH\_Med, and
ETTh1 datasets where the bridge prediction is already strong and
small corrections suffice.  The largest FM-C failure occurs on
WCMA1011lig (+91\% vs.\ FM-A), where the bimodal on/off lighting
pattern requires the generative model to explore both modes freely
rather than being anchored to the bridge mean.

DDIM loses to FM-A on all 13 datasets and to FM-C on 10 of 13.
This is consistent with the straight ODE paths of flow matching being
a better fit than the curved diffusion reverse process in the JEPA
embedding space: low-curvature velocity fields can be integrated
accurately with few steps (quantified below), whereas DDIM's
50-step budget is still insufficient to compensate for the overhead of
the $v$-prediction parameterisation.
On the public benchmarks, the pattern is consistent: FM-A achieves the
lowest MSE on all three UCI clients, while FM-C wins on ETTh1, the
smallest dataset (161~training windows), where the warm-started
initialisation provides a useful inductive bias against overfitting.

Based on these results, FM-A (standard flow matching) emerges as
the preferred generative backend: it achieves
the best reconstruction quality and the fastest inference (5~Euler
steps, $\sim$2\,ms per window on an RTX~A4000).

\paragraph{Straightness of the learned velocity field.}
Few-step integration is accurate when the trained velocity field is
close to straight, so we measure this directly.  Along the guided
sampling trajectory at the deployed guidance scale, the relative
straightness $\mathbb{E}_t\|(z_1-z_0)-v(z_t,t)\|^2 /
\mathbb{E}\|z_1-z_0\|^2$ \citep{liu2023rectified} (the mean
deviation of the velocity from the chord of its own trajectory, zero
for a perfectly straight path) lies in $[0.023, 0.062]$ (median
$0.031$) over 20~held-out windows per dataset.  A 5-step Euler solve
accordingly matches a 200-step reference to within $0.71\%$ (median
$0.29\%$), decreasing monotonically in step count on all 13~datasets:
at the deployed budget the solve is converged, not truncated.
Curvature concentrates near the data end of the path (mean deviation
by quartile of~$t$: $0.030/0.010/0.009/0.082$).

\section{Decoder Enhancement Ablation}
\label{sec:decoder_ablation}

Table~\ref{tab:decoder_ablation} isolates the contribution of the three
enhanced decoder design choices introduced in Section~\ref{sec:decoder}.
The base hourly-conditioned decoder uses a smaller architecture
($z_{\mathrm{proj}}{=}128$, MLP~$[128, 64]$, 37K~params) with standard MSE
loss and no noise augmentation. The enhanced decoder triples the
capacity ($z_{\mathrm{proj}}{=}256$, MLP~$[256, 128]$, 123K~params), adds
Load-weighted loss ($5{\times}$) and embedding noise
augmentation~($\sigma{=}0.15$).

\begin{table}[H]
  \caption{Decoder ablation: Load-only MSE (min-max $[0,1]$). Base = baseline
           hourly decoder (37K params); Enhanced = larger decoder (123K params,
           Load-weighted loss, noise augmentation). Bold marks lower MSE.}
  \label{tab:decoder_ablation}
  \centering
  \small
  \begin{tabular}{lrrr}
    \toprule
    Dataset & Base Decoder & Enhanced Decoder & $\Delta$ (\%) \\
    \midrule
    PepcoCOM & 0.0019 & \textbf{0.0019} & $-$1.8 \\
    RICom1013 & 0.0287 & \textbf{0.0257} & $-$10.3 \\
    RIInd1014 & 0.1825 & \textbf{0.1775} & $-$2.7 \\
    RIRes1012 & \textbf{0.0196} & 0.0280 & +43.1 \\
    SEMAResNstar1009 & 0.0000 & 0.0000 & --- \\
    UES\_NH\_Med & 0.1863 & \textbf{0.1587} & $-$14.8 \\
    WCMA1010res & \textbf{0.0091} & 0.0095 & +4.7 \\
    WCMA1011lig & 0.2517 & \textbf{0.2158} & $-$14.3 \\
    WCMAnatGridRes1004 & 0.0071 & \textbf{0.0055} & $-$22.4 \\
    \midrule
    UCI\_Elec\_MT\_001 & \textbf{0.0018} & 0.0019 & +7.4 \\
    UCI\_Elec\_MT\_150 & 0.0057 & \textbf{0.0049} & $-$14.0 \\
    UCI\_Elec\_MT\_320 & 0.0115 & \textbf{0.0082} & $-$28.5 \\
    \midrule
    Mean (excl.\ SEMA) & 0.0724 & \textbf{0.0656} & $-$9.3 \\
    Wins & 3 & \textbf{8} & \\
    \bottomrule
  \end{tabular}
\end{table}

The enhanced decoder improves Load-only MSE on 8 of 11 non-degenerate
datasets (SEMAResNstar1009 achieves near-zero error with both variants),
with a mean reduction of 9.3\%. The largest gains occur on the hardest
datasets (UES\_NH\_Med: $-$14.8\%, WCMA1011lig: $-$14.3\%), where
Load-weighted loss directs capacity toward the target variable and
noise augmentation compensates for bridge reconstruction errors.
RIRes1012 shows a modest regression (+43.1\% relative, absolute MSE
0.020$\to$0.028), likely due to overfitting on its particularly
regular residential pattern, the larger architecture provides less
benefit when the signal is already well-captured.  On the UCI datasets,
the enhanced decoder yields substantial improvements on MT\_150 ($-$14.0\%) and
MT\_320 ($-$28.5\%), while MT\_001 shows a slight regression (+7.4\%),
consistent with the pattern that the extra capacity helps most on
higher-variance signals.

To isolate the individual contribution of the three enhancements
(added capacity, Load-weighted loss, noise augmentation), we retrain
\emph{only} the hourly decoder on a freshly regenerated seed-42 encoder
and bridge, toggling one enhancement at a time relative to the base
decoder (Table~\ref{tab:decoder_decomp}).  The \texttt{full} recipe
reproduces the reference enhanced-decoder value (ETTh1 $0.01624$),
validating the harness, so each single-toggle delta is attributable to
that toggle rather than run-to-run noise.

\begin{table}[H]
  \caption{Decoder component decomposition: Load-only gap-MSE (seed~42,
           single run).  The \texttt{Base} column is the absolute MSE of
           the base decoder; the next three columns toggle one
           enhancement at a time and \texttt{Full} enables all three (the
           reference enhanced decoder), each reported as the relative
           change vs.\ \texttt{Base}.  Bold marks the largest single-toggle
           gain per dataset.}
  \label{tab:decoder_decomp}
  \centering
  \small
  \begin{tabular}{lrrrrr}
    \toprule
    Dataset & Base & +Capacity & +Load-wt & +Noise & Full \\
    \midrule
    ETTh1              & 0.017444 & $-$6.8\%          & \textbf{$-$11.4\%} & $-$2.3\% & $-$6.9\% \\
    UCI\_Elec\_MT\_150 & 0.004356 & \textbf{$-$0.6\%} & +2.9\%             & $-$0.4\% & +8.9\% \\
    UES\_NH\_Med       & 0.302501 & \textbf{$-$5.4\%} & $-$1.7\%           & +1.7\%   & +2.8\% \\
    \bottomrule
  \end{tabular}
\end{table}

The three enhancements contribute unevenly and dataset-dependently.
Added capacity is the most consistent, helping on all three datasets
($-6.8\%/{-}0.6\%/{-}5.4\%$); the Load-weighted loss gives the single
largest gain where the load signal dominates the reconstruction loss
(ETTh1, $-11.4\%$) but is neutral-to-slightly-negative elsewhere; and
noise augmentation is a minor, mixed-sign regulariser.  The bundled
\texttt{full} recipe improves over the base decoder on ETTh1
($-6.9\%$) but is modestly worse on UCI\_Elec\_MT\_150 ($+8.9\%$) and
UES\_NH\_Med ($+2.8\%$), so the enhancements are best read as a
capacity-plus-load-weighting default that pays off most on
load-dominated signals rather than a uniformly beneficial trio.

\section{Encoder-Swap Ablation: JEPA vs.\ a Contrastive SSL Encoder}
\label{sec:encoder_swap}

To test whether the pipeline's accuracy is attributable to the JEPA
predictive encoder specifically, rather than to self-supervised
pre-training in general, we replace \emph{only} the encoder that
produces the latent embeddings, holding the bridge, enhanced decoder,
ACI wrapper, and evaluation protocol byte-for-byte fixed.  We swap in
TS2Vec~\citep{yue2022ts2vec}, a strong and widely used contrastive
self-supervised representation for time series, re-fitting the
downstream bridge and decoder on the new 64-dimensional latent so the
comparison isolates the representation.  A \texttt{jepa} control run
through the same harness reproduces the reference per-dataset numbers
to within ${\approx}0.3\%$, confirming that the deltas in
Table~\ref{tab:encoder_swap} are attributable to the encoder rather
than harness noise.

\begin{table}[H]
  \caption{Decoder-fixed encoder swap: Load-only gap-MSE (seed~42),
           JEPA (reference control) vs.\ TS2Vec, on the twelve
           non-degenerate datasets, sorted by relative change.
           $\Delta\%>0$ means TS2Vec is worse than JEPA.}
  \label{tab:encoder_swap}
  \centering
  \small
  \begin{tabular}{lrrr}
    \toprule
    Dataset & JEPA & TS2Vec & $\Delta\%$ \\
    \midrule
    UES\_NH\_Med        & 0.312879 & 0.230516 & $-$26.3 \\
    WCMAnatGridRes1004  & 0.015930 & 0.014340 & $-$10.0 \\
    PepcoCOM            & 0.002379 & 0.002275 & $-$4.4 \\
    UCI\_Elec\_MT\_320  & 0.008183 & 0.008220 & +0.5 \\
    RIInd1014           & 0.189767 & 0.196747 & +3.7 \\
    WCMA1011lig         & 0.151660 & 0.157391 & +3.8 \\
    RIRes1012           & 0.050612 & 0.053198 & +5.1 \\
    ETTh1               & 0.016248 & 0.017218 & +6.0 \\
    RICom1013           & 0.020059 & 0.023803 & +18.7 \\
    UCI\_Elec\_MT\_150  & 0.004735 & 0.007182 & +51.7 \\
    WCMA1010res         & 0.006376 & 0.016097 & +152.5 \\
    UCI\_Elec\_MT\_001  & 0.002410 & 0.017232 & +615.0 \\
    \bottomrule
  \end{tabular}
\end{table}

JEPA gives the lower gap-MSE on \textbf{9 of 12} datasets (median
$\Delta = +4.4\%$).  The two extreme TS2Vec regressions ($+152.5\%$,
$+615.0\%$) occur on small, noisy series where the contrastive objective
appears not to converge, and there is no comparably large regression on
the JEPA side.  We therefore scope the claim precisely: JEPA is a
stronger encoder than this contrastive SSL alternative under our
pipeline, so the pipeline's accuracy is not explained by SSL
pre-training in general.  We do not assert robustness to every SSL
encoder.

\subsection{Timing of the daily peak}
\label{app:peak_timing}

Section~\ref{sec:introduction} argues that point accuracy alone cannot
expose under-confident uncertainty, which is why the pipeline is
evaluated on conformal coverage and CRPS alongside reconstruction MSE.
Peak timing is a second criterion of that kind.  Unit commitment and
ramp scheduling, reserve positioning, demand-response dispatch windows,
and peak-coincidence capacity charges all key on the hour at which load
concentrates rather than on average deviation, and a reconstruction can
match the daily profile closely in mean-square terms while placing the
load centre of mass in the wrong hour.  We therefore re-score the
stored predictions of Table~\ref{tab:encoder_swap} on timing, holding
the checkpoints, the arms, and the evaluation window fixed.

\paragraph{Metric.}
For a 24-hour profile $p$ we take the excess mass
$e_h = p_h - \min_h p$, normalise $w_h = e_h / \sum_h e_h$, and define
the \emph{centroid hour} as the argument of
$\sum_h w_h \exp(2\pi i h/24)$.  The \emph{centroid error} is the
circular distance between the observed and predicted centroid hours, in
$[0,12]$ hours.  Subtracting the minimum and normalising makes the
metric invariant to level and scale, so it responds to when load
concentrates rather than to how much of it there is.  We use the
centroid rather than the argmax hour because the argmax is
integer-valued: the paired difference between the two arms is exactly
zero on 40--60\% of days, leaving the metric without the resolution to
separate them.  The centroid produces no exact ties across the 946
scored day-pairs.

A day is scored when its observed profile has a peak-to-minimum ratio
of at least $1.10$, and a dataset is retained when at least 30 of its
days qualify.  This admits 11 of the 12 datasets in
Table~\ref{tab:encoder_swap}, excluding \texttt{UCI\_Elec\_MT\_150},
which has 29 qualifying days.  Each dataset's score is the mean over
days scored by both arms, so the arms are compared on identical day
sets throughout.  Two independent uniform centroid hours give an
expected error of exactly 6 hours, which fixes the scale of the numbers
below.

\begin{table}[H]
  \caption{Centroid error under encoder substitution, in hours (lower
           is better).  $\Delta\%$ is
           $(\mathrm{TS2Vec}-\mathrm{JEPA})/\mathrm{JEPA}$; positive
           favours JEPA, matching the convention of
           Table~\ref{tab:encoder_swap}, whose gap-MSE $\Delta\%$ is
           reproduced in the last column.}
  \label{tab:peak_timing}
  \centering
  \small
  \begin{tabular}{lrrrr}
    \toprule
    Dataset & JEPA & TS2Vec & $\Delta\%$ & gap-MSE $\Delta\%$ \\
    \midrule
    WCMA1010res         & 0.426 & 0.593 & +39.0 & +152.4 \\
    UCI\_Elec\_MT\_001  & 4.855 & 6.427 & +32.4 & +615.2 \\
    RIRes1012           & 0.946 & 1.061 & +12.1 & +5.1 \\
    WCMAnatGridRes1004  & 0.669 & 0.728 & +8.8 & $-$10.0 \\
    UES\_NH\_Med        & 1.121 & 1.158 & +3.3 & $-$26.3 \\
    RIInd1014           & 3.933 & 3.988 & +1.4 & +3.7 \\
    UCI\_Elec\_MT\_320  & 0.378 & 0.383 & +1.2 & +0.5 \\
    WCMA1011lig         & 0.553 & 0.556 & +0.5 & +3.8 \\
    PepcoCOM            & 0.415 & 0.395 & $-$4.7 & $-$4.4 \\
    RICom1013           & 0.750 & 0.701 & $-$6.5 & +18.7 \\
    ETTh1               & 0.715 & 0.662 & $-$7.3 & +6.0 \\
    \midrule
    Median              &       &       & +1.4 & +3.8 \\
    JEPA better         & \multicolumn{2}{r}{} & 8/11 & 8/11 \\
    \bottomrule
  \end{tabular}
\end{table}

\paragraph{Result.}
JEPA representations time the daily peak better on \textbf{8 of 11}
datasets, the same count the objective gives on the same panel, with a
median advantage of $1.4\%$ (Table~\ref{tab:peak_timing}).  The two
measures are related without being redundant: their per-dataset
advantages correlate at $\rho = +0.40$, they disagree on 4 of the 11
datasets, and the win sets are not identical.  The disagreements carry
the additional information.  \texttt{UES\_NH\_Med} and
\texttt{WCMAnatGridRes1004} are the two datasets where the objective
ranks JEPA worst ($-26.3\%$ and $-10.0\%$), and both are timing wins.
\texttt{ETTh1} and \texttt{RICom1013} run the other way, as gap-MSE
wins that reverse on timing.  Peak placement is a partly separate axis
from reconstruction error, and the encoder comparison does not resolve
identically on the two.

The magnitudes vary across the panel, and the variation is systematic.
Where the observed peak timing barely moves, the difference is a few
minutes or less: \texttt{WCMA1011lig} $0.003$~h,
\texttt{UCI\_Elec\_MT\_320} $0.004$~h, \texttt{UES\_NH\_Med} $0.037$~h.
Where the timing does move, it is substantial:
\texttt{WCMA1010res} $0.17$~h, \texttt{RIRes1012} $0.11$~h,
\texttt{UCI\_Elec\_MT\_001} $1.57$~h.  We predicted this relationship
before measuring it, using the day-to-day circular dispersion of the
observed centroid hour, a property of the data alone that is computed
without reference to either model.  Datasets above the median
dispersion show a median advantage of $+3.3\%$ against $+0.5\%$ below
it, with Spearman $\rho = +0.33$ against relative advantage and $+0.40$
against absolute, positive under every leave-one-out.  When the daily
shape is near-constant the climatological profile is sufficient and the
encoder barely registers; when the timing moves it has to be inferred
from context, and the representation is what supplies it.  Dispersion
also predicts each arm's absolute error level ($\rho = +0.85$), which
is the same mechanism seen from the model side, since variable timing
is what makes timing hard.

\texttt{UCI\_Elec\_MT\_001} is the extreme entry in both columns, for
different reasons in each.  Its centroid error is an outlier for both
arms ($4.86$~h and $6.43$~h against a panel median of $0.71$~h,
modified $z$-scores $9.7$ and $12.7$): TS2Vec's centroid on this
dataset sits above the 6-hour chance level and JEPA's below it.  Its
gap-MSE is unremarkable for both arms (modified $z$-scores $-0.7$ and
$0.0$), so the $+615\%$ in the last column reflects a small
denominator.  Removing it leaves both measures at 7 of 10, with medians
of $+1.3\%$ and $+3.7\%$; removing both high-error datasets leaves both
at 6 of 9.  Neither ordering depends on these entries.

\subsection{Timing of the daily peak under the reconstructive baseline}
\label{app:peak_timing_vae}

The encoder substitution of Table~\ref{tab:encoder_swap} exchanges one
self-supervised representation for another.  A second substitution is
available from Table~\ref{tab:mse}, where the JEPA latent is replaced
by the reconstructive $\beta$-VAE latent that the bridge was originally
built on.  We apply the timing metric of
Appendix~\ref{app:peak_timing}, unchanged, to that pair.

\paragraph{Arms and panel.}
Both arms are deterministic and are scored on the published terminal
window.  The first carries JEPA embeddings through the bridge to the
frame decoder; the second carries the $\beta$-VAE latent
($d_{\mathrm{lat}}{=}32$) through its own bridge to the VAE decoder,
and is the \mbox{VAE+Br.}\ column of Table~\ref{tab:mse}, reproduced
here to within $5\times10^{-7}$ on every dataset before any timing
quantity was computed.  The two arms differ in decoder as well as in
encoder, since a reconstructive baseline generates through its own
head; as in Table~\ref{tab:camels}, each contrast is read against its
own control.  The VAE baseline exists on the nine proprietary datasets
of Table~\ref{tab:mse}, eight of which carry into the timing panel;
\texttt{SEMAResNstar1009} is near-degenerate under this normalisation
and is excluded here as elsewhere.  The day guard admits 675 day-pairs
and produces no exact ties on the centroid, against $45.6\%$ ties for
the argmax hour on the same days.

\begin{table}[H]
  \caption{Centroid error under substitution of the reconstructive
           latent, in hours (lower is better).  $\Delta\%$ is
           $(\mathrm{VAE}-\mathrm{JEPA})/\mathrm{JEPA}$; positive
           favours JEPA.  The last column is the same contrast on
           gap-frame MSE, computed on the two deterministic
           reconstructions scored here.}
  \label{tab:peak_timing_vae}
  \centering
  \small
  \begin{tabular}{lrrrr}
    \toprule
    Dataset & JEPA & $\beta$-VAE & $\Delta\%$ & gap-MSE $\Delta\%$ \\
    \midrule
    WCMA1011lig         & 0.439 & 0.720 & +64.1 & $-$0.3 \\
    WCMA1010res         & 0.499 & 0.816 & +63.6 & +16.1 \\
    RICom1013           & 0.809 & 1.316 & +62.6 & +5.0 \\
    PepcoCOM            & 0.471 & 0.720 & +52.9 & +11.9 \\
    WCMAnatGridRes1004  & 0.648 & 0.824 & +27.1 & +8.9 \\
    RIInd1014           & 3.519 & 4.257 & +21.0 & $-$7.7 \\
    RIRes1012           & 1.150 & 1.123 & $-$2.4 & +22.0 \\
    UES\_NH\_Med        & 2.234 & 1.506 & $-$32.6 & +44.8 \\
    \midrule
    Median              &       &       & +40.0 & +10.4 \\
    JEPA better         & \multicolumn{2}{r}{} & 6/8 & 6/8 \\
    \bottomrule
  \end{tabular}
\end{table}

\paragraph{Result.}
JEPA representations time the daily peak better on \textbf{6 of 8}
datasets, the same count the objective gives on the same panel, with a
median advantage of $40.0\%$ (Table~\ref{tab:peak_timing_vae}).  In
absolute terms the median paired difference is $0.27$~h, close to
sixteen minutes of peak displacement, and it exceeds half an hour on
\texttt{RICom1013} and \texttt{RIInd1014}.  The mean and the median
paired differences agree in sign and in order of magnitude ($+11$ and
$+16$ minutes), so the aggregate does not rest on one dataset, and the
ordering holds under every leave-one-out, at 5 or 6 wins out of 7 with
medians between $+27\%$ and $+53\%$.

The two measures rank the panel differently.  Their per-dataset
advantages correlate at $\rho = -0.48$ and the win sets agree on 4 of
the 8 datasets.  The two datasets where the objective ranks JEPA worst,
\texttt{RIInd1014} at $-7.7\%$ and \texttt{WCMA1011lig} at $-0.3\%$,
are both timing wins, the second by $64.1\%$.  The two timing losses
run the other way, and \texttt{UES\_NH\_Med} is the dataset where the
objective margin is largest in JEPA's favour at $+44.8\%$.  Peak
placement and reconstruction error respond to different properties of
the latent, and a reconstructive representation that is competitive on
one can sit well behind on the other.

\section{Latent Dimension and EMA Sensitivity}
\label{sec:repr_dim}

SPLICE fixes the JEPA representation dimension at
$d_{\mathrm{repr}}{=}64$ and the target-EMA decay at $0.996$.  We sweep
both with the same isolation harness, overriding only the swept
hyper-parameter; the whole pipeline adapts because the bridge and
decoder read $d_{\mathrm{repr}}$ from the encoder checkpoint.  The
$d{=}64$ point reproduces the reference ETTh1 value to the digit,
validating the sweep.

\begin{table}[H]
  \caption{Representation-dimension sweep: Load-only gap-MSE
           (seed~42, EMA~0.996).  Bold marks the best dimension per
           column.}
  \label{tab:repr_dim}
  \centering
  \small
  \begin{tabular}{lrrrr}
    \toprule
    $d_{\mathrm{repr}}$ & ETTh1 & UCI\_Elec\_MT\_150 & UCI\_Elec\_MT\_320 & UES\_NH\_Med \\
    \midrule
    16  & 0.021220 & 0.005643 & 0.008427 & 0.255535 \\
    32  & 0.017795 & 0.005948 & 0.008972 & 0.267347 \\
    64  & \textbf{0.016248} & \textbf{0.004749} & \textbf{0.008200} & 0.312677 \\
    128 & 0.017740 & 0.004855 & 0.009306 & \textbf{0.232190} \\
    \bottomrule
  \end{tabular}
\end{table}

On all three public benchmarks the curve is U-shaped with a minimum at
$d{=}64$ (under-capacity below, over-fitting above).  The one
proprietary series shown (UES\_NH\_Med) prefers $d{=}128$; this is the
same short, high-variance dataset that also swings most under the
disjoint-validation check (Appendix~\ref{sec:disjoint}), so we read its
non-monotonicity as early-stopping variance rather than a genuine
dimension effect.  We therefore claim $64$ is at or near the optimum on
the public benchmarks, not that it is uniformly optimal.

\begin{table}[H]
  \caption{Target-EMA sub-sweep at $d_{\mathrm{repr}}{=}64$: Load-only
           gap-MSE (seed~42).  Bold marks the best decay per column.}
  \label{tab:ema}
  \centering
  \small
  \begin{tabular}{lrrr}
    \toprule
    EMA decay & ETTh1 & UCI\_Elec\_MT\_150 & UCI\_Elec\_MT\_320 \\
    \midrule
    0.99              & \textbf{0.016240} & 0.005437 & 0.007858 \\
    0.996 (reference) & 0.016248 & \textbf{0.004749} & 0.008200 \\
    0.999             & 0.018232 & 0.005016 & \textbf{0.007007} \\
    \bottomrule
  \end{tabular}
\end{table}

EMA is a low-sensitivity knob: the spread across the three decays is
much flatter than the dimension sweep (on ETTh1, $0.99$ and $0.996$
differ by ${\approx}0.05\%$), and no value is catastrophic on any
dataset.  The $0.996$ default is a single inherited value that
performs at or near the best everywhere tested; we claim robustness to
the EMA choice rather than strict optimality.

\section{Disjoint-Validation Robustness}
\label{sec:disjoint}

Our early-stopping epoch is selected on a validation region that can
abut the test window; we verify that this confers no unfair advantage.
No model weights or predictions ever touch test labels, the bridge
zeroes the gap latents and all losses use the training portion, so
only the early-stopping \emph{epoch} depends on a region that can
overlap the test window.  As a robustness check we re-run the full
pipeline
with a strictly disjoint pre-gap validation split
(\texttt{SPLICE\_DISJOINT\_VAL{=}1}) and compare against the standard
validation numbers.

\begin{table}[H]
  \caption{Disjoint-validation robustness: Load-only gap-MSE (seed~42),
           standard validation vs.\ strictly disjoint
           pre-gap validation, on the twelve non-degenerate datasets.}
  \label{tab:disjoint}
  \centering
  \small
  \begin{tabular}{lrrr}
    \toprule
    Dataset & Standard val. & Disjoint val. & $\Delta\%$ \\
    \midrule
    PepcoCOM           & 0.002379 & 0.002379 & 0.0 \\
    RICom1013          & 0.020059 & 0.020059 & 0.0 \\
    RIInd1014          & 0.189767 & 0.202492 & +6.7 \\
    RIRes1012          & 0.050612 & 0.050612 & 0.0 \\
    UES\_NH\_Med       & 0.312879 & 0.275488 & $-$12.0 \\
    WCMA1010res        & 0.006376 & 0.006376 & 0.0 \\
    WCMA1011lig        & 0.151660 & 0.228393 & +50.6 \\
    WCMAnatGridRes1004 & 0.015930 & 0.015930 & 0.0 \\
    UCI\_Elec\_MT\_001 & 0.002410 & 0.002410 & 0.0 \\
    UCI\_Elec\_MT\_150 & 0.004735 & 0.004698 & $-$0.8 \\
    UCI\_Elec\_MT\_320 & 0.008183 & 0.008205 & +0.3 \\
    ETTh1              & 0.016248 & 0.016653 & +2.5 \\
    \bottomrule
  \end{tabular}
\end{table}

The median change is \textbf{0.0\%} and six of twelve datasets are
exactly unchanged (long series where the disjoint split is free).  The
remaining changes are mixed in sign (e.g.\ UES\_NH\_Med $-12.0\%$,
ETTh1 $+2.5\%$), and the single large swing (WCMA1011lig $+50.6\%$) is
the short, zero-inflated lighting series whose early-stopping epoch is
inherently high-variance.  A systematic leak would produce a
consistent one-directional degradation; the mixed-sign, median-zero
pattern instead indicates early-stopping variance, not a leak
advantage.

\section{Incremental Ablation Summary}
\label{sec:incremental_ablation}

Table~\ref{tab:incremental_ablation} consolidates the marginal
contribution of each pipeline component on the nine proprietary
datasets, using Load-only metrics on the min-max $[0,1]$ scale.
Starting from the bridge output decoded by a daily MLP~(row~a),
adding per-hour weather and calendar conditioning reduces mean MSE
by 8.1\%~(row~b).  The enhanced decoder (Load-weighted loss, noise
augmentation, $3{\times}$ capacity) yields a further 9.3\%
reduction~(row~c), bringing mean MSE to~0.069.  Generating a
20-member noise-perturbed ensemble ($\sigma{=}0.15$) does not alter
point accuracy but lowers CRPS from 0.169~(=MAE for a deterministic
predictor) to 0.155 ($-8.3\%$; row~d), confirming that the latent
perturbations capture genuine predictive uncertainty.  Finally,
wrapping the pipeline with ACI~(row~e) yields 93.4--95.4\% empirical
coverage across all nine datasets.

\begin{table}[H]
  \caption{Incremental ablation on nine proprietary datasets.
           Each row adds one component to the pipeline; MSE and MAE
           are Load-only on the min-max $[0,1]$ scale.
           For deterministic predictors, $\text{CRPS}=\text{MAE}$;
           the ensemble row reports the energy-form CRPS with $M{=}20$
           members.  Coverage is the ACI empirical rate
           (target 95\%).}
  \label{tab:incremental_ablation}
  \centering
  \small
  \begin{tabular}{llcccc}
    \toprule
    & Configuration & MSE & MAE / CRPS & $\Delta_{\text{MSE}}$ (\%) & ACI Cov.\ \\
    \midrule
    (a) & Bridge $\to$ daily decoder
        & 0.083 & 0.183$^\dagger$ & --- & --- \\
    (b) & \quad + hourly conditioning
        & 0.076 & 0.174$^\dagger$ & $-$8.1 & --- \\
    (c) & \quad + Load-wt.\ loss, noise aug., $3{\times}$ capacity
        & 0.069 & 0.169$^\dagger$ & $-$9.3 & --- \\
    (d) & \quad + 20-member ensemble ($\sigma{=}0.15$)
        & 0.069 & 0.155 & --- & --- \\
    (e) & \quad + ACI ($\gamma{=}0.01$)
        & 0.069 & 0.155 & --- & 93.4--95.4\% \\
    \bottomrule
    \multicolumn{6}{l}{\footnotesize $^\dagger$For deterministic
      predictions, CRPS\,$=$\,MAE.}
  \end{tabular}
\end{table}

\section{Few-Shot Data Efficiency}
\label{app:fewshot}

Figure~\ref{fig:fewshot} shows the learning curves when fine-tuning
only the decoder on the last $N$ days of available training data.
Each curve reports mean $\pm$ standard deviation over 3~seeds.
Horizontal dashed lines indicate the Gap-MSE achieved by full-data
baselines (1\,241 days for UCI, 616 days for ETTh1).

\begin{figure}[H]
  \centering
  \includegraphics[width=\linewidth]{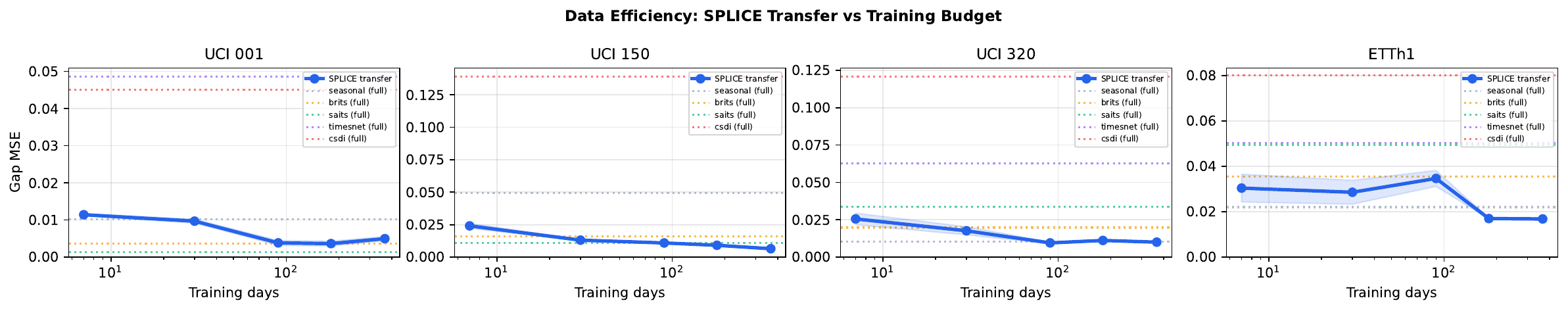}
  \caption{Few-shot data efficiency: SPLICE transfer
    (pooled JEPA + pooled bridge, decoder fine-tuned on $N$~days)
    vs.\ full-data baselines.  With $\approx$90~days of target data,
    SPLICE matches or beats methods trained on the full history.}
  \label{fig:fewshot}
\end{figure}


\begin{thebibliography}{36}

\bibitem[Assran et~al.(2023)]{assran2023jepa}
M.~Assran, Q.~Duval, I.~Misra, P.~Bojanowski, P.~Vincent, M.~Rabbat,
Y.~LeCun, and N.~Ballas.
\newblock Self-supervised learning from images with a joint-embedding
predictive architecture.
\newblock In \emph{CVPR}, 2023.

\bibitem[Bardes et~al.(2022)]{bardes2022vicreg}
A.~Bardes, J.~Ponce, and Y.~LeCun.
\newblock {VICReg}: Variance-invariance-covariance regularization for
self-supervised learning.
\newblock In \emph{ICLR}, 2022.

\bibitem[Bardes et~al.(2024)]{bardes2024vjepa}
A.~Bardes, Q.~Garrido, J.~Ponce, X.~Chen, M.~Rabbat, Y.~LeCun,
M.~Assran, and N.~Ballas.
\newblock Revisiting feature prediction for learning visual
representations from video.
\newblock \emph{arXiv:2404.08471}, 2024.

\bibitem[Cao et~al.(2018)]{cao2018brits}
W.~Cao, D.~Wang, J.~Li, H.~Zhou, L.~Li, and Y.~Li.
\newblock {BRITS}: Bidirectional recurrent imputation for time series.
\newblock In \emph{NeurIPS}, 2018.

\bibitem[Chen \& He(2021)]{chen2021simsiam}
X.~Chen and K.~He.
\newblock Exploring simple {Siamese} representation learning.
\newblock In \emph{CVPR}, 2021.

\bibitem[Du et~al.(2023)]{du2023saits}
W.~Du, D.~C\^{o}t\'{e}, and Y.~Liu.
\newblock {SAITS}: Self-attention-based imputation for time series.
\newblock \emph{Expert Systems with Applications}, 2023.

\bibitem[Du et~al.(2025)]{du2023pypots}
W.~Du, Y.~Yang, L.~Qian, J.~Wang, and Q.~Wen.
\newblock {PyPOTS}: A {Python} toolkit for machine learning on
partially-observed time series.
\newblock \emph{arXiv:2305.18811}, 2025.

\bibitem[Ennadir et~al.(2025)]{ennadir2025tsjepa}
S.~Ennadir, S.~Golkar, and L.~Sarra.
\newblock Joint embeddings go temporal.
\newblock \emph{arXiv:2509.25449}, 2025.

\bibitem[Fortuin et~al.(2020)]{fortuin2020gp}
V.~Fortuin, D.~Baranchuk, G.~R\"{a}tsch, and S.~Mandt.
\newblock {GP-VAE}: Deep probabilistic time series imputation.
\newblock In \emph{AISTATS}, 2020.

\bibitem[Gibbs \& Cand\`{e}s(2021)]{gibbs2021aci}
I.~Gibbs and E.~J. Cand\`{e}s.
\newblock Adaptive conformal inference under distribution shift.
\newblock In \emph{NeurIPS}, 2021.

\bibitem[Gneiting \& Raftery(2007)]{gneiting2007strictly}
T.~Gneiting and A.~E. Raftery.
\newblock Strictly proper scoring rules, prediction, and estimation.
\newblock \emph{Journal of the American Statistical Association},
102(477):359--378, 2007.

\bibitem[Grill et~al.(2020)]{grill2020byol}
J.-B.~Grill, F.~Strub, F.~Altch\'{e}, C.~Tallec, P.~Richemond,
E.~Buchatskaya, C.~Doersch, B.~Avila~Pires, Z.~Guo, M.~Gheshlaghi~Azar,
B.~Piot, K.~Kavukcuoglu, R.~Munos, and M.~Valko.
\newblock Bootstrap your own latent -- a new approach to self-supervised
learning.
\newblock In \emph{NeurIPS}, volume~33, pages 21271--21284, 2020.

\bibitem[Ha \& Schmidhuber(2018)]{ha2018world}
D.~Ha and J.~Schmidhuber.
\newblock World models.
\newblock \emph{arXiv:1803.10122}, 2018.

\bibitem[Hafner et~al.(2020)]{hafner2020dream}
D.~Hafner, T.~Lillicrap, J.~Ba, and M.~Norouzi.
\newblock Dream to control: Learning behaviors by latent imagination.
\newblock In \emph{ICLR}, 2020.

\bibitem[Ho \& Salimans(2022)]{ho2022cfg}
J.~Ho and T.~Salimans.
\newblock Classifier-free diffusion guidance.
\newblock \emph{arXiv:2207.12598}, 2022.

\bibitem[Knoben et~al.(2025)]{camels_spat}
W.~J.~M. Knoben, C.~Th\'ebault, K.~Keshavarz, L.~Torres-Rojas, N.~W. Chaney,
  A.~Pietroniro, and M.~P. Clark.
\newblock Catchment attributes and meteorology for large-sample spatially
  distributed analysis ({CAMELS-SPAT}): streamflow observations, forcing data
  and geospatial data for hydrologic studies across {N}orth {A}merica.
\newblock \emph{Hydrology and Earth System Sciences}, 29(20):5791--5833, 2025.
\newblock \url{https://hess.copernicus.org/articles/29/5791/2025/}.

\bibitem[LeCun(2022)]{lecun2022path}
Y.~LeCun.
\newblock A path towards autonomous machine intelligence.
\newblock \emph{openreview.net preprint}, 2022.

\bibitem[Lim et~al.(2021)]{lim2021tft}
B.~Lim, S.~\"{O}.~Ar\i k, N.~Loeff, and T.~Pfister.
\newblock Temporal {Fusion Transformers} for interpretable multi-horizon
time series forecasting.
\newblock \emph{International Journal of Forecasting}, 2021.

\bibitem[Liu et~al.(2023)]{liu2023rectified}
X.~Liu, C.~Gong, and Q.~Liu.
\newblock Flow straight and fast: Learning to generate and transfer data
with rectified flow.
\newblock In \emph{ICLR}, 2023.

\bibitem[Lipman et~al.(2023)]{lipman2023flow}
Y.~Lipman, R.~T.~Q. Chen, H.~Ben-Hamu, M.~Nickel, and M.~Le.
\newblock Flow matching for generative modeling.
\newblock In \emph{ICLR}, 2023.

\bibitem[L\'{o}pez~Alcaraz \& Strodthoff(2023)]{alcaraz2023sssd}
J.~M. L\'{o}pez~Alcaraz and N.~Strodthoff.
\newblock Diffusion-based time series imputation and forecasting with
structured state space models.
\newblock \emph{TMLR}, 2023.

\bibitem[Nichol \& Dhariwal(2021)]{nichol2021improved}
A.~Q. Nichol and P.~Dhariwal.
\newblock Improved denoising diffusion probabilistic models.
\newblock In \emph{ICML}, 2021.

\bibitem[Nie et~al.(2023)]{nie2023patchtst}
Y.~Nie, N.~H. Nguyen, P.~Sinthong, and J.~Kalagnanam.
\newblock A time series is worth 64 words: Long-term forecasting with
{Transformers}.
\newblock In \emph{ICLR}, 2023.

\bibitem[Nie et~al.(2024)]{nie2024imputeformer}
T.~Nie, G.~Qin, W.~Ma, Y.~Mei, and J.~Sun.
\newblock {ImputeFormer}: Low rankness-induced transformers for generalizable
spatiotemporal imputation.
\newblock In \emph{Proceedings of the 30th ACM SIGKDD Conference on Knowledge
Discovery and Data Mining (KDD)}, pages 2260--2271, 2024.

\bibitem[Paszke et~al.(2019)]{paszke2019pytorch}
A.~Paszke, S.~Gross, F.~Massa, A.~Lerer, J.~Bradbury, G.~Chanan,
T.~Killeen, Z.~Lin, N.~Gimelshein, L.~Antiga, A.~Desmaison,
A.~K\"opf, E.~Yang, Z.~DeVito, M.~Raison, A.~Tejani, S.~Chilamkurthy,
B.~Steiner, L.~Fang, J.~Bai, and S.~Chintala.
\newblock {PyTorch}: An imperative style, high-performance deep learning
library.
\newblock In \emph{NeurIPS}, 2019.

\bibitem[Romano et~al.(2019)]{romano2019cqr}
Y.~Romano, E.~Patterson, and E.~J. Cand\`{e}s.
\newblock Conformalized quantile regression.
\newblock In \emph{NeurIPS}, 2019.

\bibitem[Salimans \& Ho(2022)]{salimans2022progressive}
T.~Salimans and J.~Ho.
\newblock Progressive distillation for fast sampling of diffusion models.
\newblock In \emph{ICLR}, 2022.

\bibitem[Song et~al.(2021)]{song2021ddim}
J.~Song, C.~Meng, and S.~Ermon.
\newblock Denoising diffusion implicit models.
\newblock In \emph{ICLR}, 2021.

\bibitem[Tashiro et~al.(2021)]{tashiro2021csdi}
Y.~Tashiro, J.~Song, Y.~Song, and S.~Ermon.
\newblock {CSDI}: Conditional score-based diffusion models for
probabilistic time series imputation.
\newblock In \emph{NeurIPS}, 2021.

\bibitem[Vovk et~al.(2005)]{vovk2005algorithmic}
V.~Vovk, A.~Gammerman, and G.~Shafer.
\newblock \emph{Algorithmic Learning in a Random World}.
\newblock Springer, 2005.

\bibitem[Yue et~al.(2022)]{yue2022ts2vec}
Z.~Yue, Y.~Wang, J.~Duan, T.~Yang, C.~Huang, Y.~Tong, and B.~Xu.
\newblock {TS2Vec}: Towards universal representation of time series.
\newblock In \emph{AAAI}, 2022.

\bibitem[Zaffran et~al.(2022)]{zaffran2022aci}
M.~Zaffran, O.~Féron, Y.~Goude, J.~Josse, and A.~Dieuleveut.
\newblock Adaptive conformal predictions for time series.
\newblock In \emph{ICML}, 2022.

\bibitem[Zeng et~al.(2023)]{zeng2023dlinear}
A.~Zeng, M.~Chen, L.~Zhang, and Q.~Xu.
\newblock Are {Transformers} effective for time series forecasting?
\newblock In \emph{AAAI}, 2023.

\bibitem[{Trindade}(2015)]{uci_electricity}
A.~Trindade.
\newblock {ElectricityLoadDiagrams20112014} data set.
\newblock UCI Machine Learning Repository, 2015.
\newblock \url{https://archive.ics.uci.edu/dataset/321/electricityloaddiagrams20112014}.
\newblock DOI: \href{https://doi.org/10.24432/C58C86}{10.24432/C58C86}.

\bibitem[{Open-Meteo}(2024)]{openmeteo}
{Open-Meteo}.
\newblock Free weather {API}---{ERA5} historical reanalysis.
\newblock \url{https://open-meteo.com/}, 2024.

\bibitem[Hang et~al.(2023)]{hang2023minsnr}
T.~Hang, S.~Gu, C.~Li, J.~Bao, D.~Chen, H.~Hu, X.~Geng, and B.~Guo.
\newblock Efficient diffusion training via {Min-SNR} weighting strategy.
\newblock In \emph{ICCV}, 2023.

\bibitem[Wu et~al.(2023)]{wu2023timesnet}
H.~Wu, T.~Hu, Y.~Liu, H.~Zhou, J.~Wang, and M.~Long.
\newblock {TimesNet}: Temporal 2{D}-variation modeling for general time
series analysis.
\newblock In \emph{ICLR}, 2023.

\bibitem[Zhou et~al.(2021)]{zhou2021informer}
H.~Zhou, S.~Zhang, J.~Peng, S.~Zhang, J.~Li, H.~Xiong, and W.~Zhang.
\newblock {Informer}: Beyond efficient transformer for long sequence
time-series forecasting.
\newblock In \emph{AAAI}, 2021.

\end{thebibliography}
\end{document}